\documentclass[reqno, a4paper]{amsart}
\usepackage{amsmath, amssymb,amsthm}

\pdfoutput=1
\usepackage{color}

\def\comment#1{}

\usepackage{hyperref}

\usepackage{txfonts}
\usepackage{enumerate}
\usepackage{caption}
\usepackage{accents}

\providecommand{\black}[1]{\textcolor{black}{{#1}}}

\usepackage{graphicx}
\usepackage{float}
\graphicspath{ {img/} }
\newtheorem{theorem}{Theorem}
\newtheorem{assumption}[theorem]{Assumption}
\newtheorem{corollary}[theorem]{Corollary}
\newtheorem{definition}[theorem]{Definition}
\newtheorem{lemma}[theorem]{Lemma}
\newtheorem{proposition}[theorem]{Proposition}

{Important Convention}

\theoremstyle{remark}
\newtheorem{remark}[theorem]{Remark}
\newtheorem{example}[theorem]{Example}

 \newcommand{\cpl}{\text{Cpl}}
 \newcommand{\eps}{\varepsilon}

 \renewcommand{\phi}{\varphi}

\newcommand{\N}{\mathbb{N}}

\newcommand{\R}{\mathbb{R}}

\DeclareMathOperator{\supp}{supp}
\DeclareMathOperator*{\argmin}{argmin}
\newcommand{\bes}{\begin{subequations}}
\newcommand{\ees}{\end{subequations}}
\newcommand{\eea}{\end{eqnarray}}

\newcommand{\NN}{{\mathbb N}}

\newcommand{\EE}{{\mathbb E}}
\newcommand{\RR}{{\mathbb R}}

\newcommand{\cal}{\mathcal}

\usepackage{bbm}

\renewcommand{\eps}{\varepsilon}
\renewcommand{\epsilon}{\varepsilon}

\newcommand{\fourIdx}[5]{%
\setbox1=\hbox{\ensuremath{^{#1}}}%
 \setbox2=\hbox{\ensuremath{_{#2}}}%
 \setbox5=\hbox{\ensuremath{#5}}%
 \hspace{\ifnum\wd1>\wd2\wd1\else\wd2\fi}%
 \ensuremath{\copy5^{\hspace{-\wd1}\hspace{-\wd5}#1\hspace{\wd5}#3}%
 _{\hspace{-\wd2}\hspace{-\wd5}#2\hspace{\wd5}#4}%
 }}

\numberwithin{equation}{section}
\numberwithin{theorem}{section}

\renewcommand{\subset}{\subseteq}

\usepackage{subcaption}

\usepackage{graphicx}

\begin{document}

\title{Bayesian Learning with Wasserstein Barycenters}
%\author{  Julio Backhoff-Veraguas, 
 %Joaquin Fontbona, Gonzalo Rios, Felipe Tobar}

 \author{Julio Backhoff-Veraguas}
\address{Faculty of Mathematics, University of Vienna}
\email{julio.backhoff@univie.ac.at}
%\urladdr{www.math.sc.edu/$\sim$howard} 

\author{Joaquin Fontbona}
\address{Center for Mathematical Modeling, and Department of Mathematical Engineering, Universidad de Chile}
\email{fontbona@dim.uchile.cl}
%\urladdr{www.math.sc.edu/$\sim$howard} 

\author{Gonzalo Rios}
\address{NoiseGrasp SpA}
\email{grios@noisegrasp.com}
%\urladdr{www.math.sc.edu/$\sim$howard} 

\author{Felipe Tobar}
\address{Initiative for Data \& Artificial Intelligence, and Center for Mathematical Modeling, Universidad de Chile}
\email{ftobar@dim.uchile.cl}
\urladdr{http://www.dim.uchile.cl/~ftobar/}

 \begin{abstract} We introduce and study a novel model-selection strategy for Bayesian learning,  based on optimal transport, along with its associated predictive posterior law: the Wasserstein population barycenter of the posterior law over models.  We first show  how this estimator, termed Bayesian Wasserstein barycenter (BWB),  arises  naturally in a  general,  parameter-free Bayesian model-selection framework,  when the considered Bayesian risk is the Wasserstein distance. Examples are given,  illustrating how the BWB  extends some classic parametric and non-parametric selection strategies. Furthermore, we also provide explicit conditions granting the existence and statistical consistency of the BWB, and discuss some of its general and specific properties, providing insights into its advantages  compared to usual choices, such as the model average estimator.  Finally, we illustrate how this estimator can be computed  using the  stochastic gradient descent (SGD) algorithm in Wasserstein space introduced in a companion paper \cite{bakchoff2022}, and provide a  numerical example for experimental validation of the proposed method.

\end{abstract}



\maketitle
%%-----------------------------
%%      your text
%%-----------------------------
\section*{Introduction}

  Given a set $\mathcal{M}$  of probability distributions on some data space $  \mathcal{X}$,  learning a model $m\in\mathcal M$ from data points  $D=(x_{1},\dots,x_{n}) \in  \mathcal{X}^n $    consists in choosing, under a given criterion, an element  of  $\mathcal{M}$ that \emph{best} explains  $D$  as data sampled from it.
  The Bayesian paradigm provides a probabilistic approach to deal with model uncertainty in terms of a {\it prior distribution} on models $\mathcal{M}$, and also furnishes strategies to address  the problem of model selection, \black{based on the \emph{posterior distribution}  on $\mathcal{M}$ given $D$}.  This type of estimators,  usually called  \emph{predictive posterior laws}, include  
  classical  Bayesian estimators such as  the maximum a posteriori estimator (MAP), the posterior mean, the Bayesian model average estimator (BMA) and  generalizations thereof.  Predictive posterior estimators  typically result from  selection criteria  consisting in optimizing some loss function, averaged with respect to the posterior law over models, or Bayesian risk function. We refer the reader to \cite{ghosal2017fundamentals, murphybook2012} and references therein for mathematical background on Bayesian statistics and \black{their use in the machine learning community}. 
  
  \black{ In Section \ref{sec:bayesian}, we will formulate the general problem of Bayesian model selection directly on  the space of probability measures (or models) on the data space,  and show how this abstract  framework  covers both   classic finitely-parametrized settings and parameter-free model spaces,  allowing  us to retrieve  classical selection criteria as particular cases.  
An eye-opening  observation that will follow  from  adopting this viewpoint is that many classical predictive posteriors can be seen as instances of \emph{Fr\'echet means} \cite{frechet1948elements}, or barycenters in the space of probability measures,  with respect to specific metrics or divergences between them, that play the role of abstract loss functions defined  on the model space.}
 
 \black{ Building upon this general framework,  the  main  goals of this work are to introduce  a novel Bayesian  model-selection criterion by proposing a  loss function on models coming from the theory of optimal transport, and to study some of the distinctive features  of  the  predictive posterior law that results from it. 
 More precisely, let us consider  observations $D=(x_1,\dots,x_n)$ in  a metric space $( \mathcal X,d)$ and a set of candidate models $\mathcal M$ that generated these observations. Equipping the set $\mathcal M$ with a prior distribution $\Pi$, and denoting the corresponding posterior distribution over $\mathcal M$ by $\Pi_n$, we will define the
\emph{Bayesian Wasserstein barycenter estimator} (BWB) as a minimizer $\hat m_p^n\in\mathcal M$ of the ``risk  function'' } 
	\begin{align}
	\textstyle\mathcal M \ni m \mapsto  \int_{\mathcal{P}(\mathcal X)} W_{p}(m,\bar{m})^{p}\Pi_n(d\bar{m}),\label{eq intro bary}
	\end{align}
 where  $\mathcal{P}(\mathcal X)$ is  the set of probability measures on $\mathcal{X}$,   $p\geq 1$ and $W_p$ is the celebrated $p$-Wasserstein distance on probability measures on $\mathcal{X}$ associated with $d$, see \cite{villani2003topics,villani2008optimal}. 

Wasserstein barycenters were initially studied in \cite{agueh2011barycenters} and, since then, the concept has been extensively explored from both theoretical and practical perspectives. We refer the reader to the overview \cite{PaZe20} for statistical applications, to the works  \cite{cuturi2014fast,CuPe16, peyre2019computational,NIPS2021} for applications in machine learning, and to  \cite{bigot2018characterization,le2017existence, alvarez2016fixed,alvarez2018wide} for a presentation of recent developments and further references. \black{As a cautionary tale, we mention that  the problem of computing Wasserstein barycenters is known to be NP-hard, see \cite{AlBA21}. To cope with this, fast  methods aiming at learning and generating approximate  Wasserstein barycenters on the basis of  neural networks  techniques, have also been proposed,  see e.g.\ \cite{LearningWB,WIterative}. }

 \black{In Section \ref{sec:BWBE}  we will recall  Wasserstein distances and revisit Wasserstein  barycenters together with their basic properties such as existence, uniqueness and absolute continuity. In Section \ref{sec consistency},  we will rigorously introduce  the  BWB estimator $\hat m_p^n$, which  will correspond to the so-called 
\emph{population Wasserstein barycenter}  \cite{le2017existence} for the posterior distribution on models $\Pi_n$, and  we will state some of its main properties. Specifically, we will  show that the BWB has  less variance than the BMA and we will study its statistical  consistency. In particular, we will address  the question of ``posterior consistency'',  or asymptotic concentration of the posteriors $\Pi_n$ around the Dirac mass on a model $m_0$ as $n\to \infty$, whenever the data consists of  i.i.d.~observations following the  law $m_0$, alongside the  question of convergence of the BWB  to the 
``true'' distribution $m_0$  of the data in that setting. We refer the reader to \cite{diaconis1986consistency,ghosal2017fundamentals}, and references therein, for detailed accounts on  posterior consistency,  a highly desirable feature of a Bayesian estimation procedure, both from a semi-frequentist perspective as well as from the ``merging of opinions'' point of view on Bayesian statistics (cf.\ \cite[Chapter 6]{ghosal2017fundamentals}).   After reviewing central notions and tools from the framework of posterior consistency, namely the celebrated Schwartz' theorem \cite{schwartz1965bayes}, \cite[Theorem  6.17]{ghosal2017fundamentals} and the notion of  Kullback-Leibler support of $\Pi$,    we will  provide in that section equivalent and (verifiable) sufficient conditions for both the posterior consistency in the Wasserstein topology and the a.s. convergence 
\begin{align*}
\lim_{n\to\infty}W_p(\hat m_p^n,m_0)= 0\,\,\,\mbox{a.s.} 
\end{align*} 
to hold, when  $\hat m_p^n$ is the BWB computed with $n$ i.i.d. observations sampled from $m_0$.}

 \black{Additionally, Sections 1 to 3 present a series of examples illustrating the main concepts of our work, their relationship
to  standard objects in Bayesian statistics and the applicability  of our theoretical results.}

Lastly,  we will show how the BWB estimator can be calculated using a novel  stochastic algorithm,  introduced in the companion paper  \cite{bakchoff2022},  to compute population Wasserstein barycenters in a general setting. This algorithm, presented in Section \ref{sec:computation},   can be seen as an abstract  \emph{stochastic gradient descent method in the Wasserstein space} and is advantageous compared to gradient or fixed-point algorithms developed in \cite{alvarez2016fixed,alvarez2018wide,panaretos2017frechet}, whose application is restricted  to barycenters of model spaces  ${\cal M}$ comprised of finitely-many elements only.  Moreover, our algorithm 
has theoretical guarantees of convergence under suitable conditions, it can be  easily implemented for some families of regular models for which optimal transport maps are explicit or easily computed (we recall one such family in Section \ref{sec scat loc}), \black{and its convergence rate can be studied and established in some cases, see \cite{chewi2020gradient} for  the Gaussian setting.}  A comprehensive numerical experiment illustrating this method, and its natural ``batch'' variants,   will be presented in Section \ref{sec numeric} and compared  to (more conventional) empirical barycenter estimators.

{\bf Notation: }  
\begin{itemize}
    \item We denote $\mathcal{P}(\mathcal{X})$  the set of (Borel) probability measures on  $\mathcal{X}$ endowed with the weak topology, and  $\mathcal{P}_{ac}(\mathcal{X})$  the (measurable) subset of absolutely continuous probability measures, with respect to a common  reference $\sigma$-finite measure $\lambda$ on $\mathcal{X}$.
\item As a convention, we shall use the same notation for an element  $m(dx)\in \mathcal{P}_{ac}(\mathcal{X})$ and its density $m(x)$ with respect to $\lambda$. 
\item  We denote by $\text{supp}(\nu)$ the support of a measure $\nu$ and by $|\text{supp}(\nu)|$ its cardinality. 
\item  \black{Given $\Gamma \in \mathcal{P}(\mathcal{P}(\mathcal{X})) $ and a measurable subset  $
\mathcal{M} \subseteq \mathcal{P}(\mathcal{X}), 
$  we say that $\mathcal{M}$ is a {\it model space for } $\Gamma$ if $\Gamma({\cal M})=1$.} 

\item  Last, given a measurable map $T:\mathcal Y\to\mathcal Z$ and a measure $\nu$ on $\mathcal Y$ we denote by $\nu_{\sharp}T $ the image measure (or push-forward), that is,  the measure on $\mathcal Z$ given by $\nu_{\sharp}T(\cdot)=\nu(T^{-1}(\cdot))$. 

\end{itemize}

\section{Bayesian learning in model space}
\label{sec:bayesian}

We start by setting    a general framework for Bayesian learning  which  covers both finitely-parametrized  settings  (including hierarchical models) and parameter-free models.  \black{We consider  a  probability measure $\Pi\in \mathcal{P}( \mathcal{P}(\mathcal{X})) $ understood as a {\it prior} distribution on the model space $\mathcal{M} \subseteq \mathcal{P}_{ac}( \mathcal{X}) $}. In  particular we have $$\Pi(\mathcal{M}) =\Pi( \mathcal{P}_{ac}(\mathcal{X}))= 1.$$
For each  $n\in \NN\backslash \{0\}$,  $\Pi$ canonically induces a law $\mathbf{\Pi}$  on  $ \mathcal{X}^n \times \mathcal{M}$, representing the joint law of a random model $m$ chosen according to $\Pi$ and a sample $$D=(x_{1},\dots,x_{n})\subseteq \mathcal{X}^n$$ of  i.i.d.\ observations drawn from it. That is,  
$$
\mathbf{\Pi} (dx_{1},\dots,dx_{n},  dm):=  m\left( dx_{1}\right)
	\cdots m\left( dx_{n}\right) \Pi (dm)  =    m\left( x_{1}\right)
	\cdots m\left( x_{n}\right) \lambda (dx_1)\cdots \lambda (dx_n) \Pi (dm). 
$$
\black{ Note that, in the above equation and throughout,   $\Pi(dm)$ denotes integration over  $m\in {\cal M}$ w.r.t.\ $\Pi \in \mathcal{P}( \mathcal{P}(\mathcal{X})) $, whereas integration over $x\in {\cal X}$ w.r.t.\ $m \in  \mathcal{P}(\mathcal{X})  $ is denoted $ m(dx)= m(x) \lambda(dx)$. }
The law on $\mathcal{X}^n$ of the data $D$, conditionally on a model $m$,  is thus given by
\begin{align}
	\mathbf{\Pi} (dx_{1},\dots,dx_{n}\vert m): =  m\left( x_{1}\right)
	\cdots m\left( x_{n}\right) \lambda (dx_1)\cdots \lambda (dx_n),  
\end{align}
with density $\mathbf{\Pi} \left(
		x_{1},\dots,x_{n}|m\right)=m\left( x_{1}\right)
	\cdots m\left( x_{n}\right) $ with respect to $\lambda^{\otimes n} $, 
and the marginal density of $D$ with respect  to $\lambda^{\otimes n} $ is $\mathbf{\Pi} \left(x_{1},\dots,x_{n}\right):=\int_{\mathcal{M}}\bar{m}\left( x_{1}\right) \cdots\bar{m}\left( x_{n}\right)  \Pi \left( d\bar{m}\right) $.

\black{ The {\it posterior distribution} $ \mathbf{\Pi} (dm|x_{1},\dots,x_{n})$  given the data $D$ is also an element of  $ {\cal P}({\cal P}({\cal X}))$, which we denote $\Pi_n$ for simplicity  and which,  by  virtue of the Bayes rule,  is given in this setting by} 
\begin{align}\label{eq def Pi n}
	\Pi_n(dm) := \frac{\mathbf{\Pi} \left(
		x_{1},\dots,x_{n}|m\right) \Pi \left( dm\right) }{\mathbf{\Pi} \left(
		x_{1},\dots,x_{n}\right) }=\frac{m\left( x_{1}\right)
		\cdots m\left( x_{n}\right) \Pi \left( dm\right) }{\int_{\mathcal{M}}\bar{m}\left(
		x_{1}\right) \cdots\bar{m}\left( x_{n}\right)  \Pi \left( d\bar{m}\right) } \, .
\end{align}
\black{Notice that  $ (x_1,\dots,x_n)\mapsto \Pi_n=  \mathbf{\Pi} (\cdot|x_{1},\dots,x_{n})$ defines $\lambda^{\otimes n} $ a.e.\ a measurable function from ${\cal X}^n$ to $ {\cal P}({\cal P}({\cal X}))$}. 
The density $\Lambda_n(m)$ of $\Pi_n(dm)$ with respect the prior $\Pi(dm)$ is called the \black{{\it  likelihood} function.} The fact that  $\Pi_n\ll \Pi$  implies that a model space ${\cal M}$ for $ \Pi$    is a model space for $\Pi_n$ too.
%\black{[FT: solo un detalle: me hace un poquito de ruido hablar recién de likelihood ahora y no en la eq 1.1. Se podría decir que la densidad de la eq 1.1 es likelihood?]}

\smallskip 

We call {\it loss function} a  non-negative functional on models $L:\mathcal{M} \times \mathcal{M} \rightarrow \mathbb{R}$,  interpreting $L(m_0 ,\bar{m})$ as the cost of selecting  model $\bar{m} \in \mathcal{M}$ when the true model is $m_0 \in \mathcal{M}$. With a loss function and the posterior distribution over models $\Pi_n$, the Bayes risk (or expected loss) $R_L(\bar{m}|D)$ and the corresponding Bayes estimator $\hat{m}_L$  (or predictive posterior law) are respectively defined as follows:
\begin{eqnarray}
R_L(\bar{m}|D) &:=&  \textstyle\int_{\mathcal{M}} L(m,\bar{m})\Pi_n(dm) \, , \label{eq RL}\\
\hat{m}_L &\in&\textstyle \argmin_{\bar{m} \in \mathcal{M}} R_L(\bar{m}|D). \label{eq hat m abstract}
\end{eqnarray}
See \cite{berger2013statistical} for further background on Bayes risk and statistical decision theory. 
\black{A key consequence of defining both $L$ and $\Pi_n$ directly on the model space $\mathcal{M}$ (rather than on  parameter space), is that learning according to eqs.~\eqref{eq RL}-\eqref{eq hat m abstract} does not depend on the chosen parametrization or the geometry of the parameter space.}  Moreover, this point of view will allow us  to define loss functions  in terms of various metrics/divergences directly on the space $ \mathcal{P}( \mathcal{X})$, and therefore to enhance the classical  Bayesian  estimation framework through the use of optimal transportation distances on that space. Before further developing these ideas, we discuss how this general framework includes model spaces which are finitely parametrized, and recall some standard choices in that setting. \black{We will also discuss the advantages of formulating the problem of model selection directly on the model space, even when  this space can be finitely parametrized.}

%The parametric setting is useful as well, since it helps to illustrate the drawbacks of maximum likelihood estimation (MLE). If the reader is already comfortable with the present non-parametric setup, and is aware of the drawbacks of MLE, he or she may skip Section \ref{sec param} altogether.

\subsection{Parametric setting}\label{sec param}
\black{We say that $\mathcal{M}$ is finitely parametrized if there is an integer $k$, a measurable set $\Theta \subseteq \mathbb{R}^{k}$  termed  parameter space,  and a measurable function  $\mathcal{T}:\Theta \mapsto \mathcal{P}_{ac}( \mathcal{X}) $, called parametrization,  such that $\mathcal{M=T}( \Theta )$. In other words,   $m =\mathcal{T}(\theta)$ is the model corresponding to parameter $\theta$, which is classically denoted
$p(\cdot\vert \theta)$ or $p_\theta(\cdot)$.  In general,  $\mathcal{T}$ is a one-to-one function   (the model space is otherwise said to be {\it over-parametrized}).}

\black{In the standard parametric Bayesian framework,  
a prior distribution    is a probability law $p \in \mathcal{P}(\Theta)$ over $\Theta $, typically assumed to have a (equally denoted) density $p(\theta)$ with respect to the Lebesgue measure. 
A function $  \ell:\Theta\times \Theta \to \R_+$ is called a loss function (on parameters), whereby    $\ell(\theta_0,\bar{\theta})$ is  interpreted as the cost of choosing parameter $\bar{\theta}$ when the true parameter is $\theta_0$.
The parametric Bayes risk \cite{berger2013statistical} of $\bar{\theta} \in \Theta$ is then given by}
\begin{align}
	\label{eq:param_bayes_risk}
	\textstyle R_{\ell} (\bar{\theta}|D) = \int_{\Theta} \ell (\theta ,\bar{\theta})p(\theta|x_1,\dots,x_n) d\theta,  %=  \int_{\mathcal{M}} L(m,\bar{m})\Pi_n(dm),
\end{align} 
where $p(\theta|  x_1,\dots, x_n)  $   is the posterior density of $\theta$ given observations $ x_1,\dots, x_n$. The  associated Bayes estimator is defined as $\hat{\theta}_{\ell} \in\argmin_{\bar{\theta} \in \Theta }R_{\ell}(\bar{\theta}|D)$.  
\black{ Hence, if the model space $\mathcal M$ is finitely parametrized, learning a model boils down to finding the \emph{best} model parameter $\theta\in \Theta$ under a given criterion, quantified by the  parametric risk function  $  R_{\ell} (\cdot|D)$. Among continuous-valued losses on the parameter space, the \textit{de facto} choice is the quadratic one $\ell_{2}( \theta ,\bar{\theta}) =\Vert \theta -\bar{\theta}\Vert ^{2}$, whose  associated Bayes estimator is the posterior mean $\hat{\theta}_{\ell_{2}}=\int_{\Theta }\theta p( \theta|D)d\theta$. 
For one-dimensional parameter spaces, the absolute loss $\ell_{1}( \theta ,\bar{\theta}) =| \theta -\bar{\theta}| $ yields the posterior median(s) estimator(s). 
The 0-1 loss formally given by  $\ell_{0-1}( \theta ,\bar{\theta}) =1-\delta_{\bar{\theta}}( \theta )$,  with $\delta_{\bar{\theta}}$ the Dirac mass at $\bar{\theta}$,  yields the risk $R_{\ell_{0-1}}(\bar{\theta}|D)=1-p(\bar{\theta} |D) $ and  its corresponding Bayes estimator is the posterior mode or maximum a posteriori estimator (MAP), $\hat{\theta}_{\ell_{0-1}}=\hat{\theta}_{MAP}$. }

\smallskip 

 \black{The parametric case is  embedded into the considered model-space setting as follows. The push-forward  of $p$ through  $ \mathcal{T} $  defines a prior $\Pi= p_{\sharp}\mathcal{T}$   over the model space $\mathcal{M}$ in the sense discussed at the beginning of this section. If $\mathcal{T}$ is one to one, a loss function $\ell:\Theta\times \Theta\to \RR_+$ induces a loss function $L$ on the model space  $\mathcal{M=T}( \Theta )$, such that $L (\mathcal{T}(\theta_0), \mathcal{T}(  \bar{\theta}))=   \ell (\theta_0,\bar{\theta})$. More generally, any loss function $L$ on $\mathcal{M}\times \mathcal{M}$ induces a loss  functional $\ell$ on $\Theta\times \Theta$   { defined} as  $  \ell (\theta_0,\bar{\theta}):  =L (\mathcal{T}(\theta_0), \mathcal{T}(  \bar{\theta}))$. Moreover, when the data $x_1,\ldots,x_n$ under the model parameterized by $\theta$ consists of an i.i.d.~sample from $p(\cdot\vert \theta) ={\cal T}(\theta)$,  one can verify that $\Pi_n$ given in eq.~\eqref{eq def Pi n} corresponds precisely to the  the push-forward through $\mathcal{T}$ of $p(\theta|  x_1,\ldots, x_n)  $  and that $R_{\ell} (\bar{\theta}|D)$  in  eq.~\eqref{eq:param_bayes_risk} is given by 
\begin{align*}
	\textstyle R_{\ell} (\bar{\theta}|D) =   \int_{\mathcal{M}} L(m,\bar{m})\Pi_n(dm),
\end{align*}
with  $ \Pi_n(dm) = \Lambda_n(m) \Pi (dm)$ associated with the   prior  on model space $\Pi=p_{\sharp}\mathcal{T}$, and $\bar{m}= \mathcal{T}(  \bar{\theta})$.   } 

\smallskip 

\begin{example}\label{ExParaModel1}
Consider the parametric Bayesian model with parameter space $\Theta$ and sample space   ${\cal X}$ both equal to $\RR^d$ and Gaussian parametrized models ${\cal N}(\theta, \Sigma)$, with  $\Sigma\in\RR^{d\times d}$ a fixed covariance matrix,  and  $\theta$ a random mean with prior  
${\cal N}(\mu_0, \Sigma_0)$. The mean $\mu_0\in \RR^d$ and convariance matrix $ \Sigma_0$ are fixed  hyperparameteres. 
 This is classically denoted: 
 $$  p(x\vert \theta) = {\cal N}(x;\theta, \Sigma)    \,,   \quad  p(\theta)= p(\theta\vert \mu_0, \Sigma_0)=  {\cal N}(\theta; \mu_0, \Sigma_0)  . $$
 From now on,  ${\cal N}(y; \nu, {\cal K})$ stands for the density of the Gaussian law ${\cal N}(\nu, {\cal K}) $ evaluated on the value $y$. 
 Following from the introduced notation, the parametrization ${\cal T}:\Theta \to  {\cal P}({\cal X})$
is thus given by ${\cal T}(\theta)={\cal N}(\theta, \Sigma) $,  the model space is  ${\cal M}=\{{\cal N}(\theta, \Sigma)  : \theta \in \RR^d\}$   and a measure $m$ sampled from the prior $\Pi(dm)= p_\sharp {\cal T}(dm)  $ is a  Gaussian distribution on ${\cal X}=\RR^d$, with fixed covariance matrix $\Sigma$ and random mean $\theta$ distributed according to ${\cal N}(\mu_0, \Sigma_0)  $. In this case, if both $\Sigma$ and $\Sigma_0$ are nonsingular, the posterior of $\theta$ given $D=(x_{1},\dots,x_{n})$ is
\begin{equation}
    p(\theta\vert D)= p(\theta\vert D, \mu_0,\Sigma_0)= {\cal N}\left(\theta;  (\Sigma_0^{-1}+ n \Sigma^{-1})^{-1}  (\Sigma_0^{-1}\mu_0 +  n \Sigma^{-1} \bar{x}_n), (\Sigma_0^{-1}+ n \Sigma^{-1})^{-1}   \right), \label{eq:post_param_ex1.1}
\end{equation}
with $\bar{x}_n$ the sample mean of the observations $D$.  A model $m$ sampled from the posterior  $\Pi_n(dm)$ given $D$  is thus obtained by sampling $\theta$  distributed according to  $p(\theta\vert D)$ in eq.~\eqref{eq:post_param_ex1.1} and then setting $ m={\cal N}(\theta, \Sigma)$. The posterior mean estimator of the parameter $\theta$ is  therefore given by 
\begin{equation}\label{eq:postmeangaus}
\hat{\theta}_{\ell_2}= (\Sigma_0^{-1}+ n \Sigma^{-1})^{-1}  (\Sigma_0^{-1}\mu_0 +  n \Sigma^{-1} \bar{x}_n), 
\end{equation}
\black{and the corresponding predictive posterior model is the Gaussian law ${\cal N}(\hat{\theta}_{\ell_2}, \Sigma)$. }
Observe that, in this case, one could equivalently obtain $\hat{\theta}_{\ell_2}$ as the mean of the predictive posterior law $\hat{m}_L$ associated with the loss function $L$ on models, defined by $L(m,\bar{m})= \Vert \int_{\RR^d} x \, m(dx)- \int_{\RR^d} x \, \bar{m}(dx)\Vert^2. $ \black{This illustrates the equivalence of (some) pairs of losses in the model and parameter spaces that result on the same Bayes estimators. Additionally, observe that Bayesian inference on the mean and the covariance matrix could be similarly formulated in terms of a loss functions on models too, take e.g. $L'(m,\bar{m})=L(m,\bar{m}) + \Vert \Sigma_m^{1/2} - \Sigma_{\bar m}^{1/2} \Vert^2_{\text{Fro}} $ with $L(m,\bar{m})$ as before, $\Sigma_m$ the covariance matrix of a r.v.~with law $m$ and $\Vert\cdot  \Vert_{\text{Fro}}$ the Frobenius norm on matrices}.
\end{example} 

\black{The general model space framework applies equally to models with parameter and data spaces that might be of different nature:}
\begin{example}\label{ExParaModel2}
\black{Assume $\Theta=\RR_+$ and ${\cal X}=\NN$, with 
 $p(x|\lambda)= e^{-\lambda} \frac{\lambda^x}{x!}, x\in \cal X$, and $p(\lambda)\propto e^{-\beta \lambda} \lambda^{\alpha-1}, \lambda\in \Theta. $
 In this case, ${\cal T}:\Theta \to  {\cal P}({\cal X})$ is given by ${\cal T}(\lambda) =$Pois$(\lambda)$, the  Poisson distribution of parameter $\lambda$, and a measure  $m$ sampled from the prior $\Pi(dm)=p_\sharp {\cal T}(dm)  $ is a Poisson law with random parameter  $\lambda$ distributed according to the  Gamma$(\alpha,\beta)$ law. The latter is a conjugate prior for the Poisson distribution, and $m$ sampled from $\Pi_n(dm)$, the posterior on models   given  $D=(x_1,\dots,x_n)$,  is  again a Poisson law Pois$(\lambda)$,  with random parameter $\lambda$ distributed according to Gamma$(n\bar{x}_n+ \alpha,n+\beta)$.} 
\end{example}

\begin{remark}Hierarchical models are  also catered for in the proposed setting. For instance, if in Example \ref{ExParaModel1} the hyperparameter $\mu_0$ is random with known \black{density} $\pi_0$ on $\RR^d$, it can be integrated out and the prior $p$ on parameters becomes an infinite  Gaussian mixture: 
\begin{equation*} p(\theta)= p(\theta\vert  \Sigma_0)=  \int {\cal N}(\theta;\mu_0,\Sigma)\pi_0(\mu_0)d\mu_0.\label{eq:prior_hierarchical_Ex1.1}
\end{equation*}
The parematrization mapping   ${\cal T}$ is in this case the same as before, and the corresponding prior   $\Pi(dm)=p_\sharp {\cal T}(dm) $  and posteriors $\Pi_n(dm)$ on models  $m\in {\cal P}( {\cal X})$ follow the same rationale as above.  
\end{remark}

 \black{ As a cautionary note, the  following example illustrates how defining estimators   directly  in terms of parameters  might result in non-intrinsic  criteria for model selection.  }

\begin{example} \black{On the parameter space $\Theta=[0,1]$ consider the priors $p(d\theta)=d\theta$ and $\hat p(d\theta)=2\theta d\theta$, and their associated  parametrization maps $\mathcal{T} (\theta):=B(\theta)$ and  $\hat{\mathcal{T}}:=B(\theta^2)$ respectively. Here $B(\xi)$ denotes the law of a $\{0,1\}$-valued Bernoulli r.v.\ with $\xi$ the probability of it being equal to $1$. Notice that we have $\Pi:=p_\sharp \mathcal{T}=\hat{p}_\sharp \hat{\mathcal{T}}$, as the law of $\theta^2$ under $\hat p$ is uniform on $\Theta$. Starting from the prior $p$, the posterior density of $\theta$ given observations $x_1,\dots,x_N$ is proportional to $\theta^{S_N}(1-\theta)^{N-S_N}$, with $S_N=|\{i\leq N:x_i=1\}|$. Thus the MAP estimator for $\theta$ is in this case $\theta_N:=\frac{S_N}{N}$. On the other hand, starting from the prior $\hat p$, the posterior density of $\theta$ is proportional to $\theta^{2S_N+1}(1-\theta^2)^{N-S_N}$ and now the associated MAP estimator for $\theta$ is $\hat \theta_N:=\sqrt{\frac{2S_N+1}{2N+1}}$. Hence $\mathcal{T}(\theta_N)\neq \hat{\mathcal{T}}(\hat\theta_N)$ in general (although their discrepancy vanishes in the limit as $N\to\infty$). To summarize, although the same prior and posteriors at the level of models can arise by considering different parametrizations, the latter may easily define very different estimated models even if we agree on the estimation method (here the MAP).} 
\end{example}

\subsection{Non parametric setting: posterior average estimators}\label{sec posterior estimators gnral}
\black{The general  ``learning in model space'' approach relies  on loss functions that compare directly distributions   (instead of their parameters), and thus allows us to define selection criteria based on intrinsic features of the models.  It also allows for a wider choice of  model-selection criteria, which can account for  geometric or information-theoretic properties of the models.} The next result illustrates the fact that many examples of Bayesian estimators or predictive posterior, including the classical {\it model average estimator}, correspond to  finding an instance  of  {\it Fr\'echet mean}  or barycenter \cite{frechet1948elements,panaretos2017frechet}  under a suitable metric/divergence on probability measures.  See  Appendix \ref{subsec model_averages} for the proof.

\begin{proposition}\label{prop EL}
	Consider on the model $\mathcal{M} = \mathcal{P}_{ac}(\mathcal X)$ the loss functions $L(m,\bar{m})$ given by:
		\begin{itemize}
		\item[i)]  The 	$L_{2}$-distance:  $	L(m,\bar{m})=		L_{2}(m,\bar{m}):=\frac{1}{2}\int_{\mathcal{X}}\left( m(x)-\bar{m}(x)\right) ^{2}\lambda (dx), $
		\item[ii)] The squared Hellinger distance $L(m,\bar{m})=	
		H^{2}(m,\bar{m}):= \frac{1}{2}\int_{\mathcal{X}}\left( \sqrt{m(x)}-\sqrt{\bar{m}(x)}\right) ^{2}\lambda(dx). 
		$ 
		\item[iii)] The forward Kullback-Leibler divergence: 
		$	L(m,\bar{m})=	 D_{KL}(m||\bar{m}):=\int_{\mathcal{X}}m(x)\ln \frac{m(x)}{\bar{m}(x)}\lambda(dx),$ 
		\item[iv)]  The reverse Kullback-Leibler divergence  $L(m,\bar{m})=	 D_{RKL}(m|| \bar{m})=D_{KL} (\bar{m}||m): = \int_{\mathcal{X}}\bar{m}(x)\ln \frac{\bar{m}(x)}{m(x)}\lambda(dx).$

	\end{itemize}
	Assume in each case that the infimum of the corresponding Bayes risk $\bar{m}\mapsto R_L(\bar{m}|D)$ defined in eq.~\eqref{eq RL} is attained  with finite value in ${\cal M}$. 
	Then,  in cases i)   and iii)  the corresponding Bayes estimators  \eqref{eq hat m abstract} coincide with the standard \emph{Bayesian model average}:
	\begin{align}
		\label{eq:model_average}
		\textstyle \hat{m}_{BMA}(x):=\mathbb{E}_{\Pi_n}[m](x) =\int_{\mathcal{M}}m(x)\Pi_n(dm).
	\end{align}
	Furthermore,  the  Bayes estimators corresponding to the cases ii)  and iv) are given by  the \emph{square model average} and the \emph{exponential model average},  respectively: 
	\begin{align}
		\label{eq:general_model_average}
	\hat{m}_{sqr}(x) =\frac{1}{Z_{sqr}}\left( \int_{\mathcal{M}}\sqrt{m(x)}\Pi_n(dm)\right)^{2}	\,\, ,\,\,\,\,\,\,\,
		\textstyle\hat{m}_{exp}(x)=\frac{1}{Z_{exp}}\exp \int_{\mathcal{M}}\ln m(x)\Pi_n(dm), 
	\end{align}
	where $Z_{sqr}$ and $Z_{exp}$ denote the corresponding normalizing constants.
\end{proposition}

\begin{example}\label{ex:gaussvert}
If the posterior distribution was approximately  equally concentrated on the   models $m_{0}=\mathcal{N}(\mu_0,1)$ and $m_{1}=\mathcal{N}(\mu_1,1)$ with $\mu_0\neq \mu_1$, that is, two  (unimodal) Gaussian distributions with  unit variance, then the standard model average is a  bimodal non-Gaussian distribution with variance strictly larger than $1$.
\end{example} 

\begin{example}\label{ex:gaussModAv}
In the parametric Bayesian model discussed in Example  \ref{ExParaModel1}, the Bayesian model average estimator is the convolution of distributions on $\RR^d$: 
\begin{equation}\label{eq:MAGaussian}
\begin{split}
\hat{m}_{BMA}(x)=  & \int_{\RR^d} {\cal N}(x; \theta, \Sigma) {\cal N}\left(\theta;  (\Sigma_0^{-1}+ n \Sigma^{-1})^{-1}  (\Sigma_0^{-1}\mu_0 +  n \Sigma^{-1} \bar{x}_n), (\Sigma_0^{-1}+ n \Sigma^{-1})^{-1}   \right) d\theta \\
=& \int_{\RR^d} {\cal N}(x- \theta;0 ,  \Sigma) {\cal N}\left(\theta;  (\Sigma_0^{-1}+ n \Sigma^{-1})^{-1}  (\Sigma_0^{-1}\mu_0 +  n \Sigma^{-1} \bar{x}_n), (\Sigma_0^{-1}+ n \Sigma^{-1})^{-1}   \right) d\theta \\
=& \, {\cal N}\left(x;  (\Sigma_0^{-1}+ n \Sigma^{-1})^{-1}  (\Sigma_0^{-1}\mu_0 +  n \Sigma^{-1} \bar{x}_n), \Sigma + (\Sigma_0^{-1}+ n \Sigma^{-1})^{-1}   \right),
\end{split}
\end{equation}
that is, a Gaussian density with mean equal to the posterior mean estimator $\hat{\theta}_{\ell_2}$  ---see eq.~\eqref{eq:postmeangaus}--- and a covariance matrix that is strictly larger (in the usual order on nonnegative definite symmetric matrices)  than that of the  predictive posterior law ${\cal N}(\hat{\theta}_{\ell_2}, \Sigma)$ associated with it. 
\end{example} 

\black{The Bayesian estimators  considered Proposition \ref{prop EL}, eqs.~\eqref{eq:model_average}-\eqref{eq:general_model_average}, share the following characteristic: their  values at each point  $x \in \mathcal{X}$ are computed  in terms of some posterior \emph{average} of the values of certain functions evaluated at  $x$. This is due to the fact that the corresponding distances/divergences on probability distributions are \emph{``vertical''} \cite{santambrogio2015optimal}: computing the distance between distributions $m$ and $\bar{m}$ involves the integration of vertical displacements between the  graphs of their  densities across their domain.  An undesirable fact about  \emph{vertical averages} is that they are not well suited to incorporate {\it geometric} properties into the model space (as illustrated by Example \ref{ex:gaussvert}). More generally, model averages might yield solutions that can be hardly interpretable in terms of the prior and parameters, or even be intractable. This motivates us to explore the use of ``horizontal'' distances between probability distributions, thus extending the concept of  Bayes estimator by making use of the geometric features of the model space. The next section presents the main  ideas we will rely on to build such Bayes estimators. }

\section{Wasserstein distances and Barycenters: a quick review}
\label{sec:BWBE}
We shall now introduce objects analogous to those in Proposition \ref{prop EL} but suited to {\it Wasserstein distances} on the space of probability measures. The following framework is adopted in the sequel:
\begin{assumption}\label{assum:geodesic} The metric space
$(\mathcal{X},d) $ is a separable locally-compact geodesic space endowed with a $\sigma$-finite Borel measure $\lambda$   and $p\geq 1$.
\end{assumption}
By \textit{geodesic} we mean that the space $(\mathcal{X},d) $ is complete and any pair of points admit a mid-point with respect to $d$. Next, we briefly recall some basic elements of  optimal transport theory and Wasserstein distances, referring to \cite{villani2003topics,villani2008optimal} for general background.
%, and to  the recent survey \cite{panaretos2018statistical} for some of statistical applications. 

\subsection{Optimal transport and the Wasserstein distance}
\label{sec:wassersteindist} 

Given two measures $\mu ,\upsilon $ over $\mathcal{X}$ we denote by $\cpl ( \mu ,\upsilon ) $ the set of transport plans or couplings with marginals $\mu $ and $\upsilon $, i.e., $\gamma \in \cpl (\mu,\upsilon)$ if and only if  $\gamma \in \mathcal{P}(\mathcal{X} \times \mathcal{X})$,  $\gamma ( dx,\mathcal{X}) =\mu(dx) $ and $\gamma ( \mathcal{X},dy) =\upsilon(dy) $. Given a real number $p\geq 1$ we define the $p$-Wasserstein space $\mathcal W_p(\mathcal X)$ by
\begin{align*}
\textstyle \mathcal W_p(\mathcal X)&\textstyle :=\left\{\eta\in \mathcal P(\mathcal X): \int_{\mathcal X}d(x_0,x)^p \eta(dx)< \infty,\,\,\text{some }x_0  \right\}.
\end{align*}
The $p$-Wasserstein distance between measures $\mu $ and $\upsilon $ is given by
\begin{align}
	\textstyle W_{p}( \mu ,\upsilon) = \left(\inf_{\gamma \in \cpl( \mu	,\upsilon ) }\int\limits_{\mathcal{X\times X}}d(x,y)^{p}\gamma (dx,dy) \right) ^{\frac{1}{p}}.\label{eq W p}
\end{align}
An optimizer of the right-hand side of eq.~\eqref{eq W p} always exists and is called an \textit{optimal transport}. The distance $W_p$ turns $\mathcal{W}_p(\mathcal{X})$ into a complete metric space.  If in eq.~\eqref{eq W p} we assume that $p=2$, $\mathcal X$ is the Euclidean space, and if $\mu$ is absolutely continuous, then Brenier's theorem \cite[Theorem 2.12(ii)]{villani2003topics} establishes the uniqueness of a minimizer,  and guarantees that it is supported on the graph of the subdifferential of a convex function.  The corresponding gradient is thus called  an {\it optimal transport map}.  Explicit formulae for such  optimal transport maps do exist in some cases, e.g., for generic one-dimensional distributions and  multivariate Gaussians when $p=2$ (see \cite{cuestaalbertos1993optimal}).
Contrary to the distances / divergences considered in Proposition \ref{prop EL},  Wasserstein distances  are \emph{horizontal} \cite{santambrogio2015optimal}, in the sense that they involve  integrating  horizontal displacements between the graphs of probability densities.

\begin{example}\label{ex:WdGauss}
\black{The squared $2$-Wasserstein distance between multivariate Gaussian distributions  $m={\cal N}(\theta, \Sigma)$, and $\bar{m}={\cal N}(\bar\theta, \bar{\Sigma}) $ is given  by $W_2^2(m,\bar{m})= \Vert \theta - \bar{\theta}\Vert^2+ Tr \left(  \Sigma+   \bar{\Sigma}  - 2 \left(\Sigma^{1/2}\bar{\Sigma}\Sigma^{1/2}\right)^{1/2}\right) $ ---see e.g.\ \cite{GivensShortt,DowsonLandau}. Furthermore, if $\Sigma$ and $\bar{\Sigma}$ commute, we have $W_2^2(m,\bar{m}) = \Vert \theta - \bar{\theta}\Vert^2+ \Vert \Sigma^{1/2}- \bar{\Sigma}^{1/2} \Vert^2_{\text{Fro}} $.}
\end{example}

\begin{example}\label{ex:WdPoisson}
\black{The $1$-Wasserstein distance between Poisson distributions $m=$Pois $(\lambda) $  and $\bar{m}=$ Pois $(\bar{\lambda}) $  is $|\lambda- \bar{\lambda}|$, which can be verified from the well-known expression  $W_1(m,\bar{m})= \int |m(-\infty, t]- \bar{m}(-\infty, t]| dt $ (valid for general one-dimensional distributions $m,\bar{m}$), and the so-called ``Poisson-Gamma dual relation'': $e^{-\lambda} \sum_{k=0}^n\frac{\lambda^k}{k!}= \int_{\lambda}^\infty \frac{t^n}{n!}e^{-t}dt  $. When $\lambda > \bar{\lambda}  $, the optimal coupling between $X\sim $ Pois $(\lambda) $ and $Y\sim$ Pois $(\bar{\lambda}) $ is obtained taking $X\sim $  Pois $(\lambda) $  and $Y$ binomial with parameters $(X,\frac{\bar{\lambda} }{\lambda} $) conditionally on $X$. Alternatively, one could consider the coupling $(X,Y):=(N_{\lambda},N_{\bar\lambda})$, where $\{N_t\}_{t\geq 0}$ is a Poisson process with intensity 1.}
\end{example}

\subsection{ Wasserstein barycenter}
\label{sec:wassersteinbary} 

Let us now recall the  Wasserstein barycenter, introduced in  \cite{agueh2011barycenters} and further studied in \cite{pass2013optimal,kim2017wasserstein,le2017existence}, among others. Our  definition slightly extends the ones in those works in that the optimization problem is posed in a possibly strict subset of the usual one.  

\begin{definition}\label{defi bary pop}
	Let  $\Gamma \in \mathcal{P}(\mathcal{P}(\mathcal{X}))$. The $p$-Wasserstein risk of $\nu \in \mathcal{P}( \mathcal{X})$ is 
	\begin{align*}
	V_p(\nu) &:=  \int_{\mathcal{P}(\mathcal X)} W_{p}(m,\nu)^{p}\Gamma(dm)\leq +\infty.
	\end{align*}
Given a measurable set ${\cal M}\subseteq {\mathcal{P}(\mathcal X)}$, any measure  $\hat{m}_p\in \mathcal M$ which  \black{attains the quantity  $$ \inf_{\nu \in \mathcal{M}} V_p(\nu)$$
with finite value,} 
is called a $p$-Wasserstein  barycenter of $\Gamma$ over $\mathcal M$.
\end{definition} 
\black{Notice that in principle we are not assuming  ${\cal M}$ to be a model space for $\Gamma$, but  this will often be the case}. When the support of $\Gamma$ is infinite and $\mathcal{M}=\mathcal{W}_p(\mathcal{X})$, this object is termed {\it $p$-Wasserstein population barycenter}  of $\Gamma$ as introduced in \cite{le2017existence}; see  \cite{bigot2018characterization}.

\begin{example}\label{ex:gausshoriz} 
Given two univariate Gaussian distributions $m_0 = \mathcal{N}(\mu_0,\sigma_0^2)$ and $m_1 = \mathcal{N}(\mu_1,\sigma_1^2)$, one can verify ---using the expression in Example \ref{ex:WdGauss}--- that the $2$-Wasserstein barycenter for $\Gamma(dm)=\frac{1}{2} \delta_{m_0}(dm)+ \frac{1}{2} \delta_{m_1}(dm) $  is given by  $\hat{m} =  \mathcal{N}(\frac{\mu_0 + \mu_1}{2},(\frac{\sigma_0 + \sigma_1}{2})^2)$. This should be compared to Example \ref{ex:gaussvert}. Fig.~\ref{fig:interpolation}  illustrates  the corresponding vertical and a horizontal interpolations between two Gaussian densities with different means and the same variance.
\begin{figure}[ht]
	\includegraphics[width=0.49\textwidth]{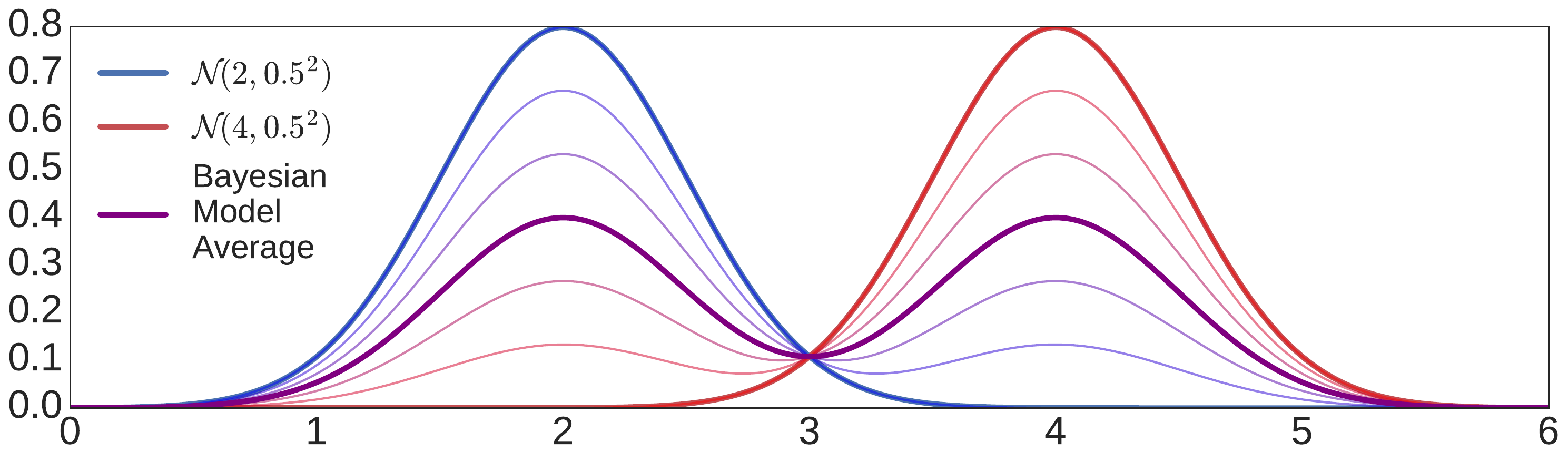}
	\includegraphics[width=0.49\textwidth]{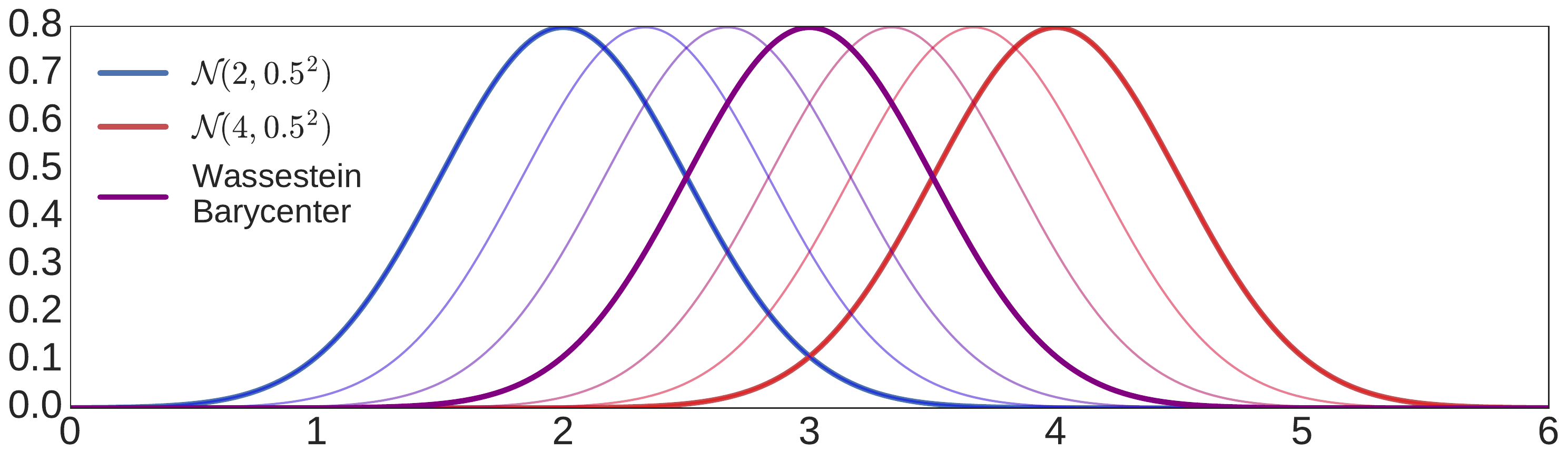}
	\vspace{-0.5em}
	\caption{Vertical interpolation (left) and horizontal interpolation (right) of two Gaussian densities.}
	\label{fig:interpolation}
\end{figure}
\end{example}

Let us introduce additional notation for the sequel and review then some basic properties of Wasserstein barycenters. 
Considering $\mathcal{W}_p(\mathcal X)$ with the complete metric $W_p$ as a base Polish metric space, we define $\mathcal{W}_p(\mathcal{W}_p(\mathcal X))$ in the natural way: 
 $\Gamma \in \mathcal{P}(\mathcal{W}_p(\mathcal{X}))$ is an element of $ \mathcal{W}_p(\mathcal{W}_p(\mathcal{X}))$ if it is concentrated on a set of measures with finite moments of order $p$, and moreover  for some (and then all) $\tilde m \in \mathcal{W}_{p}( \mathcal{X})$ it satisfies
$$\textstyle\int_{\mathcal P(\mathcal X)} W_p(m,\tilde m)^p\Gamma(dm)<\infty.$$

We endow $\mathcal{W}_p(\mathcal{W}_p(\mathcal{X}))$ with the corresponding  $p-$Wasserstein distance, which we also denote  $W_p$ for simplicity. 
Also, if $\Gamma$ is concentrated on measures with finite moments of order $p$ which have  densities with respect to $\lambda$, then we  write $\Gamma \in \mathcal{P}(\mathcal{W}_{p,ac}(\mathcal{X}))$ and use the notation $\Gamma \in \mathcal{W}_p(\mathcal{W}_{p,ac}(\mathcal{X}))$ if, as before, $\int_{\mathcal P(\mathcal X)} W_p(m,\tilde m)^p\Gamma(dm)<\infty$ for some $\tilde m$.

\smallskip

\begin{remark}\label{rem BWB implies WpWp and finmomBMA} 
\black{ If $\Gamma \in \mathcal{P}(\mathcal{P}(\mathcal{X}))$ has a $p$-Wasserstein  barycenter $\hat{m}_p$ over $\mathcal M$, then
$$\infty > \int_{\mathcal P(\mathcal X)} W_p(m,\hat{m}_p)^p\Gamma(dm) = \int_{{\cal W}_p (\mathcal X)} W_p(m,\hat{m}_p)^p\Gamma(dm),  $$
hence 
$\Gamma \in \mathcal{W}_{p}(\mathcal{W}_{p}(\mathcal{X}))$.} Moreover,  $\Gamma \in \mathcal{W}_{p}(\mathcal{W}_{p}(\mathcal{X}))$  is equivalent to the corresponding {\it model average} $\bar{m}(dx): = \mathbb{E}_{\Gamma }\left[m\right](dx)$ having a  finite $p$-moment, since for any $y\in \mathcal X$, 
	\begin{equation*}
	\textstyle\int_{\mathcal{W}_{p}(\mathcal X)} W_{p}(\delta_y ,m )^{p} \Gamma(dm) =\textstyle \int_{\mathcal{W}_{p}(\mathcal X)}\int_{\mathcal{X}}d(y,x)^{p}m(dx)\Gamma(dm) 
	= \textstyle\int_{\mathcal{X}}d(y,x)^{p}\int_{\mathcal{W}_{p}(\mathcal X)}m(dx)\Gamma(dm).
	\end{equation*}
\end{remark}	 

We next state an existence result first established in \cite[Theorem 2]{le2017existence} for the case ${\cal M}= \mathcal{W}_{p}$. See  Appendix \ref{subsec wasserstein_barycenters} for a simpler, more direct, proof. 

\begin{theorem}\label{baryexists}
\black{Suppose Assumption \ref{assum:geodesic} holds,  $\Gamma \in \mathcal{P}(\mathcal{P}(\mathcal{X}))$, and ${\cal M}\subset \mathcal{W}_{p}$ is a weakly closed set.	There exists a $p$-Wasserstein barycenter of $\Gamma$ over  ${\cal M}$ if and only if $\Gamma \in \mathcal{W}_{p}(\mathcal{W}_{p}(\mathcal{X}))$. }
\end{theorem}

Regarding uniqueness,  the following general result was proven in \cite[Proposition 6]{le2017existence}  for the case $\mathcal X=\mathbb{R}^q$ with $d$ the Euclidean distance and $p=2$ (observe that,  in that situation,  the previous result applies):

\begin{lemma}\label{uniqueB}
	Assume $\Gamma\in \mathcal{W}_2(\mathcal{W}_2(\mathbb{R}^q))$ and that there exists a set $A\subset \mathcal{W}_{2}( \mathbb{R}^{q})$ of measures with
	\begin{equation*}
		\mu \in A \text{, } B\in \mathcal{B}( \mathbb{R}^{q}) \text{, }\dim ( B) \leq q-1\
		\Longrightarrow \mu ( B) =0,
	\end{equation*}
	and $\Gamma(A) >0$. Then, $\Gamma$ admits a unique $2$-Wasserstein population barycenter over $\mathcal{W}_2(\mathbb{R}^q)$. 
\end{lemma}

\begin{remark}\label{rem exist unique bary} \black{ Observe  that the model space $\mathcal M=\mathcal W_{p,ac}(\mathcal X)$ is not weakly closed.}  Nevertheless, the existence and uniqueness of a population barycenter over that set can still be guaranteed when $p=2$, $\mathcal{X}=\mathbb{R}^{q}$, $d$ is the Euclidean distance, $\lambda$ is the Lebesgue measure, and
\begin{align}\textstyle
\Gamma\left ( \left\{m:\left \|\frac{dm}{d\lambda}\right\|_\infty<\infty \right\} \right )>0.\label{ass extra densities bdd}
\end{align} 
This was proven in \cite[Theorem 6.2]{kim2017wasserstein} for compact finite-dimensional manifolds with lower-bounded Ricci curvature (equipped with the volume measure), but one can read-off the (non-compact but flat) Euclidean case $\mathcal X=\mathbb R^q$ from the proof therein,  in order to establish  the absolute continuity of a  barycenter over $\mathcal{W}_2(\mathbb{R}^q)$,  in the setting of Lemma \ref{uniqueB}. If $|\text{supp}(\Gamma)|<\infty$ then eq.~\eqref{ass extra densities bdd} can be relaxed to the condition $\Gamma\left ( \left\{m:   m \ll  \lambda \right\} \right)>0 $, as shown in \cite{agueh2011barycenters} or \cite[Theorem 5.1]{kim2017wasserstein}.
\end{remark}

The following statement, corresponding to   \cite[Lemma 3.1]{bakchoff2022},   provides a useful description of barycenters which  generalizes a result proven in  \cite{alvarez2016fixed} when  $|\text{supp}(\Gamma)|<\infty$. 

\begin{lemma}
	\label{measurability fixed point}
Assume $p=2$, $\mathcal{X}=\mathbb{R}^{q}$, $d=$ Euclidean distance, $\lambda =$Lebesgue measure.  Let
$\Gamma \in \mathcal{W}_2(\mathcal{W}_2(\mathcal{X}))$ and  $\tilde{\Gamma} \in \mathcal{W}_2(\mathcal{W}_{2,ac}(\mathcal{X}))$. 
There exists a jointly measurable function  $\mathcal W_{2,ac}(\mathbb R^q)\times\mathcal W_2(\mathbb R^q)\times \mathbb R^q \ni(\mu,m,x)\mapsto T^m_\mu(x)$  which is $\mu(dx)\Gamma(dm)\tilde\Gamma(d\mu)$-a.s.\ equal to the  unique optimal transport map from $\mu$ to $m $ at $x$. Furthermore, letting $\hat \mu$ be a barycenter of $\Gamma $, we have
	$$ x \,=  \, \int T^m_{\hat\mu}(x) \Gamma (dm),\,\, \hat \mu(dx)-a.s. $$

\end{lemma}

\section{Bayesian Wasserstein barycenter and statistical properties}\label{sec consistency}

Building on  the Wasserstein distance as a loss function on models, we arrive to the following central object of the article:

\begin{definition}\label{defi bary}
Let us consider a prior $\Pi \in \mathcal{P}(\mathcal{P}(\mathcal{X}))$ with model space ${\cal M}\subseteq \mathcal{W}_{p,ac}( \mathcal{X})$ and data $D=(x_1,\dots,x_n)$ which determines $\Pi_n$ as in eq.~\eqref{eq def Pi n}. We define the $p$-Wasserstein Bayes risk of $\bar{m} \in \mathcal{W}_{p}( \mathcal{X})$ and a Bayes  Wasserstein barycenter  (BWB) estimator $\hat{m}_p^{n}$ over  $\mathcal M $ respectively as follows:
	\begin{eqnarray}
	V_p^n(\bar{m}|D) &:=&  \int_{\mathcal{P}(\mathcal X)} W_{p}(m,\bar{m})^{p}\Pi_n(dm) \,  \mbox{, and }\\
	\hat{m}_p^n &\in& \argmin_{\bar{m} \in \mathcal{M}} V_p^n(\bar{m}|D),\label{eq bary for posterior}
	\end{eqnarray} \black{if the corresponding minimum is finite. } 
\end{definition} 

\begin{example}\label{ex:gaussBWBvsBMA}
\black{In the setting of Example \ref{ExParaModel1}, the  $2$-Wasserstein loss function on Gaussian models (see Example \ref{ex:WdGauss}) induces   the  usual quadratic loss on the mean parameters: 
$W_2^2(m,\bar{m})= \ell(\theta,  \bar{\theta}):= \Vert \theta - \bar{\theta}\Vert^2$. Thus, in the notation of eq.~\eqref{eq:param_bayes_risk}, we have
$$ V_2^n(\bar{m}|D) =  
	\textstyle R_{\ell} (\bar{\theta}|D).  $$ 
	This implies that, for any tuple of data point $D=(x_1,...,x_n)$ the BWB corresponds to  the Gaussian distribution  ${\cal N}(\bar{\theta}, \Sigma)\in {\cal M}$ with mean $\bar{\theta}=\hat{\theta}_{\ell_2}= (\Sigma_0^{-1}+ n \Sigma^{-1})^{-1}  (\Sigma_0^{-1}\mu_0 +  n \Sigma^{-1} \bar{x}_n),$ that is, the posterior mean estimator of the mean parameter. Moreover, in this particular case the barycentric cost or optimal $2$-Wasserstein Bayes risk equals the trace of the covariance of a random vector with law given in  eq.~\eqref{eq:post_param_ex1.1}, i.e., $Tr((\Sigma_0^{-1}+ n \Sigma^{-1})^{-1})$. 
	Notice that the  covariance   $\Sigma$ of the BWB  is strictly smaller than that  of the corresponding BMA estimator in eq.~\eqref{eq:MAGaussian}  (in the usual order on symmetric positive semidefinite matrices).   This is, in fact, a general property of the BWB as claimed in Proposition \ref{prop:BWBlessBMA} below. } 
\end{example}
\begin{example}\label{ex:PoissonBWB}
\black{Similarly, in the parametric setting of  Example \ref{ExParaModel2}, the problem of finding BWB estimators 
can be written  in terms of the parametric loss induced by the corresponding $p$-Wasserstein distance  on Poisson distributions, computed in Example
\ref{ex:WdPoisson}. We thus deduce for $p=1$ that the BWB estimator $	\hat{m}_1^n$ corresponds, in that setting, to the  law Pois$(\lambda_1) $ with  $\lambda_1$ the median of the posterior distribution Gamma$(n \bar{x}_n+\alpha,n +\beta)$; for which an explicit expression is not available. }
\end{example}

\begin{example} \black{Let us assume  ${\cal X}=\RR$ and that $\Pi(dm)$ is supported on continuous models   $m$  over $\mathbb{R}$. Let $F_m$ and $Q_m:=F_m^{-1}$ denote, respectively, the cumulative distribution function and the right-continuous quantile function of such  $m$. The coupling $(x,T_{m_0}^m(x))$, with $x$ distributed like $m_0$ and $T_{m_0}^m$ the  increasing map 
$T_{m_0}^m(x) : = Q_{m}(F_{m_0}(x))$,  is known to be optimal for the $p$-Wasserstein distance, for any $p\geq  1$ (see \cite[Remark 2.19(iv)]{villani2003topics}). The BWB  $\hat{m}_n=\hat{m}_n^p$ of the posterior $\Pi_n$ is also independent of $p$ and is characterized via its quantile, as follows:} 
\begin{equation*}
\textstyle Q_{\hat{m}_n}(\cdot) = \int_{\cal M} Q_m(\cdot) \Pi_n(dm).
\end{equation*}
Interestingly enough, the model average  $\bar{m}_n:=\int m\Pi_n(dm)$ of $\Pi_n$ is in turn  characterized by its averaged cumulative distribution function:  $F_{\bar{m}_n}(\cdot) = \int_{\cal M} F_m(\cdot) \Pi_n(dm)$. See Section 5 in the companion paper \cite{bakchoff2022} for details and further discussion on one-dimensional Wasserstein barycenters, in particular, on geometric properties they inherit from the elements $m$ of the support of the prior/posterior. 
\end{example}

\medskip

\begin{remark}\label{rem Pin a.s.} \black{In the sequel,  unless otherwise stated, an a.s.~statement about $\Pi_n$ is meant to hold almost surely  with respect to the marginal law $m^{\otimes n}(dx_1,\dots,dx_n )\int_{\cal M}\Pi(dm)$ of a data sample of size $n$ (sometimes called prior predictive distribution). In particular, for $\Pi(dm)-$ almost every $m$, such statement holds for $m^{\otimes n}(dx_1,\dots,dx_n )-$ almost every sample $(x_1,\dots,x_n)$.}
\end{remark}

\begin{remark}\label{rem W_pW_p wrt marginal law} We observe that $\Pi\in \mathcal W_{p}(\mathcal W_p(\mathcal X))$ implies for each $n\geq 1$ that 
$$\Pi_n\in \mathcal W_{p}(\mathcal W_p(\mathcal X))\,\,\, \text{a.s.}$$
Indeed, for fixed $\tilde m\in \mathcal W_p(\mathcal X)$, 
(we thank an anonymous referee for pointing this out) the following quantity is finite:
\begin{equation}
    \begin{split}
\int_{\cal M} W_p(m, \tilde m )^p \Pi(dm)   = & \int_{{\cal X}^n} \int_{\cal M} W_p(m, \tilde m )^pm(x_1)\cdots m(x_n)  \Pi(dm) \lambda(dx_1)\cdots \lambda(dx_n) \\ 
=&      \int_{{\cal X}^n} \int_{\cal M} W_p(m, \tilde m )^p \Pi_n (dm)    \int_{\mathcal{M}}\bar{m}\left( x_{1}\right) \cdots\bar{m}\left( x_{n}\right)  \Pi \left( d\bar{m}\right) \lambda(dx_1)\cdots \lambda(dx_n) \\ 
=&      \int_{{\cal X}^n} \left[ \int_{\cal M} W_p(m, \tilde m )^p \Pi_n (dm)  \right]  \,   \bar{m}^{\otimes n}(dx_1,\dots,dx_n )\int_{\cal M}\Pi(d\bar{m}).
    \end{split}
\end{equation}
However, $\Pi\in \mathcal W_{p}(\mathcal W_p(\mathcal X))$ is in general not enough for $\Pi_n\in \mathcal W_{p}(\mathcal W_p(\mathcal X))$ to hold a.s.\ for i.i.d.\ data points $(x_1,\dots, x_n)$ sampled from a \emph{fixed} law $m_0$, which would be the natural setting to formulate the question of Bayesian consistency (see next subsection). Corollary \ref{lemma m0KL PinWpWp} below  ensures this fact for suitable given laws $m_0$, in the framework of Bayesian consistency.  In Appendix \ref{subsec existence_barycenters} we further provide a sufficient condition on the prior $\Pi$ (termed {\it integrability after updates})   ensuring that 
$\Pi_n\in \mathcal W_{p}(\mathcal W_p(\mathcal X))$ for \emph{every} possible tuple of data points $(x_1,\dots,x_n)$. 
%that even if $\Pi\in \mathcal W_{p}(\mathcal W_p(\mathcal X))$, it may still happen that $\Pi_n\notin \mathcal W_{p}(\mathcal W_p(\mathcal X))$. and therefore the existence of a BWB  estimator, for all $n$.
\end{remark} 

The following statement gathers the   discussion in Section \ref{sec:wassersteinbary} for the case $\Gamma=\Pi_n$,  as well as the main point of Remark  \ref{rem W_pW_p wrt marginal law}:

\smallskip

\begin{theorem}\label{thm summa}
\black{Suppose Assumption \ref{assum:geodesic} holds  and the  model space $\mathcal M$ is weakly closed. Let  $\Pi \in \mathcal{P}(\mathcal{P}(\mathcal{X}))$ be a prior with model space ${\cal M}\subseteq \mathcal{W}_{p,ac}( \mathcal{X})$ and  $\Pi_n$ be the corresponding posterior given the  data $D=(x_1,\dots,x_n)$. The following are equivalent:
\begin{itemize}
    \item[a)] A $p$-Wasserstein barycenter estimator $\hat m_p^n$ for $\Pi_n$ over $\mathcal M$ exists a.s. 
    \item[b)]  $\Pi_n\in \mathcal W_{p}(\mathcal W_p(\mathcal X))$ a.s. 
    \item[c)] The model average $\bar{m}^n(dx) = \mathbb{E}_{\Pi_n}\left[m\right](dx)$ has  a.s.\ a finite $p$-moment. 
\end{itemize}
Moreover, if   
$\Pi\in \mathcal W_{p}(\mathcal W_p(\mathcal X))$,    a $p$-Wasserstein barycenter estimator $\hat m_p^n$ over $\mathcal M$ exists a.s.\ for every $n\geq 1$.}
\end{theorem}

We thus make in all the sequel the following  assumption: 

\begin{assumption}\label{ass nice}
 $\Pi\in \mathcal{W}_p(\mathcal{W}_{p,ac}(\mathcal{X}))$ and there exists a weakly closed model space  $\mathcal M\subseteq \mathcal W_{p,ac}(\mathcal X)$ for $\Pi$.
 %\black{BORRARIAMOS ESTO and $\Pi_n\in \mathcal W_{p}(\mathcal W_p(\mathcal X))$ a.s.\ all $n$.}
\end{assumption} 
Since $\Pi_n \ll \Pi$ a.s., Assumption \ref{ass nice} together with Remark \ref{rem W_pW_p wrt marginal law}  imply $\Pi_n\in \mathcal{W}_p(\mathcal{W}_{p,ac}(\mathcal{X}))$ a.s. for all $n$.

\smallskip 

We will next study some basic statistical properties of the BWB estimator. 

\subsection{Variance reduction with respect to  BMA}\label{subsec:varreducBMA}

In this subsection, we assume that ${\cal X}= \mathbb{R}^q$,  $\lambda = $ Lebesgue measure, $d=$ Euclidean distance and $p=2$.  Let $\hat{m}:=\hat{m}^n_2$ be the unique population barycenter of $\Pi_n$ in that case,  and denote by   $(m,x)\mapsto T^m(x)$ a measurable function equal   $\lambda(dx)\Pi(dm)$  a.e.\ to the  unique optimal transport map from $\hat{m}$ to $m \in \mathcal{W}_2(\mathcal{X})$. As a consequence of Lemma \ref{measurability fixed point} we have  the fixed-point property $\hat{m} =( \int T^m \Pi(dm))(\hat{m})$. Thus, for all convex functions $\phi$, non negative or with at most quadratic growth,  we have
\begin{eqnarray}
	\mathbb{E}_{\hat{m}}[\phi (x)]
	&=& \textstyle  \int_{\mathcal{X}}\phi (x)\hat{m}(dx)
	=\int_{\mathcal{X}}\phi\left (\int_{\mathcal{M}} T^m(x) \Pi_n(dm)\right)\hat{m}(dx)\notag \\
	&\leq& \textstyle  \int_{\mathcal{X}}\int_{\mathcal{M}} \phi ( T^m(x)) \Pi_n(dm)\hat{m}(dx)
	=\int_{\mathcal{M}}\int_{\mathcal{X}}\phi ( T^m(x)) \hat{m}(dx)\Pi_n(dm)\notag \\
	&=& \textstyle   \int_{\mathcal{M}}\int_{\mathcal{X}}\phi(x) m(dx)\Pi_n(dm)
	=	\int_{\mathcal{X}} \phi(x) \int_{\mathcal{M}}m(dx)\Pi_n(dm) = \mathbb{E}_{\hat{m}_{BMA}} [\phi(x)],\notag
\end{eqnarray}
where $\hat m_{BMA} =\mathbb{E}_{\Pi_n}[m]$ is the Bayesian model average in eq.~\eqref{eq:model_average}. We have used  Jensen's inequality and Fubini's theorem. This means that  the BWB estimator  is less spread,  or \textit{smaller}, in the sense of  convex-order of probability measures, than the BMA.  As a consequence,  we have:

\black{ \begin{proposition}\label{prop:BWBlessBMA}
	Consider $p= 2$, and let $\hat m_{BMA}$ and $\hat{m}_2^n$ respectively denote the BMA and the $BWB$ estimators associated with $\Pi_n$. Then, we have $\mathbb{E}_{\hat{m}_{BMA}}[x] = \mathbb{E}_{\hat{m}^n_2}[x]$ and $\mathbb{E}_{\hat{m}_{BMA}}[\Vert x \Vert^2] \geq \mathbb{E}_{\hat{m}^n_2}[\Vert x \Vert^2]$. In other words,  the BWB  estimator  has less variance than the BMA. Furthermore, with $\bar x$  denoting the mean w.r.t.\ the $BMA$ or $BWB$,  the corresponding covariances satisfy:   $\mathbb{E}_{\hat{m}_{BMA}}[ (x-\bar{x})( x-\bar{x})^t ] \geq \mathbb{E}_{\hat{m}^n_2}[(x-\bar{x})( x-\bar{x})^t]$ in the usual order for symmetric positive semidefinite matrices. 	The inequalities are strict unless $\Pi_n$ is a Dirac mass.  
\end{proposition}}

\begin{proof}\black{   Given the previous discussion, the equality  of means is obtained by taking $\phi(x)$  and $-\phi(x)$ equal to each coordinate of $x$ and the inequality of variances by taking $\phi(x)=\Vert x \Vert^2$. The inequality for the covariances follows by taking $\phi(x)=(y^t (x-\bar{x}))^2 $ for arbitrary $y\in \RR^d$. }

\black{For the last claim, we just need to make explicit  the corresponding equality case of Jensen's inequality as used in the preceding discussion. Since $\|x\|^2$ is strictly convex, the equality of second moments implies that $\hat{m}_2^n(dx)$  a.e.\ $x$, that   $T^m(x) $ is $\Pi_m(dm)$ a.s. constant.  This entails that the map $T:=T^m$ does not depend on $m$ and that for $\Pi_m(dm)$ a.e.\ $m$, it holds that $m=T(\hat{m}_2^n)$.  The equality case for the covariances is reduced to the previous one considering their traces using also the equality of means. }
\end{proof}

\subsection{Convergence to the true model and Bayesian consistency}\label{subsec:consitency}

A natural question in Bayesian statistics is whether a given predictive posterior  estimator is \emph{consistent}  (see \cite{schwartz1965bayes,diaconis1986consistency,ghosal2017fundamentals}). In short,  this means the convergence  of  the predictive posterior law, in some specified sense,  towards the \emph{true} model $m_0$, as we observe more and more i.i.d.\ data sampled from $m_0$. 
We are specifically interested in the  question of whether the BWB estimator converges to the model $m_0$ and, more precisely, on conditions which guarantee that 
\begin{equation}\label{eq consistentW}
W_p(\hat m_p^n,m_0)\to 0,\,\, m_0^{(\infty)} a.s.
\end{equation}
\black{as $n\to \infty$, where $\hat{m}_p^n$ is for each $n$ a BWB  over a model space ${\cal M}$.}  Here and in the sequel,  $m_0^{(\infty)}$ denotes  the product law of the infinite sample $\{x_n\}_n$ of i.i.d.\ data distributed according to $m_0$. We will see that this question is linked to the notion of consistency \black{(cf. \cite[Definition 6.1]{ghosal2017fundamentals})} of the prior, introduced next:

\begin{definition}
	The prior $\Pi$ is said to be  consistent at $m_0$ in the  weak topology  (resp.\ $p$-Wasserstein topology)  if for each open neighbourhood $U$ of $m_0$ in the weak   topology of ${\cal P}(\mathcal X)$ (resp.\ $p$-Wasserstein topology of ${\cal W}_p(\mathcal X)$), one has
	$\Pi_n(U^c)\to 0\, , \,\,\, m_0^{(\infty)}-a.s.$
\end{definition}
%
%\blue{Equivalently (\cite[Proposition 6.2]{ghosal2017fundamentals}) consistency at $m_0$ in the $p$-Wasserstein (resp. weak)  topology means the $m_0^{(\infty)}$-almost sure convergence in probability of the posterior to the dirac centred at $m_0$.}

\begin{remark} \black{We notice that in the literature of Bayesian consistency, see e.g. \cite{schwartz1965bayes,ghosal2017fundamentals},   $\Pi$ satisfying the above property (in either topology)  would be called  ``strongly consistent'' at $m_0$   in allusion  to the $m_0^{(\infty)}$-a.s.\  convergences, whereas the term  ``weakly consistent''  would be used when those convergences hold in probability. Since we will be only dealing with the almost sure notion,  in order to avoid possible confusions with topological concepts,   the adverb ``strongly''   is omitted throughout when referring to consistency of the prior, while the adverb  ``weakly'' only refers to the weak topology on probability measures. }  
\end{remark}

\smallskip 

The celebrated Schwartz theorem  \cite{schwartz1965bayes} provides sufficient conditions for consistency w.r.t.\ a given topology, see also \cite[Chapter 6]{ghosal2017fundamentals} for a  modern treatment. A key ingredient  is the notion of Kullback-Leibler support:

\begin{definition}
	A measure $m_{0}$ is an element the Kullback-Leibler support of $\Pi$, denoted $$m_0 \in \text{KL}\left( \Pi \right),$$ if $\Pi \left(m:D_{KL}\left(m_{0}||m\right) < \varepsilon \right)>0$ for every $\varepsilon >0$, with $D_{KL}\left(m_{0}||m\right)$  the reverse Kullback-Leibler entropy defined as $\int \log \frac{m_{0}}{m}(x) m_0(dx)$ if $m_0\ll m$ and as $+\infty$ otherwise. 
\end{definition}
\begin{remark} \black{The  statistical model is interpreted as being  ``correct'' or \emph{well specified},   if the  data distribution $m_0$ is an element of  $\supp(\Pi)$, the support of $\Pi$ w.r.t.\  the weak topology,   see  \cite{berk1966, grendarjudge, kleijn2004bayesian, kleijn2012bernstein,ghosal2017fundamentals}. The condition $m_0 \in \text{KL}\left( \Pi \right)$ is stronger. Indeed, by the Csiszar-Pinsker inequality and the fact that the dual bounded-Lipschitz distance (metrizing the weak topology in ${\cal P}(\mathcal X)$) is majorized by the total variation distance, one can check that  $\text{KL}\left( \Pi \right)\subseteq  \supp(\Pi)$. In particular, one has  $\text{KL}\left( \Pi \right)\subseteq {\cal M}$ for any weakly closed model space ${\cal M}$ for $\Pi$. 
(The reader may consult the  mentioned works for the \emph{misspecified}  framework too.)}
\end{remark}

\begin{remark}\label{abs cont m_0 marg}
If  $m_0 \in \text{KL}\left( \Pi \right)$, then  $m_0(dx) \ll m(dx)\int_{\cal M}\Pi(dm)$,  the marginal law of one data point $x$ in the Bayesian model defined by $\Pi $. Indeed, for any measurable set $A\subseteq {\cal X}$ such that $\int_{\cal M} \int_{A} m(dx)\Pi(dm)=0 $ we have
$ \Pi \left(m: m_0\ll m,  m(A)=0   \right)= \Pi \left(m: m_0\ll m  \right)   \geq  \Pi \left(m:D_{KL}\left(m_{0}||m\right)< +\infty \right) >0.  $ This will be useful later. 
\end{remark}

We recall (see Theorem 6.17 and Example 6.20 in \cite{ghosal2017fundamentals}) the following result which concerns the weak topology: 
\begin{theorem}\label{schartzweak}
Assume (only) that $({\cal X},d)$ is Polish  endowed with a $\sigma$-finite Borel measure $\lambda$, that $\Pi\in {\cal P}_{ac} ({\cal X})$  and that  $m_0 \in \text{KL}\left( \Pi \right)$. Then, $\Pi$
is consistent at $m_0$ in the weak topology. 
\end{theorem}

 We  now state a general result relating  consistency of  $\Pi$  at $m_0$ in the $p$-Wasserstein topology and the convergence \eqref{eq consistentW}  (or consistency of the predictive posterior $\hat{m}_p^n$),   with other asymptotic properties of the posterior laws in the Wasserstein setting. Recall  that the notation  $W_p$  throughout stands for  the Wasserstein distance both in  $\mathcal{W}_p(\mathcal X)$ and $\mathcal{W}_p(\mathcal{W}_p(\mathcal X))$. 
 %Recall also that $\mu_k \to \mu$ in $W_p(\cal X)$ iff for all continuous functions $\psi$  on $(\cal X)$ with $|\psi(x)| \le K(1+d^p(x,x_0))$  for some $K \in \mathbb{R}_+$ it holds that $\int_{\mathcal{X}}\psi(x)\mu_n(dx) \to \int_{\mathcal{X}}\psi(x)\mu(dx)$; see \cite{villani2008optimal}.  Similarly, $\Gamma_k \to \Gamma$ in $W_p(\cal X)$ iff for all continuous functions $\Psi$  on ${\cal W}(\cal X)$ with $|\Psi(m)| \le K(1+W^p(m,\bar m))$  for some $K \in \mathbb{R}_+$ and some $\bar m\in {\cal W}({\cal X}) $,  it holds that $\int_{ {\cal W}({\cal X})} \Psi(x)\Gamma_n(dm) \to \int_{{\cal W}({\cal X})}\Psi(m)\Gamma(dm)$. 

 \begin{theorem}\label{thm equiv consistency}
Suppose  Assumptions \ref{assum:geodesic} and \ref{ass nice} hold, \black{that $m_0 \in {\cal M}$, and that }
\begin{equation}\label{cond PinWpWp}
\Pi_n \in \mathcal{W}_p(\mathcal{W}_p(\mathcal X)), \quad m_{0}^{\otimes n}-\mbox{ a.s. for all }n\geq 1.
\end{equation}
The following are equivalent:
\begin{itemize}
\item[a)]  $W_p(\Pi_n,\delta_{m_0}) \rightarrow 0$, $m_{0}^{(\infty)}$-a.s.\ as $n\to \infty$. 
\item[b)]  $m_{0}^{(\infty)}$-a.s.\ as $n\to \infty$, we have  $W_p(\hat m_p^n,m_0)\rightarrow 0$ and the  barycentric cost (or optimal $p-$Wasserstein Bayes risk or)  $\int_{\mathcal M}W_p(m,\hat m_p^n)^p\,\Pi_n(dm)$ goes to $0$.
\item[c)] $\Pi$ is  consistent at $m_0$ in the $p$-Wasserstein topology and the $p-$moment of the  BMA estimator \eqref{eq:model_average} converge $m_{0}^{(\infty)}$-a.s.\ to that of $m_0$ as $n\to \infty$, i.e.\ for some (and then all) $x_0\in \cal X$ we have 
\begin{equation}\label{eq:convpmomBMA}
		\textstyle \int_{{\cal X}} d( x,x_0)^p \hat{m}_{BMA}^n(dx) =\int_{\mathcal{M}} \int_{{\cal X}} d( x,x_0)^p m(dx)\Pi_n(dm) \to \int_{{\cal X}} d( x,x_0)^p m_0(dx) \, , m_{0}^{(\infty)}\mbox{ a.s.} 
\end{equation}
\item[d)] $\Pi$ is  consistent at $m_0$ in the weak topology and for some (and then all) $x_0\in \cal X$ we have  \begin{equation}\label{eq:convL1mom}\int_{\cal M} \left\vert \int_{{\cal X}} d( x,x_0)^p m(dx) -  \int_{{\cal X}} d( x,x_0)^p m_0(dx) \right\vert \Pi_n(dm) \to 0, \, m_{0}^{(\infty)}\mbox{ -a.s.}  
\end{equation} 

\end{itemize}

\end{theorem}

\begin{proof}  
By minimality of a barycenter over ${\cal M}$,
	$$  \int_{\mathcal M}W_p(m,\hat m_p^n)^p\,\Pi_n(dm)\leq \int_{\mathcal M}W_p(m,m_0)^p\,\Pi_n(dm)=   \textstyle W_p(\Pi_n,\delta_{m_0})^p. $$
Thus, for some $c>0$ depending only on $p$, 	\begin{align}\label{eq:boundBWBW}
		\textstyle W_p(m_0,\hat m_p^n)^p\ &\textstyle \leq c \int_{\mathcal M}W_p(m,\hat m_p^n)^p\,\Pi_n(dm) + c\int_{\mathcal M}W_p(m, m_0)^p\,\Pi_n(dm) \\
		&\leq \textstyle 2c\,W_p(\Pi_n,\delta_{m_0})^p
	\end{align}
proving that a) $\Rightarrow$ b).  The converse  b) $\Rightarrow$ a) follows from
	$$  \int_{\mathcal M}W_p(m,m_0)^p\,\Pi_n(dm)\leq c \int_{\mathcal M}W_p(m,\hat m_p^n)^p\,\Pi_n(dm)+ c\, W_p(\hat m_p^n,m_0)^p. $$
Let us now show that a) $\Rightarrow$ c). The convergence  $W_p(\Pi_n,\delta_{m_0}) \rightarrow 0$ implies  (by the Portmanteau theorem)  that $\limsup_{n\to \infty} \Pi_n(F)\leq \delta_{m_0}(F)$ for all closed sets $F$ of $ {\cal W}_p({\cal X})$. Taking $F=U^c$ with $U$ a neighborhood of $m_0$ yields the consistency of $\Pi$ at $m_0$ in the $p$-Wasserstein  topology. Moreover it implies that for each 
$\tilde{m}\in {\cal W}_p({\cal X})$, 
$$\int_{\cal M} W_p^p(m, \tilde{m}) \Pi_n(dm) \to \int_{\cal M} W_p^p(m, \tilde{m}) \delta_{m_0}(dm) =  W_p^p(m_0, \tilde{m}) \, , m_{0}^{(\infty)}\mbox{ -a.s.}   $$
as $n\to \infty$. Choosing $\tilde{m}=\delta_{x_0}$, we have  $ W_p^p(m, \tilde{m}) = \int_{\cal X} d( x,x_0)^p m(dx) $ for any $m$,  from where  \eqref{eq:convpmomBMA} follows. 

We next prove that c) $\Rightarrow$ a). The space ${\cal W}_p(\cal X)$ being Polish, there is countable basis $\cal U$ of  open neighborhoods  of $m_0$  such that  $m_{0}^{(\infty)}$-a.s., $\Pi_n(U^c)\to 0 $ for all $U \in \cal U$.  Thus, if $G \subseteq {\cal W}_p(\cal X)$ is any  open set  such that  $m_0\in G$, for some $U \in \cal U$ we have $U\subseteq G$ and therefore 
$\liminf_{n\to \infty} \Pi_n(G)\geq  \liminf_{n\to \infty}  \Pi_n(U) =  1- \lim_{n\to \infty}  \Pi_n(U^c) =1=\delta_{m_0}(G)$, $m_{0}^{(\infty)}$-a.s.  
This implies, by the Portmanteau theorem, that  the sequence $(\Pi_n)_{n}$ weakly  converges  to $\delta_{m_0}$, $m_{0}^{(\infty)}$-a.s., as probability measures on the metric  space ${\cal W}_p(\cal X)$.   As in the previous (converse) implication, we obtain from \eqref{eq:convpmomBMA} the convergence of some moments of order $p$ of $\Pi_n$, to the corresponding moment of $\delta_{m_0}$,  $m_{0}^{(\infty)}$-a.s. This plus the weak  convergence just established imply that  $W_p(\Pi_n,\delta_{m_0}) \rightarrow 0$ $m_{0}^{(\infty)}$-a.s. 

\medskip 

We have thus established that a), b) and c) are equivalent. Notice now that  the function 
\begin{equation}\label{eq:Phim}
m \mapsto \Phi(m):=  \left|\int_{\cal X}\left(  d( x,x_0)^p - \int_{\cal X} d( y,x_0)^p m_0(dy) \right) m(dx) \right|,
\end{equation}
is continuous on ${\cal W}_p({\cal X})$, since 
 $x \mapsto  d( x,x_0)^p - \int_{\cal X} d( y,x_0)^pm_0(dy) $ is continuous with polynomial growth of order $p$ on ${\cal X}$.  Moreover, $|\Phi(m)| \leq \int_{\cal X}d( y,x_0)^p m_0(dy) + W_p^p(m,\delta_{x_0}) $, that is, $\Phi$ has polynomial growth of order at most $p$ on ${\cal W}_p(\cal X)$. Therefore, if a) or equivalently c)  holds,  we have $m_{0}^{(\infty)}$-a.s.  that
$$  \int_{\cal M} \Phi(m) \Pi_n(dm) \to \int_{\cal M} \Phi(m) \delta_{m_0}(dm) = \Phi(m_0)=0$$
which is tantamount to \eqref{eq:convL1mom}. Moreover  if c) holds, consistency of $\Pi$ at $m_0$ in the weak topology is obvious. This shows that c) $\Rightarrow$ d).

Last, if d) holds, we deduce with Markov's inequality that, 
for each rational $\varepsilon>0$, 
\begin{align}\label{eq:convprobanmom}
		\Pi_n\left\{ m\in {\cal M}: \left\vert \int_{{\cal X}}  d( x,x_0)^p m(dx) -  \int_{{\cal X}}  d( x,x_0)^p m_0(dx) \right\vert \geq \varepsilon \right\} \to 0 \, , m_{0}^{(\infty)}\mbox{ a.s.}
	\end{align}
as $n\to \infty$. This, together with the consistency of $\Pi$ at $m_0$ w.r.t.\ weak topology, is equivalent to having that consistency w.r.t.\ the $p-$Wasserstein topology. Moreover, we have \begin{multline*}
\left| \int_{{\cal X}}  d( x,x_0)^p\hat{m}_{BMA}^n(dx) - \int_{{\cal X}}  d( x,x_0)^p m_0(dx)\right| v\\ \leq   \int_{{\cal M}} \left| \int_{{\cal X}}  d( x,x_0)^p m(dx) -  \int_{{\cal X}}  d( x,x_0)^p m_0(dx) \right| \Pi_n(dm)   \end{multline*}
and so the convergence \eqref{eq:convL1mom} implies the convergence \eqref{eq:convpmomBMA}. This and the previous show that d) implies c), concluding the proof. 

\end{proof} 

\begin{remark}
The convergence  \eqref{eq:convprobanmom} for each $\varepsilon>0$ is exactly what one must add to consistency  at $m_0$ of $\Pi$ in the weak topology to  obtain such consistency in the $p$-Wasserstein topology.  However, the latter is  only equivalent to  the $m_{0}^{(\infty)}$-a.s.\ weak convergence  of  $\Pi_n$ to $\delta_{m_0}$ as measures on the metric space ${\cal W}_p({\cal X})$ as $n\to \infty$,  which is not enough to grant the $m_{0}^{(\infty)}$-a.s.\  convergence in ${\cal W}_p({\cal W}_p({\cal X}))$  of  $\Pi_n$ to $\delta_{m_0}$. Similarly, consistency  at $m_0$ of $\Pi$ in the $p$-Wasserstein topology cannot in general be obtained by adding only the  convergence of moments \ref{eq:convpmomBMA} to the consistency at $m_0$ of $\Pi$ in the weak topology.  Of course, if the space ${\cal X}$ is bounded,  weak and  $p$-Wasserstein topologies on it coincide,  and consistency of $\Pi$ at $m_0$ in the weak topology {\it implies} in that case the convergence \eqref{eq:convL1mom} (since the function $\Phi$ in the previous proof is in that case  continuous and bounded). Hence, in that case,  all the equivalent properties in Theorem \ref{thm equiv consistency} are satisfied. The same is true if $\Pi(dm)-$a.e. $m$ is supported on a fixed bounded (weakly) closed set ${\cal Y}\subseteq {\cal X}$ (just replace  ${\cal X}$ by  ${\cal Y}$).   

\end{remark}

An immediate consequence of the proof of Theorem \ref{thm equiv consistency} (cf. the estimate \eqref{eq:boundBWBW}) is the following bound, which can be used to obtain quantitative estimates for the average rate of convergence  of the BWB,  if the posterior and the Wasserstein distance between models are explicit enough: 

\begin{corollary} 
Under the assumptions of Theorem \ref{thm equiv consistency}, for some constant $c_p>0$ we have
$$\EE\left( W_p(m_0,\hat m_p^n)^p  \right)\leq c_p \EE\left(W_p(\Pi_n,\delta_{m_0})^p \right) . $$ 
\end{corollary}

The following result gathers sufficient conditions for  consistency at $m_0$ of $\Pi$ in the $p$-Wasserstein topology and convergence of the BWB. 

\begin{corollary}
\label{lemma m0KL PinWpWp} \black{Suppose Assumptions \ref{assum:geodesic} and \ref{ass nice} hold, and moreover that $m_0\in\text{KL}\left( \Pi \right) $. Then $\Pi$ is  consistent at $m_0$ in the weak topology and  condition \eqref{cond PinWpWp} holds. If moreover either condition (i) or (ii) below hold, then $\Pi$ is  consistent at $m_0$ in the $p$-Wasserstein topology and  $W_p(\hat m_p^n,m_0)\rightarrow 0$,  $m_{0}^{(\infty)}$-a.s.\  as $n\to \infty$:
\begin{itemize}
    \item[(i)]  the convergence  \eqref{eq:convL1mom} holds;
    \item[(ii)] for some $q>p$,  $m_{0}^{(\infty)}$-a.s.\ the $q$ moments of the BMA estimator are bounded uniformly  in $n$.
\end{itemize}
 }
\end{corollary}
\begin{proof}
    In view of Theorem \ref{schartzweak}, to prove the first claim we  just need to prove that  \eqref{cond PinWpWp} holds. To that end, let us first check that, whenever  $m_0\in\text{KL}\left( \Pi \right) $, we have
    $m_0^{\otimes n}(dx_1,\dots,dx_n ) \ll m^{\otimes n}(dx_1,\dots,dx_n )\int_{\cal M}\Pi(dm).$
    If $A\subseteq {\cal X}^n$ measurable is such that $\int_{A}m^{\otimes n}(dx_1,\dots,dx_n )\int_{\cal M}\Pi(dm)=0$, then for $\Pi(dm)-$a.e.\ $m$ and every $x_i\in {\cal X}$, $i=2,..,n$, one has $m(\{x\in {\cal X} : (x,x_2,...,x_n)\in A\})=0 $, from which we get  $m_0(\{x\in {\cal X} : (x,x_2,...,x_n)\in A\})=0$ by  Remark \ref{abs cont m_0 marg},  hence $m_0^{\otimes n}(A)=0$. 
    We conclude noting that the set  
    $A:= \{ ( x_1,\dots, x_n)\in {\cal X}^n :  \Pi_n\notin \mathcal W_{p}(\mathcal W_p(\mathcal X)) \} $ has null $m^{\otimes n}(dx_1,\dots,dx_n )\int_{\cal M}\Pi(dm)$- measure, by  Remark \ref{rem W_pW_p wrt marginal law}.  

Given the previous and  part d) of Theorem \ref{thm equiv consistency}, the convergence \eqref{eq:convL1mom} immediately implies that $\Pi$ is  consistent at $m_0$ in the $p$-Wasserstein topology and that $W_p(\hat m_p^n,m_0)\rightarrow 0$  $m_{0}^{(\infty)}$-a.s.\ as $n\to \infty$. To conclude the proof it is enough to show that the $m_{0}^{(\infty)}$-a.s.\ uniform boundedness of the $q$ moments of the BMA estimator for some $q>p$, implies, under the given assumptions, that said convergence \eqref{eq:convL1mom}  holds.  Consider to that end the  function  $\Phi$ defined by \eqref{eq:Phim}  in  the proof  of Theorem  \ref{thm equiv consistency}. For each $\varepsilon>0$ we have 
\begin{equation*}
    \begin{split}
        \int_{\cal M} \Phi(m) \Pi_n(dm)  \leq & \,  \varepsilon + \int_{{\cal M}  \cap \{ \Phi(m) \geq \varepsilon  \} } \Phi(m) \Pi_n(dm)  \\
        \leq & \, \varepsilon  + \left( \int_{{\cal M}   } \Phi(m)^{q/p}\Pi_n(dm) \right)^{p/q} \left( \Pi_n(m:  \Phi(m) \geq \varepsilon  )\right)^{1-p/q}
    \end{split}
\end{equation*}
The assumption on the $q-$moments of the BWA estimators imply that $\sup_n \int_{{\cal M}   } \Phi(m)^{q/p}\Pi_n(dm)  $ is finite,  $m_{0}^{(\infty)}$-a.s., following from applying Jensen's inequality to the convex function $s\mapsto |s|^{q/p}$ on $\RR$ and the integral w.r.t. $m(dx) $ defining $ \Phi(m)$. Since $\varepsilon>0$ is arbitrary, it just remains to ensure that $\Pi_n(m:  \Phi(m) \geq \varepsilon  )\to 0 $ $m_{0}^{(\infty)}$-a.s. as $n\to \infty$ (that is, the convergence \eqref{eq:convprobanmom} holds). The elementary relations   $t=t\wedge R + (t-R)_+ $ and $(t-R)_+\leq t\mathbf{1}_{\{t> R\}}$  for all $t,R\geq 0$ yield the bound
\begin{equation*}
    \begin{split}
    \Pi_n(m:  \Phi(m) \geq \varepsilon  ) \leq & \,  
    \Pi_n\left( m: \left\vert \int_{{\cal X}}  R\wedge d( x,x_0)^p m(dx) -  \int_{{\cal X}}R \wedge  d( x,x_0)^p   m_0(dx) \right\vert \geq \varepsilon/2\right)  \\
     & +  \,  
    \Pi_n\left( m:  \int_{\{ x\in {\cal X} : d( x,x_0)^p>R\}}  d( x,x_0)^p\left[ m(dx) + m_0(dx)\right]  \geq \varepsilon/2\right),  \\
    \end{split}
\end{equation*}
where the first term on the r.h.s.\ goes  $m_{0}^{(\infty)}$-a.s. to $0$ as $n\to \infty$, since $x\mapsto R\wedge d( x,x_0)^p$ is a bounded continuous function and $\Pi$ is consistent  at $m_0$ in the weak topology. The second term on the r.h.s.\ is bounded by
$$ \frac{2}{\varepsilon} \left[  \sup_{n\in \NN} \int_{\cal M} \int_{{\cal X} } \mathbf{1}_{ \{d( x,x_0)^p>R\}} d( x,x_0)^p  m(dx) \Pi_n(dm) +  \int_{{\cal X} } \mathbf{1}_{ \{d( x,x_0)^p>R\}} d( x,x_0)^p  m_0(dx)    \right]. $$
For each $\varepsilon>0$, this expression can be made arbitrarily small by taking $R$ large enough, since $m_0\in {\cal W}_p({\cal X})$ and because  the uniform boundedness of the $q-$moments of the BWA estimators implies their $p-$moments are uniformly  integrable.  We conclude that $\Pi_n(m:  \Phi(m) \geq \varepsilon  )\to 0 $ $m_{0}^{(\infty)}$-a.s.\ as $n\to \infty$ as desired. 
\end{proof}

\smallskip

\begin{example} \label{Ex:gaussBayesconsist}
\black{Given  Gaussian distributions  $m_0={\cal N}(\theta_0, \Sigma)  $ and $m={\cal N}(\theta, \Sigma)  $ on ${\cal X }=\RR^d$, one can check that  $D_{KL}\left(m_{0}||m\right) = \frac{1}{2}(\theta-\theta_0)^t\Sigma^{-1}(\theta-\theta_0)$ and hence that $\text{KL}\left( \Pi \right)= \supp(\Pi)$ in the  parametric Bayesian model dealt with in Examples  \ref{ExParaModel1} and  \ref{ex:gaussBWBvsBMA}. Therefore, by Theorem \ref{schartzweak}, $\Pi$ in those examples is consistent w.r.t.\  the weak topology at $m_0={\cal N}(\theta_0, \Sigma) $ for any $\theta_0 \in \RR^d$. Moreover, the prior $\Pi$ is easily seen to be in $\mathcal{W}_2(\mathcal{W}_2(\RR^d))$,
hence Lemma \ref{lemma m0KL PinWpWp}
ensures that  $\Pi_n \in \mathcal{W}_2(\mathcal{W}_2(\RR^d)),  m_{0}^{\otimes n}-$a.s. for all $n\geq 1$. This fact alternatively follows from the existence of a  $2-$Wasserstein barycenter for $\Pi_n$ for any data points $D$, verified in this case.
To verify consistency of $\Pi$ in the $2$-Wasserstein topology at any $m_0$ as well as  convergence of the BWB,   let us compute $W_2^2(\Pi_n,\delta_{m_0})$  and apply directly Theorem \ref{thm equiv consistency}. Denoting by  $\theta$ a r.v.\ with law 
${\cal N}\left( \hat{\theta}_{\ell_2}, (\Sigma_0^{-1}+ n \Sigma^{-1})^{-1}   \right)$, we  find }
 \begin{align*}
         W_2^2(\Pi_n,\delta_{m_0}) = & \,   \int_{\cal M} W_2^2(m,m_0) \Pi_n(dm) \\
         = & \, \EE_{\theta}(\Vert \theta - \theta_0\Vert^2) \\
          = & \,  \Vert  \theta_0- \hat{\theta}_{\ell_2} \Vert^2+ \EE_{\theta}( \Vert  \hat{\theta}_{\ell_2}- \theta \Vert^2 )\\
          = & \, W_2^2( \hat m_2^n,m_0) +  \int_{\cal M} W_2^2(m,\hat m_2^n) \Pi_n(dm),
 \end{align*}
 with  $\EE_{\theta}( \Vert  \hat{\theta}_{\ell_2}- \theta \Vert^2 )= \int_{\cal M} W_2^2(m,\hat m_2^n)\Pi_n(dm)= tr((\Sigma_0^{-1}+ n \Sigma^{-1})^{-1})\to 0$ as $n\to \infty$ and  
 $$  \Vert  \theta_0- \hat{\theta}_{\ell_2} \Vert^2 = W_2^2( \hat m_2^n,m_0)=   \Vert(\Sigma_0^{-1}+ n \Sigma^{-1})^{-1}  (\Sigma_0^{-1}\mu_0 +  n \Sigma^{-1} \bar{x}_n)-  \theta_0\Vert^2 $$
 which $m_{0}^{(\infty)}$ a.s. goes to $0$  as $n\to \infty$ by the Law of Large Numbers,  whenever the data  $(x_n)_{n\geq 1}$ consists of an i.i.d. sample from the true model $m_0$. Notice that the identity
 $W_2^2(\Pi_n,\delta_{m_0}) = W_2^2( \hat m_2^n,m_0) +  \int_{\cal M} W_2^2(m,\hat m_2^n) \Pi_n(dm) $
 found in this case, can be seen as a {\it bias-variance type decomposition}   for the posterior law $\Pi_n$ on ${\cal W}_2(\cal X)$. In this case, one can readily see that $\EE(W_2^2( \hat m_2^n,m_0))\leq C/n$. 
 
\end{example}

\begin{example} \label{Ex:PoissBayesconsist}
\black{Although the BWB estimator $	\hat{m}_1^n$ in the Poisson parametric Bayesian model of Example \ref{ExParaModel2}  is not explicit (see Example \ref{ex:PoissonBWB}),  we can still apply Theorem \ref{thm equiv consistency} a) to prove that it converges to the true model generating the data.  Indeed, if $m_0=$Pois$(\lambda_0)$, using the expression for the $1-$Wasserstein distance between Poisson laws computed in Example \ref{ex:WdPoisson} we see that 
\begin{equation*}
    \begin{split}
W_1(\Pi_n,\delta_{m_0})= &  \int W_1(m,m_0)\Pi_n(dm) = \EE|\lambda -\lambda_0| \leq  \EE(|\lambda -\lambda_0|^2)^{1/2} ,
%\quad  \mbox{ and} \\
%W_2(\Pi_n,\delta_{m_0})^2= &  \int W_2(m,m_0)^2\Pi_n(dm) = \EE(|\lambda -\lambda_0|^2)+ \EE|\lambda -\lambda_0| \leq  \EE(|\lambda -\lambda_0|^2)+  \EE(|\lambda -\lambda_0|^2)^{1/2}  \\
    \end{split}
\end{equation*}
with $\lambda$ a r.v.\ with law Gamma$(n \bar{x}_n+\alpha,n +\beta)$. An elementary computation using the Gamma distribution's mean and variance shows, in this case,  that
$$\EE(|\lambda -\lambda_0|^2) =\frac{n \bar{x}_n+\alpha}{(n +\beta)^2}  +  \left[  \frac{n \bar{x}_n+\alpha}{n +\beta}- \lambda_0  \right]^2, $$
which  goes to zero  $m_{0}^{(\infty)}$ a.s.\ as $n\to \infty$, whenever the data  $(x_n)_{n\geq 1}$ is an i.i.d. sample from the true model $m_0$. We deduce with Corollary \ref{lemma m0KL PinWpWp} that $\EE( W_1(m_0,\hat m_1^n)) \leq C/\sqrt{n}$.  }
\end{example}

\smallskip

As noticed earlier, consistency of $\Pi$ at $m_0$ w.r.t.\ the $p$-Wasserstein  topology is not enough to grant that  $W_p(\Pi_n, \delta_{m_0}) \to 0\,\,m_{0}^{(\infty)}-a.s.$. But we will see next that this is true under a  boundednees condition on the support of $\Pi$. Recall that 
$\supp(\Pi)$ is  said to be bounded in ${\cal W}_p({\cal X})$ if  $$\text{diam}(\Pi):=\sup\limits_{m,\bar m\in \supp(\Pi)} W_p(m,\bar m)<\infty.$$ 
   A typical example is  the finitely parametrized case with compact  parameter space $\Theta$ and continuous  parametrization mapping ${\cal T}:\Theta  \to {\cal W}_p({\cal X})$.  More generally,  $\text{diam}(\Pi)<\infty $  amounts to $\Pi$ being  supported on a set of models with  centered $p-$moments  bounded by a constant.  In particular,   $\mathcal X$ and the support of every $m\in \supp(\Pi)$ may be unbounded, and still $\supp(\Pi)$ be bounded.   We have 
\begin{lemma}
	\label{lemma:wasserstein-consistency-convergence}
	If $\Pi$ is  consistent at $m_0$ in the $p$-Wasserstein topology and $\text{supp}(\Pi)$ is bounded in ${\cal W}_p({\cal X})$, then $W_p(\Pi_n, \delta_{m_0}) \to 0\,\,m_{0}^{(\infty)}-a.s.$ as $n\to \infty$.
	\end{lemma}
	\begin{proof}
		Let  $\epsilon>0$ and $B=\{m: W_p(m,m_0)< \epsilon\}$, then
		\begin{align*}
			W_p(\Pi_n,\delta_{m_0})^p &= \textstyle  \int_{\mathcal M}W_p(m,m_0)^p\,\Pi_n(dm)\\
			&=\textstyle  \int_{B}W_p(m,m_0)^p\,\Pi_n(dm) +  \int_{B^c}W_p(m,m_0)^p\,\Pi_n(dm)\\
			&\leq \textstyle  \epsilon^p + \int_{B^c}W_p(m,m_0)^p\,\Pi_n(dm).
		\end{align*}
		Since  $\epsilon>0$ is arbitrary, we only need to check that the second term in the last line goes $m_{0}^{(\infty)}$-a.s. to zero as $n\to \infty$. By consistency we have  $\Pi_n(B^c) \to 0\,\,,  m_{0}^{(\infty)}-$a.s. as $n\to \infty$, and since $\text{supp}(\Pi_n)\subset \text{supp}(\Pi)$, we conclude that
		\begin{align*}
			 \textstyle   \int_{B^c}W_p(m,m_0)^p\,\Pi_n(dm) \leq \text{diam}(\Pi)^p\Pi_n(B^c) \rightarrow 0, \,\, m_{0}^{(\infty)}-a.s.
		\end{align*}
	\end{proof}
\begin{remark} \black{$\Pi$  consistent at $m_0$ in the weak topology and $\text{supp}(\Pi)$  bounded in ${\cal W}_p({\cal X})$ is in general not enough to obtain the  conclusion of Lemma \ref{lemma:wasserstein-consistency-convergence}.} \black{To see this, write \begin{align}
\label{eq:explanation_for_referee}
  W_p(\Pi_n,\delta_{m_0})^p=\int_{B_w} W_p(m,m_0)^p\,\Pi_n(dm) + \int_{B_w^c} W_p(m,m_0)^p\,\Pi_n(dm),   
\end{align}
where $B_w$ is a fixed small weak neighborhood of $m_0$ (e.g.\ one can use the bounded Lipschitz distance to build $B_w$). As in the proof of Lemma \ref{lemma:wasserstein-consistency-convergence}, the second term in the r.h.s.\ of \eqref{eq:explanation_for_referee} goes $m_{0}^{(\infty)}$-a.s.\ to zero as $n\to \infty$ in this case too. However there is no reason why the term $\int_{B_w} W_p(m,m_0)^p\,\Pi_n(dm)   $ should be small. The reason is that for the functional $m\mapsto \Psi(m):= W_p(m,m_0)^p$, even if bounded on $\mathcal M:=\supp(\Pi)$, there is no reason why $\Psi|_{B_w\cap\mathcal M}$ should be small (no matter how small $B_w$ may be). Indeed, the statement ``$\Psi|_{B_w\cap\mathcal M}$ is small if  $B_w$ is small'' would mean mathematically that the weak and the $p-$Wasserstein topologies coincide on $\mathcal M$ locally around $m_0$, but this is not true in general, even if $\mathcal M$ is bounded\footnote{For instance, for $\cal X=\R$ and $p=1$ we have that $\mathcal M:= \{m_n:=\frac{n-1}{n}\delta_0+\frac{1}{n}\delta_n\}_{n\in\mathbb N}\cup\{\delta_0\}$ is $1-$Wasserstein bounded, and yet $m_n\to \delta_0$ weakly but not in $1-$Wasserstein topology.}. This should not be confused with the fact that, if $\mathcal M$ equiped with the  $p-$Wasserstein topology  and metric is bounded, then on $\mathcal W_p(\mathcal M)$ weak and $p$-Wasserstein convergence coincide. 
Consistency at $m_0$ of $\Pi$ in the weak topology, on the other hand, is equivalent to $\Pi_n\to \delta_{m_0}$,  $m_{0}^{(\infty)}$-a.s.\ in the weak topology of $\mathcal W_p(\mathcal M)$, when $\mathcal M$ is equipped with the weak topology. This does not imply $\Pi_n\to \delta_{m_0}$, $m_{0}^{(\infty)}$-a.s.\ in the weak topology of $\mathcal W_p(\mathcal M)$, when $\mathcal M$ is equipped with the $p-$Wasserstein topology, and hence provides another point of view as to why the l.h.s.\ of \eqref{eq:explanation_for_referee} does not go to zero without stronger assumptions.   }
\end{remark}

The next  result  based on Schwartz theorem  provides a (rather strong) condition ensuring that the equivalent properties in  Theorem \ref{thm equiv consistency} hold. Unfortunately, we have not been able to prove such a result under more general  assumptions.

\begin{theorem}\label{thm main consistency}
Under  \black{Assumptions \ref{assum:geodesic} and \ref{ass nice}}, suppose moreover that  $m_0 \in KL(\Pi)$ and that,  for some  $\lambda_0>0$ and $ x_0 \in \mathcal{X}$,  one has $$\textstyle \sup\limits_{m\in \supp(\Pi) } \int_{\mathcal{X}}e^{\lambda_0 d^p(x,x_0)}m(dx) < +\infty.$$
Then, $\Pi$  is consistent at $m_0$ in the $p$-Wasserstein topology. Moreover, we have 
 $W_p(\Pi_n,\delta_{m_0}) \rightarrow 0$, $m_{0}^{(\infty)}$-a.s. and the BWB estimator is consistent in the sense that  $$W_p(\hat m_p^n,m_0)\to 0,\,\,\,m_{0}^{(\infty)}-a.s.$$ 
\end{theorem}

Before proving Theorem \ref{thm main consistency} some remarks on its assumptions and proof are in order.

\begin{remark}\label{rem boundsupp} \black{The uniform control assumed on $p-$exponential moments implies, by Jensen's inequality, that 
$\sup\limits_{m\in \supp(\Pi) } W_p^p(m,\delta_{x_0})= \sup\limits_{m\in \supp(\Pi) } \int_{\mathcal{X}} d^p(x,x_0)m(dx)   < +\infty.$
By triangle inequality, this in turn implies  that 
$\supp(\Pi)$ is bounded. }
\end{remark}

\begin{remark} \black{The general picture of Bayesian consistency  (including the misspecified case) parallels in several aspects Sanov's large deviations theorem  (see e.g.\ the Bayesian Sanov Theorem in  \cite[Theorem in  2.1]{grendarjudge} and references therein), and it similarly  relies  on exponential controls of  (posterior) integrals.    The proof of   Theorem \ref{thm main consistency}  follows the argument of \cite[Example 6.20]{ghosal2017fundamentals}, where  Theorem \ref{schartzweak} above is obtained by combining  Schwartz' theorem (\cite[Theorem 6.17]{ghosal2017fundamentals}) with Hoeffding's concentration inequality, to get uniform exponential controls of the posterior mass of complements of weak neighborhoods of $m_0$, which are defined in terms of  bounded random variables.  
 A $p-$exponential moment control is what is needed to derive such    concentration inequalities for unbounded random variables (e.g.\ moments),  defining neighborhoods in the Wasserstein topology.  Notice  that finite  $p-$exponential moments  are also required for  Sanov's   theorem to hold in the  $p-$Wasserstein topology \cite{wang2010sanov}. The uniform bound on exponential moments appears however too strong an assumption (it does not hold e.g.\ in the setting of Example \ref{Ex:gaussBayesconsist}). The question of   relaxing that condition is left for future work.} 
\end{remark}

\begin{remark}
\black{ In the \emph{misspecified} framework dealt with in \cite{berk1966, grendarjudge, kleijn2004bayesian, kleijn2012bernstein, ghosal2017fundamentals}, and paralleling results  applicable to the weak topology in those works,   we expect the convergence $\hat m_p^n\to \argmin_{m\in\mathcal M} D_{KL}\left(m_{0}||m\right) $ w.r.t.\  $W_p$ to hold,  under suitable assumptions.} 
\end{remark}

\begin{proof}[Proof of Theorem \ref{thm main consistency}]
	%We want to apply Proposition \ref{prop:wasserstein-consistency-convergence}. 
	First we show that if $U$ is any $\mathcal W_p(\mathcal X)$-neighbourhood of $m_0$ then $m_0^{(\infty)}$-a.s.\ we have $\liminf_n \Pi_n(U)\geq 1$. %It is then clearly enough to prove, for each $\epsilon>0$ that $U := \{m:W_p(m,m_0)<\epsilon\}$. 
	According to Schwartz Theorem in the extended form  \cite[Theorem 6.17]{ghosal2017fundamentals}, under the assumption that $m_0 \in KL(\Pi)$, it is enough to find for each such $U$ a sequence of measurable functions or ``tests'' $\phi_n:\mathcal X^n\to [0,1]$ such that
	\begin{enumerate}
		\item $\phi_n(x_1,\dots,x_n)\to 0, \,\,\, m_0^{(\infty)}-a.s$, and
		\item $\limsup_n \frac{1}{n}\log\left(\int_{U^c} m^{\otimes n}(1-\phi_n)\Pi(dm)\right) <0.$
	\end{enumerate}
First we will construct tests $\{\phi_n\}_n$ that satisfy Point 1 and Point 2 above, over an appropriate subbase of neighbourhoods $U$, to finally extend those properties to general neighborhoods.

Recall that $\mu_k \to \mu$ in $W_p(\cal X)$ iff for all continuous functions $\psi$  on $\cal X$ with $|\psi(x)| \le K(1+d^p(x,x_0))$  for some $K \in \mathbb{R}_+$ it holds that $\int_{\mathcal{X}}\psi(x)\mu_n(dx) \to \int_{\mathcal{X}}\psi(x)\mu(dx)$; see \cite{villani2008optimal}.  Given such $\psi$ and $\epsilon > 0$ we define the open sets $$ \textstyle  U_{\psi,\epsilon} := \left\{m: \int_{\mathcal{X}} \psi(x) m(dx) < \int_{\mathcal{X}} \psi(x) m_0(dx) + \epsilon\right\},$$ which form a sub-base for the $p$-Wasserstein neighborhood system at the distribution $m_0$. We can assume that $K=1$ by otherwise considering $U_{\psi/K, \epsilon/K} $ instead. Given a neighborhood $U := U_{\psi,\epsilon}$ of $m_0$ as above, we define the test functions
$$\phi_n(x_1,\dots,x_n)  = \left \{
\begin{array}{ll}
1 & \text{ if }\frac{1}{n}\sum_{i=1}^n \psi(x_i) > \int_{\mathcal{X}} \psi(x) m_0(dx)  + \frac{\epsilon}{2},\\
0&\text{otherwise}.
\end{array}
\right .
$$
By the law of large numbers, $m_0^{(\infty)}-a.s: \phi_n(x_1,\dots,x_n)\to 0$, so Point 1 is verified. Point 2 is trivial if $r:=\Pi(U^c)=0$, so we assume from now on that $r>0$. Thanks to the exponential moments control assumed on $\supp(\Pi)$, for $\Pi(dm)$ a.e. $m$  the random variable $Z = 1 + d^p(X,x_0)$ with $X \sim m$ has a moment-generating function $\mathcal{L}_m(t)$ which is finite for all $t\in [0, \lambda_0]$, namely $$ \textstyle   \mathcal{L}_m(t):=\mathbb{E}\left[e^{tZ}\right]=e^t\int_{\mathcal{X}}e^{t d^p(x,x_0)}m(dx) <+\infty.$$
We thus have the bounds  
$$ \int_{\mathcal X}|\psi(x)|^k m(dx) \leq \mathbb{E}\left[Z^{k}\right] \leq k!\mathcal{L}_m(t)t^{-k},\,\, \forall \lambda_0\geq t>0, $$
for all  $k\in \NN$. Therefore, we may  apply   Bernstein's inequality in the form of \cite[Corollary 2.10]{massart2007concentration} to the random variables $\{-\psi(x_i)\}_i$ under the measure $ m^{(\infty)}$ on $\mathcal X^{\mathbb N}$, obtaining for any $\alpha<0$ that
\begin{align*}
	 \textstyle   {m}^{(\infty)}\left(\sum_{i=1}^n \left[\psi(x_i) - \int_{\mathcal{X}} \psi(x) {m}(dx)  \right]\leq \alpha\right) \leq e^{ -\frac{\alpha^2}{2(v-c\alpha)}},
\end{align*}
with $v:=2n\mathcal{L}_{{m}}(t)t^{-2}$, $c:=t^{-1}$, and $0<t\leq \lambda_0$. Going back to the tests $\phi_n$ and using the definition of $U^c$ we deduce that
\begin{align*}
	 \textstyle  \int_{U^c} m^{\otimes n}(1-\phi_n)\Pi(dm) 
	&=  \textstyle  \int_{U^c}m^{\otimes n}\left( \frac{1}{n}\sum_{i=1}^n \psi(x_i) \leq \int_{\mathcal{X}} \psi(x) m_0(dx)  + \frac{\epsilon}{2}\right)\Pi(dm)\\
	& \textstyle  \leq \int_{U^c}m^{\otimes n}\left( \frac{1}{n}\sum_{i=1}^n \psi(x_i) \leq \int_{\mathcal{X}} \psi(x) m(dx)  - \frac{\epsilon}{2}\right)\Pi(dm)\\
	&= \textstyle   \int_{U^c}m^{\otimes n}\left(\sum_{i=1}^n \left[\psi(x_i) - \int_{\mathcal{X}} \psi(x) m(dx)  \right]\leq -\frac{n\epsilon}{2}\right)\Pi(dm)\\
	& \textstyle  \leq \int_{U^c} \exp\left\{-\frac{n \epsilon^2}{2} \frac{t^2}{8\mathcal{L}_{m}(t)+t\epsilon}\right\}\Pi(dm) \\
	& \textstyle  \leq r\, \exp\left\{-\frac{n \epsilon^2}{2} \frac{t^2}{8\sup_{m\in \text{supp}(\Pi)}\mathcal{L}_{m}(t)+t\epsilon}\right\} .
\end{align*}
Thanks  to the  uniform control of  exponential moments on the support,  we conclude as desired that
$$ \textstyle  \limsup_n \frac{1}{n}\log\left(\int_{U^c} m^{\otimes n}(1-\phi_n)\Pi(dm)\right) \leq - \frac{t^2\epsilon^2}{16\sup_{m\in \text{supp}(\Pi)}\mathcal{L}_{m}(t)+2t\epsilon} <0.$$
Now, a general neighborhood $U$ contains a finite intersection of, say,  $N\in\mathbb N$ elements of the sub-base, i.e. $\bigcap_{i=1}^N U_{\psi_i,\epsilon_i} \subset U$, so
\begin{align*}
	 \textstyle  \int_{U^c} m^{\otimes n}(1-\phi_n)\Pi(dm) & \textstyle  \leq \sum_{i=1}^{N} \int_{U_{\psi_i,\epsilon_i}^c} m^{\otimes n} (1-\phi_n)\Pi(dm) .
\end{align*}
Therefore we can conclude as in the sub-base case that Point 2 is verified. All in all, we have established that $\Pi$ is $p$-Wasserstein consistent at $m_0$. Thanks to Lemma  \ref{lemma:wasserstein-consistency-convergence} and the boundedness of $\supp(\Pi)$ (see Remark \ref{rem boundsupp}) the last two  claims are  immediate. 
\end{proof}

\section{BWB calculation through descent  algorithms in Wasserstein space} 
\label{sec:computation}

In this section we  review  some of the available methods to compute  Wasserstein barycenters, and then explain  how  they can be used  to  calculate the  BWB estimator.  We will first survey  the method proposed by \'Alvarez-Esteban, Barrio, Cuesta-Albertos and Matr\'an in \cite[Theorem 3.6]{alvarez2016fixed} and  by Zemel and Panaretos in \cite[Theorem 3, Corollary 2]{panaretos2017frechet},  applicable in the setting where the probability distribution $\Gamma$  on  the model space  has finite  support. This method is  interpreted in \cite{panaretos2017frechet}   as a {\it gradient descent in the Wasserstein space}.    We will then discuss the main aspects  of
the  {\it stochastic gradient descent in Wasserstein space}   (SGDW), introduced in our companion paper \cite{bakchoff2022}, building also upon the gradient descent idea. The latter method  allows moreover the computation of population barycenters for a distribution $\Gamma$ on the model space, using  a streaming of random probability distributions sampled from it. In particular,  we will recall the conditions established in \cite{bakchoff2022} which ensure its convergence. The reader is referred to \cite{peyre2019computational,CuPe16,MaGeMi21,dognin2019wasserstein} for alternative approaches to computing  Wasserstein barycenters.
\smallskip  

The two discussed methods can be easily implemented when  explicit analytical expressions for optimal transport  maps between distributions in the model space are available  (see  \cite{bakchoff2022}  or Section \ref{sec numeric} for some examples). Moreover, in that case these methods 
 can easily be coupled with sampling procedures (MCMC or others) for the posterior distribution $\Pi_n$  in order to compute the BWB. 
We will see that the computation of the BWB   through the SGDW  has several advantageous features. In particular, when expressions for optimal transport maps  are known, it can be done at nearly the same  cost as the posterior sampling.

\smallskip

From now on  we specialize  Assumption \ref{assum:geodesic}, and make the following set of assumptions: 
\begin{assumption}  \label{ass_A1}
	$p=2$, $\mathcal X=\mathbb R^q$, $d=$ Euclidean metric, $\lambda=$ Lebesgue measure.  Furthermore, $\Gamma\in \mathcal{W}_2(\mathcal{W}_{p,ac}(\R^q))$ and there is a model space  $\mathcal M\subseteq \mathcal W_{p,ac}(\R^q)$ for $\Gamma$ which is weakly closed. 
\end{assumption}

The following concept will be central in the ``gradient-type'' algorithms we consider:   

\begin{definition}
We say that $\mu\in\mathcal W_{2,ac}(\R^q)$ is a Karcher mean of $\Gamma\in \mathcal{W}_2(\mathcal{W}_{2,ac}(\R^q)) $ if
$$\textstyle\mu\left( \left\{x: x =\int_{\mathcal W_2(\R^q)} T_\mu^m(x)\Gamma(dm)\right\} \right)=1.$$ 
\end{definition}
\begin{remark} It is known that any 2-Wasserstein barycenter is a Karcher mean (c.f.\ \cite{panaretos2017frechet}). However, the class of Karcher means is in general a strictly larger one, see \cite{alvarez2016fixed}. For conditions ensuring uniqueness of Karcher means, see  \cite{panaretos2017frechet,PaZe20}  for the case when the support of $\Gamma $ is finite and  \cite{bigot2018characterization} for the  case of an infinite support. In one dimension, the uniqueness of Karcher means holds without further assumptions. See  \cite{bakchoff2022} for a deeper discussion.  
\end{remark}

\subsection{Gradient descent on Wasserstein space}\label{sec grad desc}\label{sec determ descent}

Consider $\Gamma\in \mathcal{W}_2(\mathcal{W}_{p,ac}(\R^q))$ finitely supported:  for some $m_i\in \mathcal{W}_{p,ac}(\R^q) $, $i=1,\dots ,L\in\N $, we have 
$$\Gamma= \textstyle \sum_{i\leq L}\lambda_i\delta_{m_i}.$$
Following \cite{alvarez2016fixed} and \cite{panaretos2017frechet}, we  define an operator over $\mathcal{W}_{p,ac}(\mathcal{X}) $ by 
\begin{align}
\label{eq:G-operator}
  G(m) \textstyle := \left(\sum_{i=1}^{L}\lambda_iT_m^{m_i}\right)(m).
\end{align}
Starting from $\mu_0 \in \mathcal{W}_{p,ac}(\mathcal{X})$ one can then define the sequence
\begin{align}
\label{eq:fixed-point}
  \mu_{n+1} := G(\mu_n), \text{ for } n \geq 0.
\end{align}
The next result proven  in \cite[Theorem 3.6]{alvarez2016fixed} and independently  in \cite[Theorem 3, Corollary 2]{panaretos2017frechet}, establishes the convergence of the above sequence to a fixed-point of $G$, which is nothing other than a Karcher mean for $\Gamma$:

\begin{proposition}
	The sequence $\{\mu_n\}_{n \geq 0}$ in eq.~\eqref{eq:fixed-point} is tight and every weakly convergent subsequence of $\{\mu_n\}_{n \geq 0}$  converges in ${W}_2$ to an absolutely continuous measure in $\mathcal{W}_{2}(\mathbb{R}^q)$ which is a Karcher mean of $\Gamma$. If some $m_i$ has a bounded density, and if there exists a unique Karcher mean $\hat{m}$, then $\hat{m}$ is the Wasserstein barycenter of $\Gamma$ and  $W_2(\mu_n, \hat{m}) \rightarrow 0$.
\end{proposition}
 Panaretos and Zemel  \cite[Theorem 1]{panaretos2017frechet} discovered that the sequence \eqref{eq:fixed-point}  can indeed be interpreted as a gradient descent (GD) scheme with respect to the  \emph{Riemannian-like} structure  of the Wasserstein space $\mathcal{W}_{2}(\mathbb{R}^q)$.  In fact,  
the functional  on $\mathcal{W}_2(\R^q)$ given by 
\begin{align*}\label{eq:FGD}
F(m) :=& \textstyle  \frac{1}{2}\sum_{i=1}^{L}\lambda_iW_2^2(m_i,m)
\end{align*}
has a {\it Frechet derivative}  at each point $m\in  \mathcal{W}_{2,ac}(\R^q) $, given by
\begin{eqnarray*}
\label{eq:Frechet-derivative}
F^\prime(m) = \textstyle  -\sum_{i=1}^{L}\lambda_i(T_m^{m_i}-I) = I-\sum_{i=1}^{L}\lambda_iT_m^{m_i} \in L^2(m),
\end{eqnarray*}
where $I$ is the identity map in $\mathbb R^q$. 
  This means that  for each such $m$,    one has
  \begin{equation}\label{eq:FrechDeriv}
  \frac{F(\hat{m})-F(m)-\int_{\R^q} \langle  F'(m)(x), T_m^{\hat{m}} (x)- x \rangle m(dx)  }{W_2(\hat{m},m)} \longrightarrow 0 , 
  \end{equation}
  when $W_2(\hat{m},m)$ goes to zero, by virtue of \cite[Corollary 10.2.7]{ags08}. 
It follows from Brenier's theorem \cite[Theorem 2.12(ii)]{villani2003topics} that $\hat{m}$ is a fixed point of $G$ defined in eq.~\eqref{eq:G-operator} if and only if $F^\prime(\hat{m}) = 0$. The gradient descent sequence in Wasserstein space (GDW)  with step $\gamma$ starting from $\mu_0 \in \mathcal{W}_{2,ac}(\mathbb{R}^q)$   is defined by (c.f.\ \cite{panaretos2017frechet})
\begin{align*}
\label{eq:gradient-descent}
\mu_{n+1}:= G_{\gamma}(\mu_n), \text{ for } n \geq 0, \mbox{ where }
\end{align*}
\begin{align*}G_{\gamma}(m) &:=  \textstyle \left[I + \gamma F^\prime(m)\right](m)= \left[(1-\gamma)I + \gamma \sum_{i=1}^{L}\lambda_iT_m^{m_i}\right](m)= \left[I + \gamma \sum_{i=1}^{L}\lambda_i(T_m^{m_i}-I)\right](m), 
\end{align*}
and it coincides with the sequence in eq.~\eqref{eq:fixed-point}  if $\gamma=1$. These ideas serve as inspiration for the stochastic gradient descent iteration in the next part.

\subsection{Stochastic gradient descent for population barycenters}\label{sec stoch grad}

We recall next the stochastic gradient descent sequence introduced in the companion paper \cite{bakchoff2022}, where the reader is referred to for details. Additionally to the assumptions in the previous part, we will also make use of an extra one introduced in \cite{bakchoff2022}:

\begin{assumption}
 $\Gamma$  \textcolor{black}{has a $W_2$-compact model space $K_\Gamma\subset\mathcal W_{2,ac}(\mathcal \R^q )$. Moreover this set is \emph{geodesically \black{convex}}:   for every $\mu,\nu\in K_\Gamma$ and $t\in[0,1]$,  $((1-t)I+tT_\mu^\nu)(\mu)\in K_\Gamma$, with $I$ the identity operator.}
\label{ass_A2}
\end{assumption}

In particular, under these assumptions,   for each $\nu \in \mathcal W_2(\R^q)$ and $\Gamma(dm) $ a.e.\ $m$, there is a unique optimal transport map $T_m^\nu$ from $m $ to  $\nu$ and, {\color{black} by  \cite[Proposition 6]{le2017existence}, the  2-Wasserstein population barycenter is unique}.  We notice that, although strong at first sight, assumption \ref{ass_A2}   can be guaranteed in suitable parametric situations (e.g., Gaussian, or even the location scatter setting recalled in Section \ref{sec scat loc}), or under moment and density constraints on the measures in $K_\Gamma$ (e.g., under uniform bounds on their moments of order $2+\epsilon$ and their Boltzmann entropy, which are geodesically convex  functionals, see \cite{ags08}).

\begin{definition}
	Let {\color{black}$\mu_0 \in K_\Gamma$}, $m_k \stackrel{\text{iid}}{\sim}  \Gamma$,  and $\gamma_k > 0$ for $k \geq 0$. We define the stochastic gradient descent in Wasserstein space  sequence (SGDW)  by
	\begin{equation}
	\label{eq:sgd-seq}
	\mu_{k+1} := \left[(1-\gamma_k)I + \gamma_k T_{\mu_k}^{m_k}\right](\mu_k) \text{ , for } k \geq 0.
	\end{equation}
\end{definition}
%The reasons as to why we can truthfully refer to the above sequence as stochastic gradient descent will become apparent in Sections \ref{sec grad desc} and \ref{sec stoch grad}.  

The sequence is a.s.\ well-defined, as one can show by induction that $\mu_k\in \mathcal W_{2,ac}(\R^q)$ a.s.~ thanks to Assumption \ref{ass_A2}. 
 The rationale for definition \ref{eq:sgd-seq} is similar to that of Section \ref{sec determ descent}, though now we wish to emphasize the population case:  If we call  now 
\begin{align}
F(\mu) :=& \textstyle  \frac{1}{2}\int_{\mathcal W_2(\R^q)} W_{2}^2(\mu,m)\Gamma(dm)\label{eq:FSGD}
\end{align}
the functional minimized by a 2-Wasserstein barycenter, then we (formally at least)  expect
\begin{align}
F^\prime(\mu) (x)=& \textstyle  -\int_{W_2(\R^q)} (T_\mu^{m}-I))\Gamma(dm)(x). \label{eq:F'SGD}
\end{align}
Hence, $(I-T_\mu^{m})$ with $m\sim \Gamma$, is an unbiased estimator of $F'(\mu)$. This  immediately suggests  the stochastic descent sequence \eqref{eq:sgd-seq}  introduced in  \cite{bakchoff2022},  drawing inspiration from the classic SGD ideas \cite{RoMo51}. 

 Clearly $\mu $ is a Karcher mean  for $\Gamma$ iff $\|F'(\mu) \|_{L^2(\mu)}=0$.  Just like for the GD sequence, the SGD sequence is typically expected to converge to stationary points, or   Karcher means in the present setting, rather than to minimisers. Next  theorem  provides  sufficient conditions for the  SGDW sequence to  a.s.\ converge to a Wasserstein barycenter, and is the main result of  \cite{bakchoff2022}.    
The  following   assumption on the steps $\gamma_k$, standard in the framework of SGD methods, is needed: 
\begin{equation}
\label{eq:cond1_2-gamma}
 \textstyle \sum_{k=1}^{\infty}\gamma_k^2 < \infty \qquad \mbox{ and }  \qquad
\textstyle \sum_{k=1}^{\infty}\gamma_k = \infty.
\end{equation}
\textcolor{black}{
\begin{theorem}
	\label{thm:sgd-convergence}
	Suppose Assumptions \ref{ass_A1} and \ref{ass_A2}, as well as  conditions \eqref{eq:cond1_2-gamma} hold. Furthermore, suppose that $\Gamma$ admits a unique Karcher mean. Then,  the SGD sequence $\{\mu_k\}_k$ in eq.~\eqref{eq:sgd-seq} is a.s.~convergent to the unique 2-Wasserstein barycenter $\hat{\mu}$ of $\Gamma$. Moreover, we have  $\hat{\mu}\in K_\Gamma$.
\end{theorem}}

\subsection{Batch stochastic gradient descent on Wasserstein space}
\label{sec:generalization-sgd}

We briefly recall how the variance of the SGDW sequence can be reduced by using batches:

\begin{definition}
	Let $\mu_0 \in K_\Gamma$, $m_k^i \stackrel{\text{iid}}{\sim}  \Gamma$, and $\gamma_k > 0$ for $k \geq 0$ and $i=1,\dots,S_k$. The batch stochastic gradient descent (BSGD) sequence is given by 
	\begin{align}		\label{eq:sgd-batch-seq} \textstyle
	\textstyle\mu_{k+1} := \left[(1-\gamma_k)I + \gamma_k \frac{1}{S_k}\sum_{i=1}^{S_k} T_{\mu_k}^{m_k^i}\right](\mu_k).
	\end{align}
\end{definition}

The following two results, extracted from \cite{bakchoff2022}, justify the above definition: The first result states that this sequence is still converging, while the second one states that batches help \emph{reducing noise}:

\begin{proposition}
	\label{thm:sgd-batch-convergence}
	Under the assumptions of Theorem \ref{thm:sgd-convergence} the BSGD sequence $\{\mu_t\}_{t\geq0}$ in eq.~\eqref{eq:sgd-batch-seq} converges a.s.\ to the 2-Wasserstein barycenter of $\Gamma$.
\end{proposition}
\begin{proposition}\label{prop batch noise reduction}
 {\color{black} The batch estimator  for $F^\prime(\mu)$ of batch size $S$, given  by} $-\frac{1}{S}\sum_{i=1}^{S}(T_{\mu}^{m_i}-I)$, has  integrated variance $$\textstyle\mathbb{V}[-\frac{1}{S}\sum_{i=1}^{S}(T_{\mu}^{m_i}-I)] := \int \text{Var}_{(m_i)\sim\Pi^{\otimes S}} \left[\frac{1}{S}\sum_{i=1}^S(T_\mu^{m_i}(x)-x)\right ]\mu(dx) = \mathcal{O}(\frac{1}{S}), $$
 i.e. $\mathbb{V}[-\frac{1}{S}\sum_{i=1}^{S}(T_{\mu}^{m_i}-I)]$ decreasing linearly in the batch size.
\end{proposition}

\subsection{Computation of the BWB}\label{sec:computation_bayes_bary}
It is immediate to deduce  a simple methodology  based on the SGDW algorithm, to compute the BWB estimator for general (finitely or infinitely supported) posterior laws $\Pi_n\in  \mathcal{W}_2(\mathcal{W}_{2,ac}(\R^q))$.

We make the practical assumption that we are capable of generating independent models $m_i$ from the posteriors $\Pi_n$  (in the parametric setting, this can be done through  efficient Markov Chain Monte Carlo (MCMC) techniques  \cite{andrieu2003introduction,goodman2010ensemble,brooks2011handbook} or  transport sampling procedures \cite{el2012bayesian,parno2015transport,kim2015tractable,marzouk2016introduction}). 
On the theoretical side, we  assume conditions \ref{ass_A1} and \ref{ass_A2} are satisfied by $\Gamma=\Pi$,  the prior law on models,  implying  the posterior $\Pi_n$ a.s.  satisfies  those conditions for all $n$ too.  

The proposed method can be sketched as follows:
\begin{enumerate}
    \item Given a prior on models $\Pi $ and data $x_1,\dots, x_n$,  sample $\mu^{(n)}_0 \sim  \Pi_n$  and set $k=0$.
    \item  Sample $m^{(n)}_{k}$ independent from $\mu^{(n)}_0 ,m^{(n)}_0,\dots,  m^{(n)}_{k-1}$
    \item 
   Set 
	\begin{equation*}
	\mu^{(n)}_{k+1} := \left[(1-\gamma_k)I + \gamma_k T_{\mu^{(n)}_k}^{m^{(n)}_k}\right](\mu^{(n)}_k) \text{ , for } k \geq 0.
	\end{equation*} 
	\item  Increase $k$ by $1$ and go to (2).
\end{enumerate}  
The algorithm can be run until $k\in \N$  large enough such that the squared $2-$Wasserstein distance 
$$W_2^2(\mu^{(n)}_{k+1} , \mu^{(n)}_{k})= \gamma_k^2 \int_{\R^q}  |x-  T_{\mu^{(n)}_k}^{m^{(n)}_k} (x)|^2 d\mu^{(n)}_k(x)$$  between $\mu^{(n)}_{k+1}$ and $ \mu^{(n)}_{k}$  is repeatedly smaller than some given positive  threshold.   

Moreover, under the assumptions of Theorem \ref{thm main consistency},  $\Pi$ is consistent in the $W_2$-topology at the law $m_0$ of the data $x_1,\dots,x_0$ and then, for some (random) large enough $n,k\in \N$  and given $\varepsilon>0$ one has
 $W_2(\mu^{(n)}_{k}, m_0) \leq W_2(\mu^{(n)}_{k}, \hat{m}_2^{n})  + W_2( \hat{m}_2^{n}, m_0)  \leq  \varepsilon,   $
 where $\hat{m}_2^{n}$ is the $BWB$. 
Notice that, besides the sequential generation of a finite  i.i.d. sequence $\mu^{(n)}_0 ,m^{(n)}_0,\dots,  m^{(n)}_{k}\sim \Pi_n$,  at each step $k+1$ one only needs to compose a new transport  map $\left[(1-\gamma_k)I + \gamma_k T_{\mu^{(n)}_k}^{m^{(n)}_k}\right]$ with the transport map  pushing  forward  $\mu^{(n)}_0 $ to $\mu^{(n)}_k$, cumulatively constructed in the previous iterations.   
 If an expression for each map $T_{\mu^{(n)}_k}^{m^{(n)}_k}$ is available, this can easily be  done (and stored), specifying (only) the values of the lastly computed map on a pre-fixed grid. 
 
 \smallskip

The batch version SGWD can be implemented in a similar way to compute the BWB estimator.

\medskip

\begin{remark}\label{rem_bary_emp} An alternative, natural   approach would be to sample, for given  $n$  a fixed number $k$ of realizations 
$m_i \stackrel{\text{iid}}{\sim} \Pi_n$,  $i=1,\dots,k$, and compute, using the GDW algorithm, the Wasserstein barycenter $\hat{m}_2^{(n,k)}$ of the (finitely  supported) empirical measure 
	\begin{align*}
		 \textstyle  \Pi_n^{(k)}:= \frac{1}{k}\sum_{i=1}^{k}\delta_{m_i}\,\,\in\mathcal{W}_2(\mathcal{W}_{2,ac}(\R^q), 
	\end{align*} 
	or {\it $2$-Wasserstein empirical barycenter} of $\Pi_n$ (see \cite{bigot2018characterization}). 
Indeed, by Varadarajan's theorem, conditionally on $\Pi_n$, a.s. $ \Pi_n^{(k)}$ converges weakly as $k\to \infty$ to $\Pi_n$, and in the $W_2$ metric as soon as $\Pi_n\in  \mathcal{W}_2(\mathcal{W}_{2,ac}(\R^q))$. Since under our assumptions $\Pi_n$ a.s.\ has a unique $2$-Wasserstein  barycenter $\hat{m}_2^{n}$ and, by \cite[Theorem 3]{le2017existence}, $\hat{m}_2^{(n,k)}$ converges with respect to $W_2$ a.s.\ as $k\to \infty$ to it, we also get through this approach that 
 $$W_2(\hat{m}_2^{(n,k)}, m_0) \leq W_2(\hat{m}_2^{(n,k)}, \hat{m}_2^{n})  + W_2( \hat{m}_2^{n}, m_0)  \leq  \varepsilon   $$
 for some (random) large enough $n,k\in \N$. The clear disadvantage of this method is computational:  besides generating $k$ samples of $\Pi_n$, we need to additionally run a possibly large number of GDW steps to approximate $\hat{m}_2^{(n,k)}$, and we need to evaluate at each step $k$ new transport maps (instead of one, for the SGDW). Moreover, if additional $k'$ new  samples from $\Pi_n$ become available, we need to run the whole scheme again to take advantage of this new information.  On the contrary, the online nature of the SGDW method allows one to  refine the already computed estimator by only performing  $k'$ new steps of the algorithm. 
\end{remark}

\section{Numerical experiments}\label{sec numeric}

Before presenting the experimental validation of the proposed methods we give a brief presentation of the scatter-location family of distributions. Our experiments consider this family because the optimal transport maps between two laws in the scatter-location family can be described explicitly. This property facilitates the numerical computation/approximation of barycenters, as the various iterative algorithms described so far take a more amenable form.   See \cite{bakchoff2022} for further examples.

\subsection{Location-Scatter family}\label{sec scat loc}
We follow the setting of \cite{alvarez2018wide}: Given a fixed distribution $\tilde{m} \in \mathcal{W}_{2,ac}(\mathbb{R}^q)$, referred to as \emph{generator}, the associated location-scatter family is given by $$\mathcal{F}(\tilde{m}) : = \{\mathcal{L}(A\tilde{x}+b)\ |\  A \in \mathcal{M}_+^{q\times q} ,b \in \mathbb{R}^q, \tilde{x} \sim \tilde{m} \},$$
where $\mathcal{M}_+^{q\times q}$ is the set of symmetric positive definite matrices of size $q\times q$. Without loss of generality we can assume that $\tilde{m}$ has zero mean and identity covariance. Note that  $\mathcal{F}(\tilde{m})$ is the multivariate normal family if $\tilde{m}$ is the standard multivariate normal distribution.

The optimal map between two members $m_1 = \mathcal{L}(A_1\tilde{x}+b_1)$ and $m_2 = \mathcal{L}(A_2\tilde{x}+b_2)$ of $\mathcal{F}(\tilde{m})$ is explicit, given by  $T_{m_1}^{m_2}(x) = A(x-b_1) + b_2$ where $A = A_1^{-1}(A_1A_2^2A_1)^{1/2}A_1^{-1} \in \mathcal{M}_+^{q\times q}$. This family of optimal maps contains the identity and is closed under convex combination. 

If $\Gamma$ is supported on $\mathcal{F}(\tilde{m})$, then its $2$-Wasserstein barycenter $\hat{m}$ belongs to $\mathcal{F}(\tilde{m})$. In fact, call its mean $\hat{b}$ and its covariance matrix $\hat{\Sigma}$. Since the optimal map from $\hat{m}$ to $m$ is $T_{\hat{m}}^{m}(x) = A_{\hat{m}}^{m}(x-\hat{b}) + b_m$, where $A_{\hat{m}}^{m} = \hat{\Sigma}^{-1/2}(\hat{\Sigma}^{1/2}\Sigma_m\hat{\Sigma}^{1/2})^{1/2}\hat{\Sigma}^{-1/2}$, and we know that $\int T_{\hat{m}}^{m}(x)\Gamma(dm) = x$, $\hat{m}$-almost surely, then we must have that $\int A_{\hat{m}}^{m} \Gamma(dm) = I$, since clearly $\hat{b}=\int b_m \Gamma(dm)$, and as a consequence $\hat{\Sigma} = \int (\hat{\Sigma}^{1/2}\Sigma_m\hat{\Sigma}^{1/2})^{1/2}\Gamma(dm)$.

A stochastic gradient descent iteration, starting from a distribution $\mu = \mathcal{L}(A_0\tilde{x}+b_0)$, sampling some $m=\mathcal{L}(A_m\tilde{x}+b_m) \sim \Gamma$, and with step $\gamma$, produces the measure $\nu=T_0^{\gamma,m}(\mu):= ((1-\gamma)I + \gamma T_\mu^m)(\mu)$. If $\tilde{x}$ has a multivariate distribution $\tilde{F}(x)$, then $\mu$ has distribution $F_0(x)=\tilde{F}(A_0^{-1}(x-b_0))$ with mean $b_0$ and covariance $\Sigma_0 = A_0^2$. We have that $T_0^{\gamma,m}(x) =  ((1-\gamma)I + \gamma A_{\mu}^{m})(x-b_0) + \gamma b_m + (1-\gamma)b_0$ with $A_{\mu}^{m} := A_0^{-1}(A_0A_m^2A_0)^{1/2}A_0^{-1}$. Then $\nu$ has distribution $$F_1(x) = F_0([T_0^{\gamma.m}]^{-1}(x)) =\tilde{F}( [(1-\gamma)A_0+\gamma A_{\mu}^{m}A_0]^{-1}(x-\gamma b_m - (1-\gamma)b_0 ) ),$$ with mean $b_1 = (1-\gamma)b_0 + \gamma b_m$ and covariance
\begin{align*}
	\Sigma_1 = A_1^2&= [(1-\gamma)A_0+\gamma A_0^{-1} (A_0A_m^2A_0)^{1/2}][(1-\gamma)A_0+\gamma (A_0A_m^2A_0)^{1/2}A_0^{-1}]\\
	 &=  A_0^{-1}[(1-\gamma)A_0^2+\gamma  (A_0A_m^2A_0)^{1/2}][(1-\gamma)A_0^2+\gamma (A_0A_m^2A_0)^{1/2}]A_0^{-1}\\
	 &=A_0^{-1}[(1-\gamma)A_0^2+\gamma (A_0A_m^2A_0)^{1/2}]^2A_0^{-1}. 
\end{align*}
The batch stochastic gradient descent iteration is characterized by
\begin{align*}
	b_{1} &= \textstyle (1-\gamma)b_0+\frac{\gamma}{S}\sum_{i=1}^{S}b_{m^{i}}\\
	A_{1}^2 &=\textstyle A_0^{-1}[(1-\gamma)A_0^2+\frac{\gamma}{S}\sum_{i=1}^{S} (A_0A_{m^{i}}^2A_0)^{1/2}]^2A_0^{-1}.
\end{align*}

\subsection{Experiment}\label{sec numeric_experiment}

We considered a model within a location-scatter family (LS), with generator  $\tilde{m}$  on $\mathbb{R}^{15}$ with independent coordinates, as follows:
\begin{itemize}
	\item coordinates 1 to 5 are standard Normal distributions
	\item coordinates 6 to 10 are standard Laplace distributions, and 
	\item coordinates 11 to 15 are standard Student's $t$-distributions ($3$ degrees of freedom).
\end{itemize}
Fig.~\ref{fig:generator} shows samples (uni- and bi-variate marginals) from coordinates $\{1,2,6,7,11,12\}$ of $\tilde{m}$.

\begin{figure}[ht]
	\includegraphics[width=0.7\textwidth]{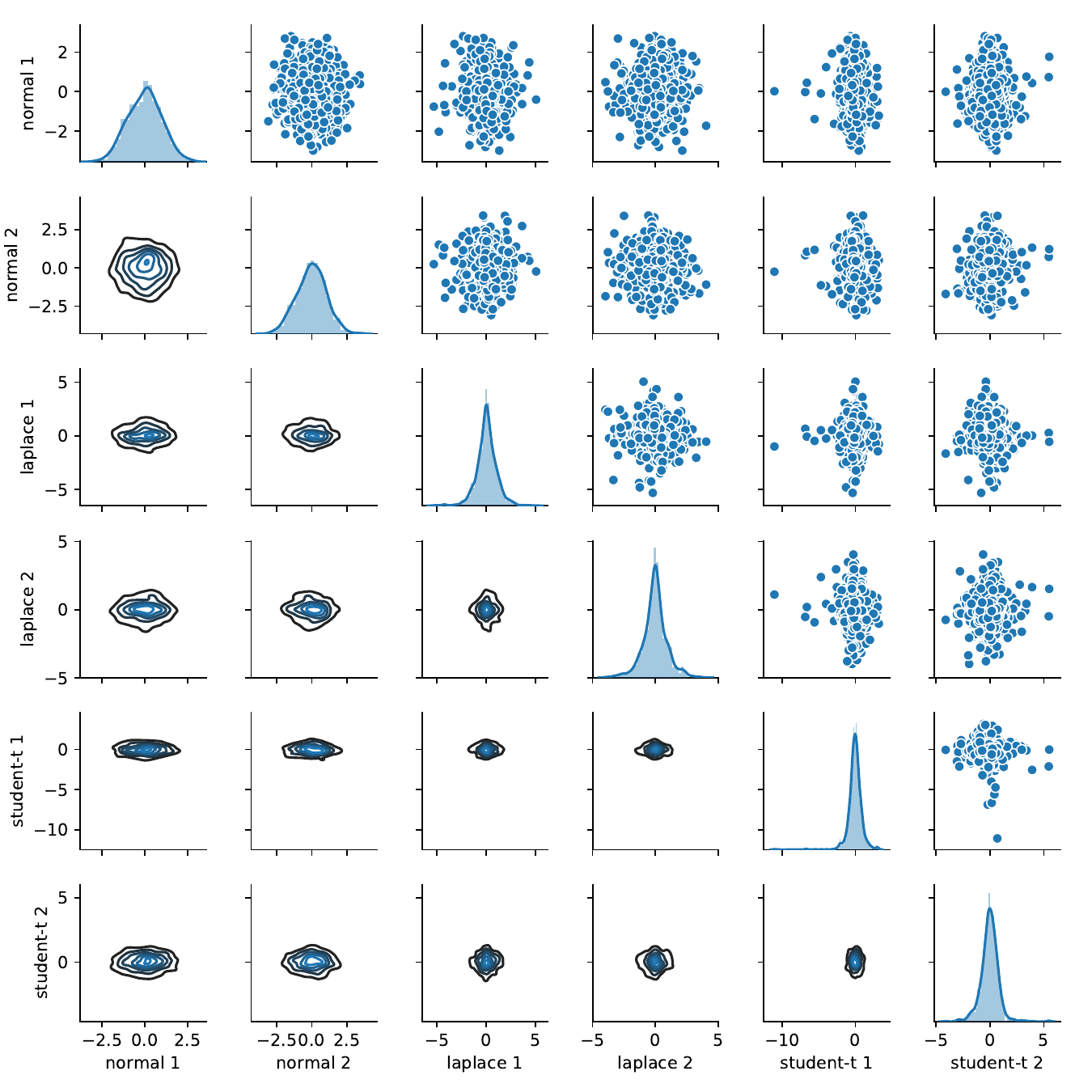}
	\caption{Samples from the univariate (diagonal) and bivariate (off-diagonal) marginals for 6 coordinates from the generator distribution $\tilde{m}$. The diagonal and lower triangular plots are smoothed histograms, whereas the upper-diagonal ones are collections of samples. }
	\label{fig:generator}
\end{figure}
In this LS family,  we chose the true model $m_0$  with  location vector $b \in \mathbb{R}^{15}$ defined as $b_i = i-1$ for $i=1,\ldots,15$, and scatter matrix $A = \Sigma^{1/2}$ with $\Sigma_{i,j} = K\left( \left(\frac{i-1}{14}\right)^{1.1} ,\left(\frac{j-1}{14}\right)^{1.1}\right)$ for $i,j=1,\ldots,15$ \footnote{We chose $\left(\frac{j-1}{14}\right)^{1.1}$ for $j=1,\ldots,15$ because this defines a non-uniform grid over $[0,1]$.},  with the covariance (kernel) function $K(i,j) = \eps \delta_{ij} + \sigma \cos\left(\omega(i-j)\right)$. Given the parameters $\eps, \sigma$ and $\omega$, the so constructed covariance matrix will be denoted $\Sigma_{\eps,\sigma,\omega}$. We chose the parameters $\eps = 0.01$, $\sigma = 1$ and $\omega = 5.652 \approx 1.8\pi$ for $m_0$. Thus,  under the true model $m_0$ the coordinates can be negatively/positively correlated and there is also a coordinate-independent noise component due to the Kronecker delta $\delta_{ij}$. Fig.~\ref{fig:true_distribution} shows the covariance matrix and three coordinates of the  \emph{true} model $m_0$.

\begin{figure}[ht]
	\includegraphics[width=0.52\textwidth]{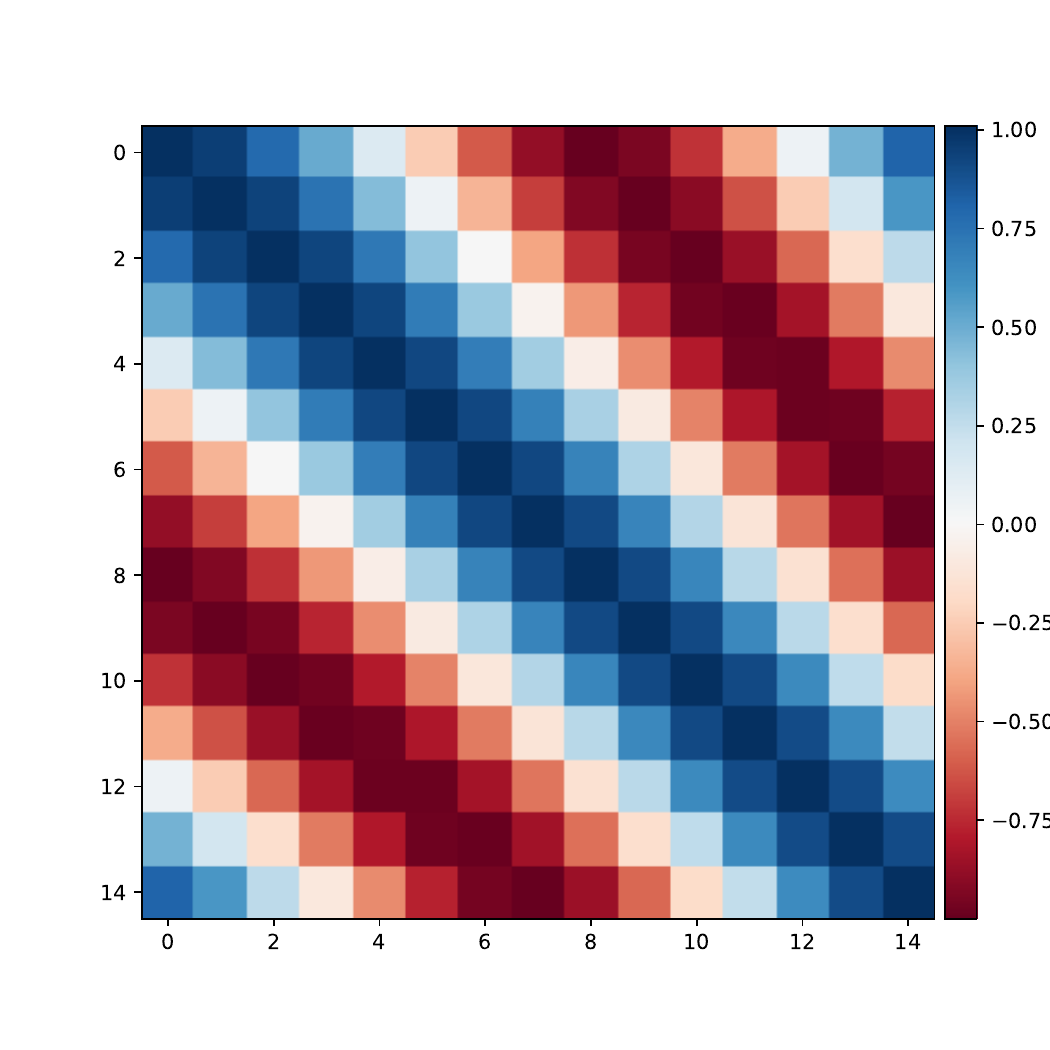}\hfill
	\includegraphics[width=0.47\textwidth]{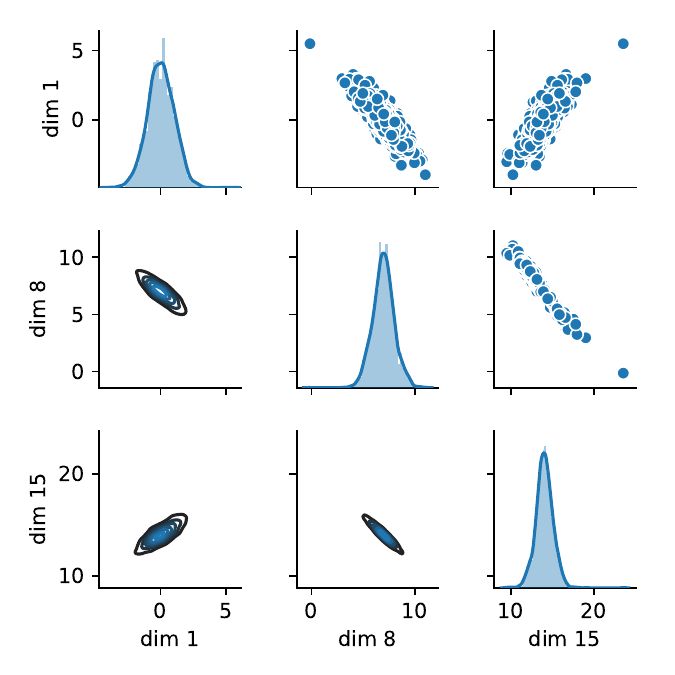}
	\caption{True model $m_0$: covariance matrix (left), and univariate and bivariate marginals for dimensions 1, 8 and 15 (right). Notice that some coordinates are positively or negatively correlated, and some may even be close to uncorrelated.}
	\label{fig:true_distribution}
\end{figure}
	
The model prior $\Pi$ is the push-forward induced by a prior over the mean vector $b$ and the parameters of the covariance $\Sigma_{\eps,\sigma,\omega}$,  chosen   independently according to : 
\begin{equation}
p(b,\Sigma_{\eps,\sigma,\omega}) = \mathcal{N}(b|0,\text{I})\,\text{Exp}(\eps|20)\,\text{Exp}(\sigma|1)\,\text{Exp}(\omega^{-1}|15),	
\end{equation}
where $\text{Exp}(\cdot|\lambda)$ is a exponential density with rate $\lambda$. Given $n$ samples from the true model $m_0$ (also referred to as \emph{observations} or \emph{data points}), $k$ samples are produced from the posterior measure $\Pi_n$ using Markov chain Monte Carlo (MCMC).

The numerical analysis presented in what follows focuses on the behavior of the BWB as a function of both the number of samples $k$ and the number of data points $n$.

\subsubsection*{Numerical consistency of the empirical posterior }

We first validate the empirical measure $\Pi_n^{(k)}$ as a consistent sample version of the true posterior under the $W_2$ distance, that is, we check that $W_2(\Pi_n^{(k)}, \delta_{m_0}) \rightarrow W_2(\Pi_n, \delta_{m_0})$ for large $k$. We estimate $W_2(\Pi_n^{(k)}, \delta_{m_0})$ 10 times for each combination of the (number of) observations $n$ and samples $k$ in the following sets 
\begin{itemize}
	\item $k \in \{1, 5, 10, 20, 50, 100, 200, 500, 1000\}$
	\item $n \in \{10, 20, 50, 100, 200, 500, 1000, 2000, 5000, 10000\}$
\end{itemize} 
Fig.~\ref{fig:convergence_w2} shows the 10 estimates of $W_2(\Pi_n^{(k)}, \delta_{m_0})$ for different values of $k$ (in the $x$-axis) and of $n$ (color coded). Notice how the estimates become more concentrated for larger $k$ and that the Wasserstein distance between the empirical measure $\Pi_n^{(k)}$ and the true model $m_0$ decreases for larger $n$. Additionally, Table \ref{tab:std_w2} shows that the standard deviation of the 10 estimates of $W_2(\Pi_n^{(k)}, \delta_{m_0})$ decreases as either $n$ or $k$ increases.

\begin{figure}[ht]
	\includegraphics[width=0.7\textwidth]{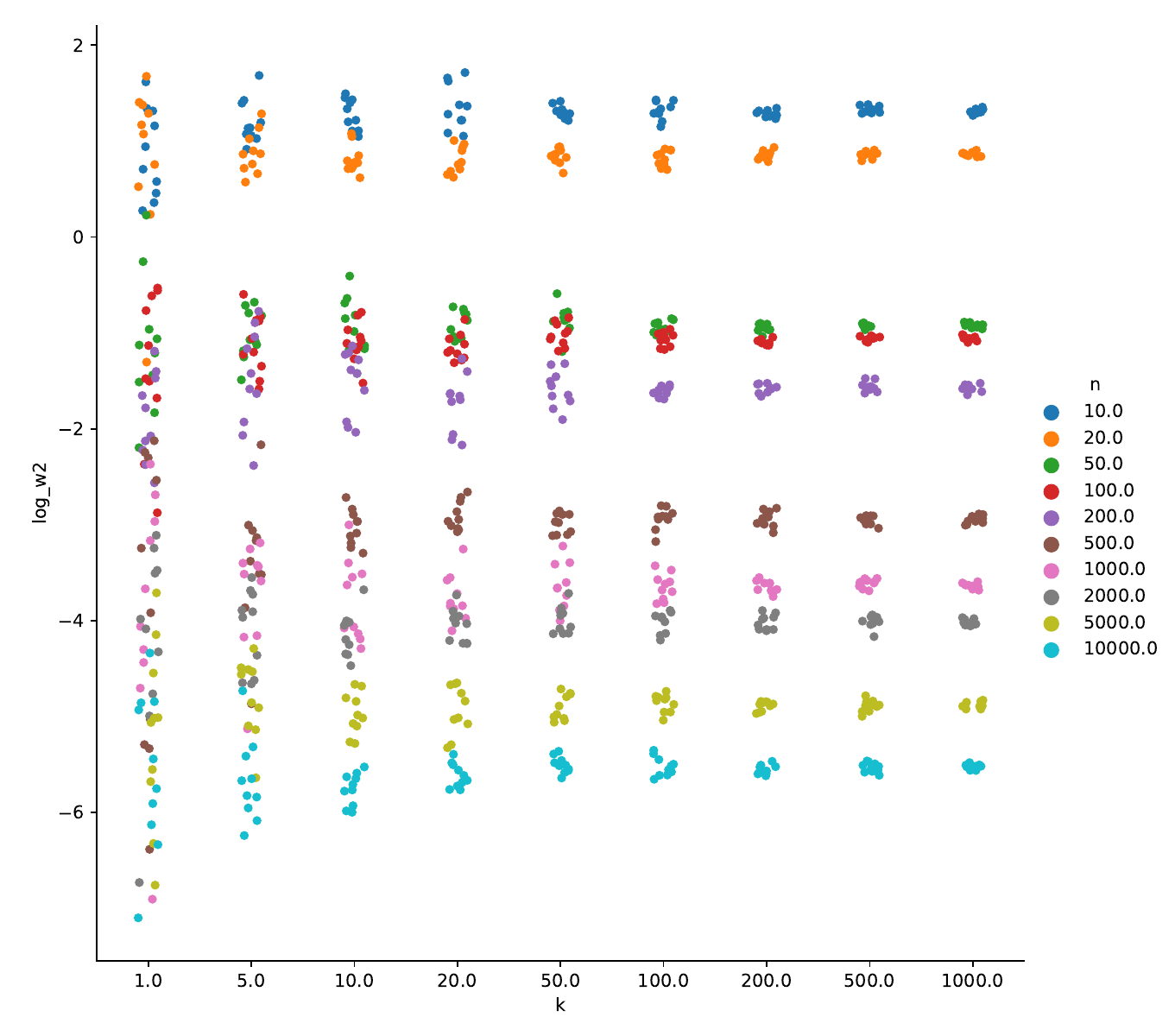}
	\caption{Wasserstein distance between the empirical measure $\Pi_n^{(k)}$ and  $\delta_{m_0}$ in logarithmic scale for different number of observations $n$ (color coded) and samples $k$ ($x$-axis). For each pair ($n,k$), 10 estimates of $W_2(\Pi_n^{(k)}, \delta_{m_0})$ are shown.}
	\label{fig:convergence_w2}
\end{figure}

\begin{table}[ht]
\small
	{\small
		\caption{Standard deviation of $W_2^2(\Pi_n^{(k)}, \delta_{m_0})$, using 10 simulations, for different values of observations $n$ and samples $k$.}
		\label{tab:std_w2} 
	}
	\centering
\begin{tabular}{l|ccccccccc}
	\hline
	n \textbackslash~ k &  1    &  5    &  10   &  20   &  50   &  100  &  200  &  500  &  1000 \\
	\hline
	10     &  1.2506 &  0.8681 &  0.5880 &  0.9690 &  0.2354 &  0.3440 &  0.1253 &  0.1330 &  0.0972 \\
	20     &  1.5168 &  0.5691 &  0.3524 &  0.3182 &  0.1850 &  0.1841 &  0.1049 &  0.0811 &  0.0509 \\
	50     &  0.3479 &  0.0948 &  0.1275 &  0.0572 &  0.0623 &  0.0229 &  0.0157 &  0.0085 &  0.0092 \\
	100    &  0.2003 &  0.1092 &  0.0712 &  0.0469 &  0.0431 &  0.0254 &  0.0087 &  0.0079 &  0.0084 \\
	200    &  0.0749 &  0.1249 &  0.0717 &  0.0533 &  0.0393 &  0.0101 &  0.0092 &  0.0109 &  0.0072 \\
	500    &  0.0478 &  0.0285 &  0.0093 &  0.0086 &  0.0053 &  0.0056 &  0.0045 &  0.0023 &  0.0022 \\
	1000   &  0.0299 &  0.0113 &  0.0113 &  0.0064 &  0.0067 &  0.0036 &  0.0016 &  0.0012 &  0.0007 \\
	2000   &  0.0145 &  0.0071 &  0.0040 &  0.0031 &  0.0027 &  0.0019 &  0.0014 &  0.0011 &  0.0006 \\
	5000   &  0.0072 &  0.0031 &  0.0015 &  0.0018 &  0.0010 &  0.0007 &  0.0004 &  0.0005 &  0.0002 \\
	10000  &  0.0038 &  0.0020 &  0.0005 &  0.0005 &  0.0004 &  0.0004 &  0.0002 &  0.0002 &  0.0001 \\
	\hline
\end{tabular}
\end{table}

\subsubsection*{Wasserstein distance between the empirical barycenter and the true model}
For each empirical posterior $\Pi_n^{(k)}$,  we computed the empirical Wasserstein barycenter $\hat{m}_2^{(n,k)}$ as suggested in Remark \ref{rem_bary_emp}. Thus, we used the iterative GDW  procedure in eq.~\eqref{eq:fixed-point}, namely the (deterministic) gradient descent method, and repeated this calculation 10 times. As a stopping criterion for gradient descent, we considered the relative variation of the $W_2$ cost; the computation was terminated when this cost fell below $10^{-4}$. Fig.\ \ref{fig:convergence_barycenter} shows the $W_2$ distances between the so computed barycenters and the true model, while Table \ref{tab:w2_baricenter_m0} shows the average distance for each pair ($n,k$). Notice that, in general, both the average and standard deviation of the barycenters decrease as either $n$ or $k$ increases, yet for large values (e.g.,\ $n=2000,5000$) numerical issues appear.

\begin{figure}[ht]
	\includegraphics[width=0.7\textwidth]{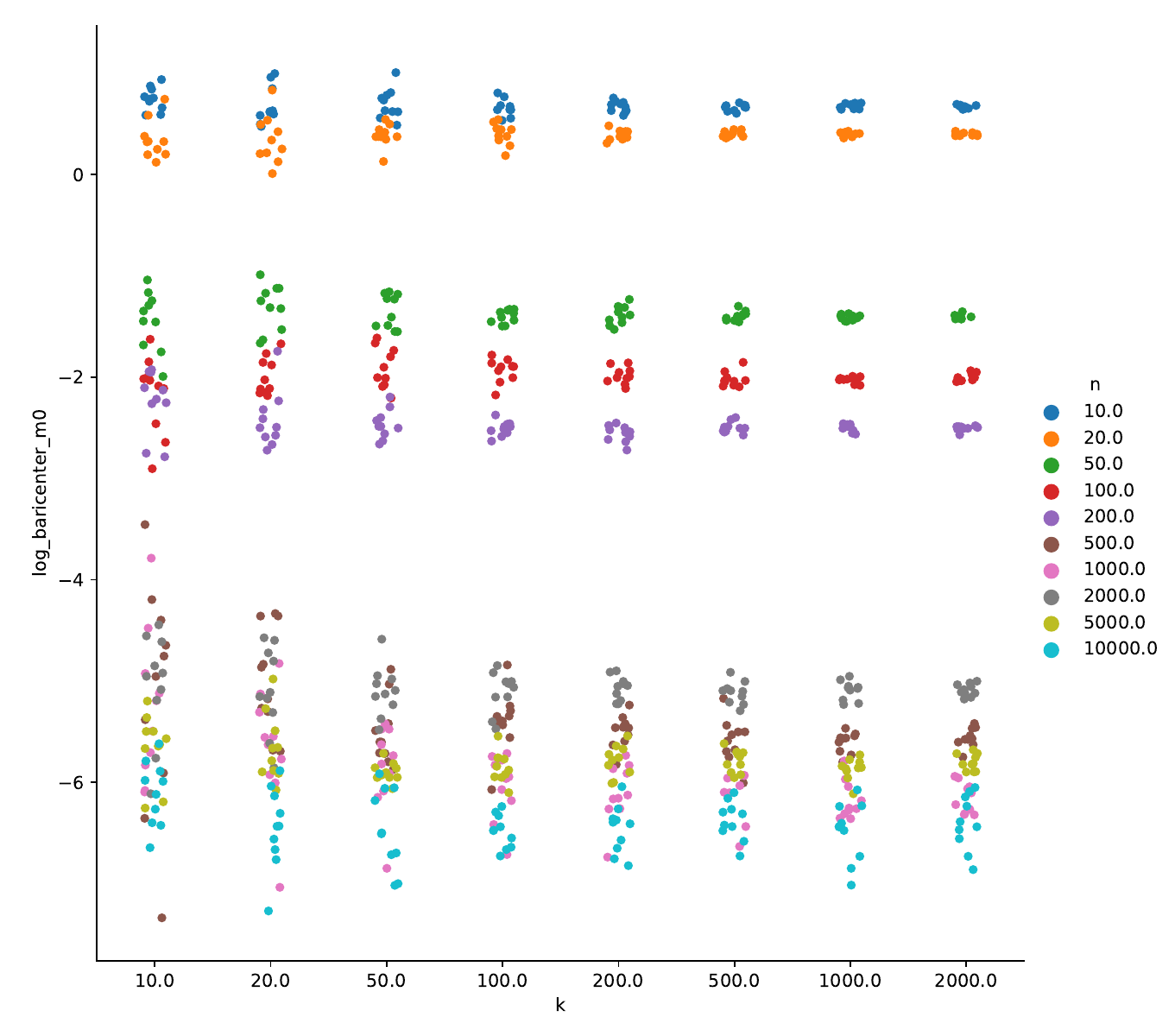}
	\caption{$W_2$ distance between the empirical barycenters $\hat{m}_2^{(n,k)}$ and the true model ${m_0}$ in logarithmic scale for different number of observations $n$ (color coded) and samples $k$ ($x$-axis). For each pair ($n,k$), 10 estimates of $W_2(\hat{m}_2^{(n,k)}, m_0)$ are shown.}
	\label{fig:convergence_barycenter}
\end{figure}

\begin{table}[ht]
\small
	{\small
		\caption{Sample average  of $W_2^2(\hat{m}_2^{(n,k)}, m_0)$, using 10 simulations, for different values of observations $n$ and samples $k$.}
		\label{tab:w2_baricenter_m0} 
	}
	\centering
	\begin{tabular}{l|cccccccc}
		\hline
		n / k &   10   &  20   &  50   &  100  &  200  &  500  &  1000 & 2000 \\
		\hline
		10 &  2.1294 &  2.0139 &  2.0384 &  1.9396 &  1.9608 &  1.9411 &  1.9699 &  1.9548 \\
		20&  1.4382 &  1.4498 &  1.4826 &  1.4973 &  1.4785 &  1.4953 &  1.4955 &  1.4914 \\
		50 &  0.2455 &  0.2759 &  0.2639 &  0.2468 &  0.2499 &  0.2483 &  0.2443 &  0.2454 \\
		100 &  0.1211 &  0.1387 &  0.1509 &  0.1458 &  0.1379 &  0.1328 &  0.1318 &  0.1349 \\
		200 &  0.1116 &  0.0922 &  0.0859 &  0.0817 &  0.0777 &  0.0824 &  0.0820 &  0.0819 \\
		500 &  0.0094 &  0.0077 &  0.0043 &  0.0047 &  0.0041 &  0.0038 &  0.0037 &  0.0039 \\
		1000 &  0.0068 &  0.0039 &  0.0031 &  0.0025 &  0.0023 &  0.0022 &  0.0021 &  0.0021 \\
		2000 &  0.0072 &  0.0066 &  0.0063 &  0.0062 &  0.0063 &  0.0060 &  0.0062 &  0.0062 \\
		5000 &  0.0037 &  0.0037 &  0.0028 &  0.0029 &  0.0031 &  0.0031 &  0.0028 &  0.0030 \\
		10000 &  0.0023 &  0.0017 &  0.0017 &  0.0015 &  0.0016 &  0.0017 &  0.0016 &  0.0017 \\
		\hline
	\end{tabular}
\end{table}

\subsubsection*{Distance between the empirical barycenter and the Bayesian model average}
We then compared the  empirical Wasserstein barycenters $\hat{m}_2^{(n,k)}$ to the standard  Bayesian model averages, denoted here $\bar{m}^{(n,k)}$, in terms of their distance to the true model $m_0$, for $n=1000$ observations. To that end, we estimated the $W_2$ distances via empirical approximations with $1000$ samples for each model based on \cite{flamary2017pot}. We simulated this procedure 10 times for $k\in\{10,20,50,100,200,500,1000\}$. Fig.~\ref{fig:barycenter_vs_model_average} shows the sample average and variance of the $W_2$ distances of the Wasserstein barycenters and Bayesian model averages. The empirical barycenter is seen to be closer to the true model than the model average,  regardless of the number of MCMC samples $k$.

\begin{figure}[ht]
	\includegraphics[width=0.9\textwidth]{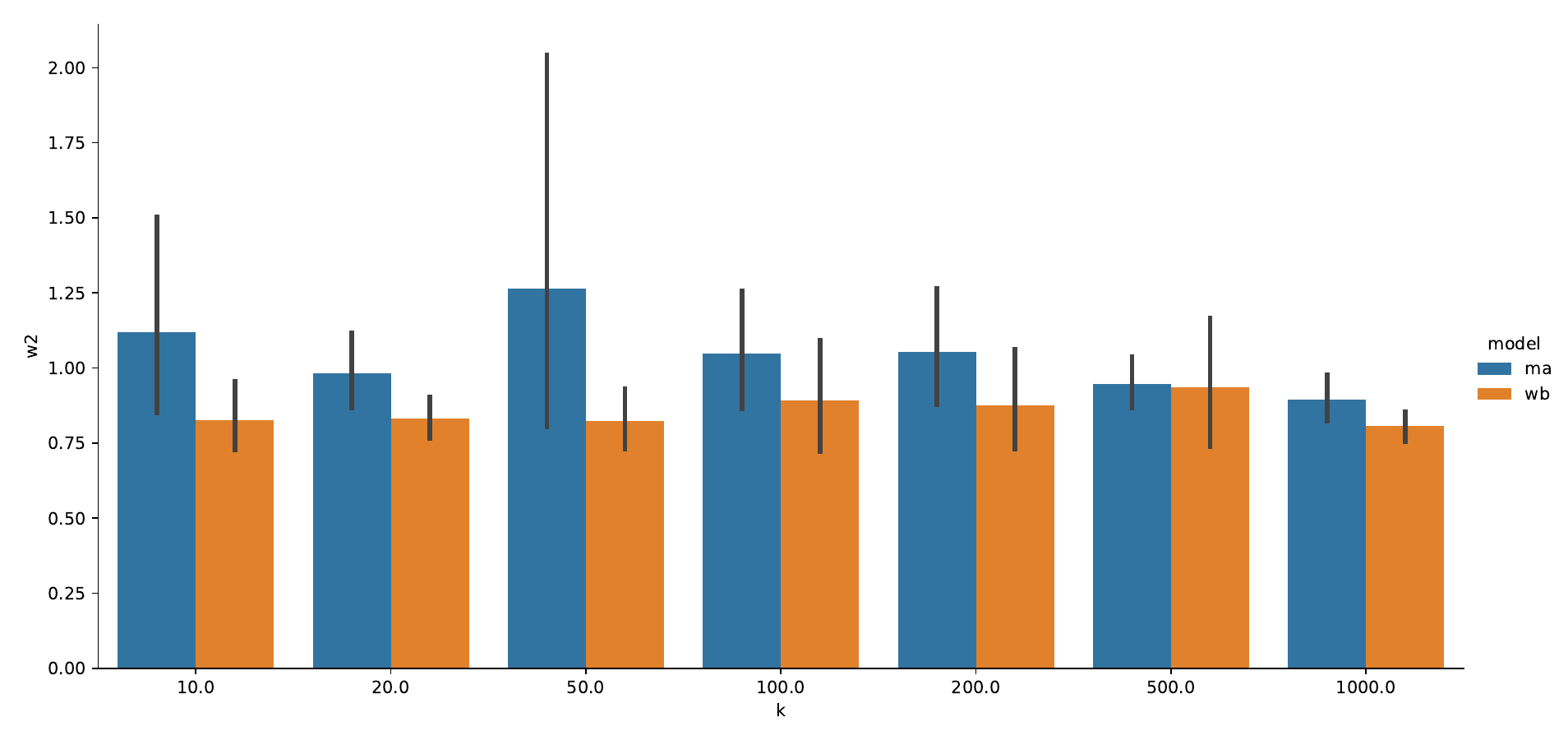}
	\caption{Averages (bars) and standard deviations (vertical lines) of $W_2^2(\hat{m}_n^{(k)}, m_0)$ denoted as \textbf{wb} in orange, and $W_2^2(\bar{m}_n^{(k)}, m_0)$ denoted as \textbf{ma} in blue, for $n=1000$ and different numbers of samples $k$. We considered 10 simulations for each $k$.}
	\label{fig:barycenter_vs_model_average}
\end{figure}

\iffalse
...and Table \ref{tab:barycenter_vs_model_average}....
\begin{table}[ht]
\small
	{\small
		\caption{Average and standard deviation of $W_2^2(\bar{m}_n^{(k)}, m_0)$ and $W_2^2(\hat{m}_n^{(k)}, m_0)$  for $n=1000$, using 10 simulations.}
		\label{tab:barycenter_vs_model_average} 
	}
	\centering
	\begin{tabular}{l|cccccc}
		\hline
		{} & \multicolumn{2}{c}{$W_2^2(\bar{m}_n^{(k)}, m_0)$} & \multicolumn{2}{c}{$W_2^2(\hat{m}_n^{(k)}, m_0)$} \\
		\hline
		{k} &   mean &    std &   mean &    std \\
		\hline
		10   & 1.1197 & 0.5578 & 0.8259 & 0.1921 \\
		20   & 0.9836 & 0.2174 & 0.8320 & 0.1202 \\
		50   & 1.2645 & 1.1879 & 0.8249 & 0.1626 \\
		100  & 1.0487 & 0.3480 & 0.8931 & 0.3073 \\
		200  & 1.0551 & 0.3347 & 0.8749 & 0.3039 \\
		500  & 0.9469 & 0.1464 & 0.9367 & 0.3858 \\
		1000 & 0.8946 & 0.1331 & 0.8072 & 0.0881 \\
		\hline
	\end{tabular}
\end{table}

\fi

\subsubsection*{Computation of the Wasserstein barycenter with batch-SGDW }

Lastly, we compared the empirical barycenters $\hat{m}_2^{(m,k)}$ to the barycenter obtained by the batch SGDW method in eq.~\eqref{eq:sgd-batch-seq} with different batch sizes. Here, we shall denote the latter by  $\hat{m}_2^{(n,t,s)}$ with $n$ the number of observations, $t$ the steps of the algorithm, and $s$ the batch size.  Fig.~\ref{fig:cost_stochastic_barycenters} shows the evolution of the $W_2^2$ distance between the stochastic gradient descent sequences and the true model $m_0$ for $n\in\{10,20,50,100,200,500,1000\}$ observations and batches of sizes $s\in\{1,15\}$, with step-size $\gamma_t = \frac{1}{t}$ for $t=1,\ldots,200$. Notice from Fig.~\ref{fig:cost_stochastic_barycenters} that the larger the batch, the more concentrated the trajectories of $\hat{m}_{n, s}$ become.  Additionally, Table \ref{tab:mean_stochastic_barycenters} summarizes the means of the distance $W_2^2$ to the true model $m_0$, using the sequences after $t=100$ against the empirical estimator using all the simulations with $k \geq 100$. Finally, Table \ref{tab:std_stochastic_barycenters} shows the standard deviation of the distance $W_2^2$ to the true model $m_0$, which  can be seen  to decrease as the batch size grows. Critically, we observe that for batch sizes $s \geq 5$ the stochastic estimation was \emph{better} than its empirical counterpart, i.e., it had smaller variance with similar (or even smaller) bias. This is noteworthy given the fact that computing our Wasserstein barycenter estimator via the batch stochastic gradient descent method is computationally less demanding than computing it via the empirical method.

\begin{figure}[H]
	\includegraphics[width=0.95\textwidth]{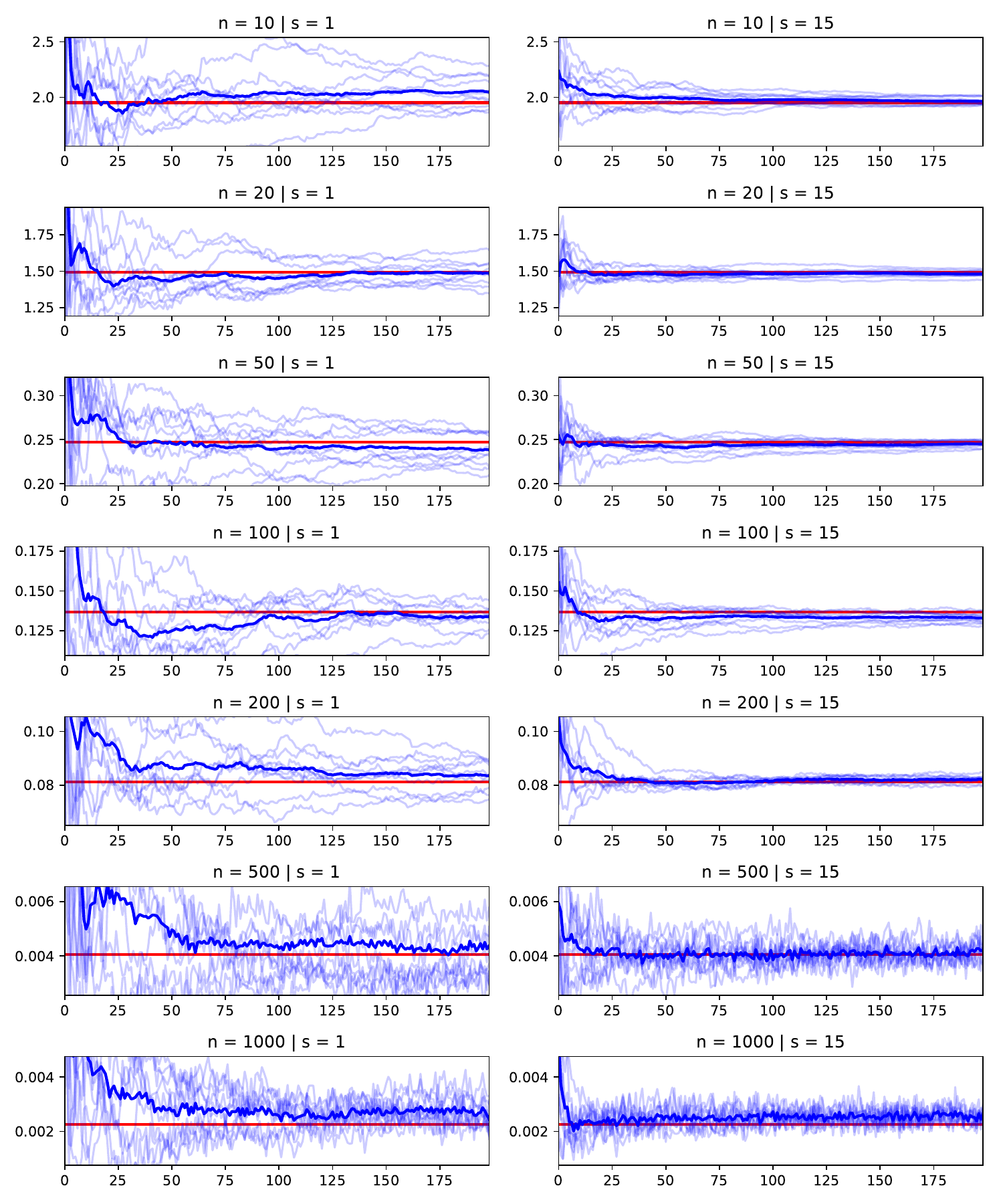}
	\caption{Evolution of the $W_2^2$ cost for 10 realizations of the SGDW sequence computing the BWB and their mean (blue),  versus an empirical barycenter estimator (red), for $n=10,20,50,100,200,500,1000$ and batches sizes $s= 1,15$.}
	\label{fig:cost_stochastic_barycenters}
\end{figure}

\begin{table}[ht]
\small
	{\small
		\caption{Means of $W_2^2$ of the stochastic gradient estimations (using the sequences with $t \geq 100$) and that of the empirical estimator (using the simulations with $k \geq 100$), across different combinations of observations $n$ and batch size $s$. 
		%for $n=10,20,50,100,200,500,1000$ 
		}
		\label{tab:mean_stochastic_barycenters} 
	}
	\centering
	\begin{tabular}{l|cccccc|c}
		\hline
		n / s &    1 &    2 &    5 &   10 &   15 &   20 &  empirical \\
		\hline
		10 & 2.0421 & 2.0091 & 1.9549 & 1.9721 & 1.9732 & 1.9712 &     1.9532 \\
		20 & 1.4819 & 1.4868 & 1.5100 & 1.4852 & 1.4840 & 1.4891 &     1.4916 \\
		50 & 0.2406 & 0.2512 & 0.2465 & 0.2427 & 0.2444 & 0.2460 &     0.2469 \\
		100 & 0.1340 & 0.1392 & 0.1340 & 0.1349 & 0.1334 & 0.1338 &     0.1366 \\
		200 & 0.0843 & 0.0811 & 0.0819 & 0.0807 & 0.0820 & 0.0819 &     0.0811 \\
		500 & 0.0044 & 0.0042 & 0.0039 & 0.0039 & 0.0041 & 0.0040 &     0.0041 \\
		\hline
	\end{tabular}
\end{table}

\begin{table}[ht]
\small
	{\small
		\caption{Std.\ deviations of $W_2^2$ of the stochastic gradient estimations (using the sequences with $t \geq 100$) and that of empirical estimator (using the simulations with $k \geq 100$), across different combinations of observations $n$ and batch size $s$. 
			%for $n=10,20,50,100,200
		 }
		\label{tab:std_stochastic_barycenters} 
	}
	\centering
	\begin{tabular}{l|cccccc|c}
		\hline
		n / s &    1 &    2 &    5 &   10 &   15 &   20 &  empirical \\
		\hline
		10  & 0.1836 & 0.1071 & 0.0526 & 0.0474 & 0.0397 & 0.0232 &     0.0916 \\
		20  & 0.0751 & 0.0565 & 0.0553 & 0.0189 & 0.0253 & 0.0186 &     0.0790 \\
		50  & 0.0210 & 0.0174 & 0.0072 & 0.0084 & 0.0050 & 0.0039 &     0.0138 \\
		100 & 0.0102 & 0.0076 & 0.0049 & 0.0048 & 0.0035 & 0.0023 &     0.0112 \\
		200 & 0.0074 & 0.0045 & 0.0021 & 0.0035 & 0.0013 & 0.0017 &     0.0047 \\
		500 & 0.0016 & 0.0007 & 0.0005 & 0.0004 & 0.0004 & 0.0004 &     0.0009 \\
		1000& 0.0005 & 0.0006 & 0.0004 & 0.0004 & 0.0003 & 0.0003 &     0.0005 \\
		\hline
	\end{tabular}
\end{table}

\subsubsection*{Conclusion} Based on this illustrative numerical example, we can conclude:
\begin{itemize}
\item the empirical posterior constructed using MCMC is consistent under the $W_2$ distance and therefore it can be relied upon to compute Wasserstein barycenters,
\item the empirical Wasserstein barycenter estimator tends to converge faster (and with lower variance) to the true model than the empirical Bayesian model average,
\item computing the population Wasserstein barycenter estimator via batch stochastic gradient descent is promising as an alternative to computing the empirical barycenter (i.e.,\ to applying the deterministic gradient descent method to a finitely-sampled posterior).
\end{itemize}

\section*{Acknowledgments} We would like to thank two anonymous referees for suggesting us some of the parametric examples we have analyzed, and for their valuable questions and comments,  which allowed us to clarify the presentation of the paper and refine our results. This work was partially funded by the ANID-Chile grants: Fondecyt-Regular 1201948(JF) and 1210606 (FT); Center for Mathematical Modeling ACE210010 and FB210005 (JF, FT); and Advanced Center for Electrical and Electronic Engineering FB0008 (FT).

\bibliography{BLWB}{}
\bibliographystyle{plain}

\appendix

\section{Bayes estimators as generalized model averages}
\label{subsec model_averages}

We prove here Proposition \ref{prop EL}. For notational simplicity we omit the subscripts of estimators and normalizing constants.\\

\noindent \textbf{i) Squared $L_{2}$-distance:}
%\begin{equation*}\textstyle 
%	L_{2}(m,\bar{m})=\frac{1}{2}\int_{\mathcal{X}}\%left( m(x)-\bar{m}(x)\right) ^{2}\lambda (dx).
%\end{equation*}
By Fubini's theorem, minimizing $R_{L}(\bar{m}|D)$ in this case amounts to minimize 
$$\bar m\mapsto \frac{1}{2}\int_{\mathcal{X}}\left\{\int_{\mathcal{M}}\left( m(x)-\bar{m}
	(x)\right) ^{2}\Pi_n (dm)\right\}\lambda (dx), $$
	over the set of densities. But the optimal value cannot be better than if we minimize pointwise the term $\bar m(x)$ inside the curly brackets, obtaining the candidate
	$$\hat m_{BMA}(x)=\int_{\mathcal{M}} m(x)\Pi_n (dm)=\mathbb{E}[m](x).$$
	As this pointwise minimizer is already a probability density, we conclude.\\

%By the fundamental lemma of calculus of variations, denoting
%\begin{align*}\textstyle 
%	\mathcal{L}\left( x,\bar{m},\bar{m}^{\prime }\right) =\frac{1}{2}\int_{\mathcal{M}}\left( m(x)-\bar{m}(x)\right)^{2}\Pi (dm|D),
%\end{align*}
%the extremes of $R_{L}(\bar{m}|D)$ are weak solutions of the Euler-Lagrange equation 
%\begin{eqnarray*}
%	\textstyle \frac{\partial \mathcal{L}\left( x,\bar{m},\bar{m}^{\prime }\right) }{
%		\partial \bar{m}} &=&\textstyle \frac{d}{dx}\frac{\partial \mathcal{L} \left( x,\bar{m},
%		 \bar{m}^{\prime }\right) }{\partial \bar{m}^{\prime }} \\
%	\textstyle \int_{\mathcal{M}}\left( m(x)-\bar{m}(x)\right) \Pi (dm|D) &=& \textstyle 0. 
%\end{eqnarray*}
%We deduce that the optimum is reached at the Bayesian model average
%\begin{align*}\textstyle 
%	\int_{\mathcal{M}}m(x)\Pi (dm|D).
%\end{align*}

\noindent \textbf{ii) Squared Hellinger distance:} Writing
\begin{equation*}\textstyle 
	H^{2}(m,\bar{m})=\frac{1}{2}\int_{\mathcal{X}}\left( \sqrt{m(x)}-\sqrt{\bar{m
		}(x)}\right) ^{2}\lambda (dx)=1-\int_{\mathcal{X}}\sqrt{m(x)\bar{m}(x)}
	\lambda (dx),
\end{equation*}
we see that optimizing the asociated Bayes risk amounts to  maximizing over $\bar{m}\in L^1({\cal X},  \lambda)$ the concave functional $\bar{m}\mapsto \int_{\mathcal{X}} \sqrt{\bar{m}(x)} f(x) \lambda (dx)$ with $\sqrt{y}=-\infty $ for $y<0$ and $f(x)= \int_{\mathcal{M}}\sqrt{m(x)} \Pi_n (dm) $, under the constraint $\int_{\cal X} \bar{m}(x)\lambda(dx)=1$.  Notice by the Cauchy-Schwarz inequality that this functional is finite if (and  only if)  $\bar{m}\geq 0$, $\lambda$-a.e. Hence, introducing $\gamma$  a real Lagrange multiplier   for the  constraint, we need to find a critical point $\bar{m}\geq 0$  of the concave functional  
$$\bar{m}\mapsto \int_{\mathcal{X}} \sqrt{\bar{m}(x)} f(x) \lambda (dx) + \gamma\left( \int_{\cal X} \bar{m}(x)\lambda(dx)-1 \right), $$
which must thus be a maximum. Again, as in i), we cannot do better in this case than if we maximize  for each $x\in {\cal X}$ the  concave functional  $y\mapsto \sqrt{y} f(x)+ \gamma y$ over $y\geq 0$. Finding  for each $x$ the  critical point $y$  in terms of $\gamma$ and  integrating then w.r.t. $\lambda(dx)$ to get rid of $\gamma$,  we find the extremum of $R_{H^{2}}(\bar{m}|D)$ is attained when $\bar{m}$ equals the 
Bayesian \emph{square} model average:
\begin{equation*}
\begin{split}
	\hat{m}(x) &=\textstyle \frac{1}{Z}\left( \int_{\mathcal{M}}\sqrt{m(x)}\Pi_n
	(dm)\right) ^{2} \mbox{ with}\\ 
	Z &=\textstyle \int_{\mathcal{X}}\left( \int_{\mathcal{M}}\sqrt{m(x)}\Pi_n
	(dm)\right) ^{2}\lambda (dx).
\end{split}
\end{equation*}

\noindent \textbf{iii) Forward Kullback-Leibler divergence:} the loss function $L(m,\bar{m})$ is now
\begin{equation*}\textstyle 
	D_{KL}(m||\bar{m})=\int_{\mathcal{X}}m(x)\ln \frac{m(x)}{\bar{m}(x)}\lambda(dx),
\end{equation*}
and so the associated Bayes risk can be written as
\begin{eqnarray*}
	\textstyle 	R_{D_{KL}}(\bar{m}|D) &=&\textstyle \int_{\mathcal{M}}\int_{\mathcal{X}}m(x)\ln \frac{
		m(x)}{\bar{m}(x)}\lambda (dx)\Pi_n (dm) \\
	&=&\textstyle \int_{\mathcal{X}}\int_{\mathcal{M}}m(x)\ln m(x)\Pi_n(dm)\lambda (dx)-\int_{\mathcal{X}}\int_{\mathcal{M}}m(x)\Pi_n
	(dm)\ln \bar{m}(x)\lambda (dx) \\
	&=&\textstyle  C-\int_{\mathcal{X}}\mathbb{E}[m](x)\ln \bar{m}(x)\lambda (dx).
\end{eqnarray*}
Introducing the Boltzmann entropy of $\mathbb{E}[m]$ and adjusting the constant $C$  we get that 
\begin{eqnarray*}
	\textstyle 	R_{D_{KL}}(\bar{m}|D) &=& \textstyle C^{\prime }+\int_{\mathcal{X}}\mathbb{E}[m](x)\ln 
	\mathbb{E}[m](x)\lambda (dx)-\int_{\mathcal{X}}\mathbb{E}[m](x)\ln \bar{m}
	(x)\lambda (dx) \\
	&=&C^{\prime }+D_{KL}(\mathbb{E}[m]||\bar{m}),
\end{eqnarray*}
so the optimizer of $R_{D_{KL}}(\bar{m}|D)$ %and $D_{RKL}(\mathbb{E}[m],\bar{m})$
is the Bayesian model average. \\

\noindent \textbf{iv) Reverse Kullback-Leibler divergence:} since $
	D_{RKL}(m||\bar{m})=	D_{KL}(\bar{m}||m) ,
$
we have 
\begin{eqnarray*}
	R_{D_{RKL}}(\bar{m}|D) &=&\textstyle \int_{\mathcal{M}}\int_{\mathcal{X}}\bar{m}(x)\ln 
	\frac{\bar{m}(x)}{m(x)}\lambda (dx)\Pi_n (dm) \\
	&=&\textstyle \int_{\mathcal{X}}\bar{m}(x)\ln \bar{m}(x)\lambda (dx)-\int_{\mathcal{X}}
	\bar{m}(x)\int_{\mathcal{M}}\ln m(x)\Pi_n (dm)\lambda (dx) \\
	&=&\textstyle \int_{\mathcal{X}}\bar{m}(x)\ln \bar{m}(x)\lambda (dx)-\int_{\mathcal{X}}
	\bar{m}(x)\ln \exp \mathbb{E}[\ln m(x)]\lambda (dx) \\
	&=&\textstyle \int_{\mathcal{X}}\bar{m}(x)\ln \frac{\bar{m}(x)}{\exp \mathbb{E}[\ln m(x)]}
	\lambda (dx).
\end{eqnarray*}
Denote by $Z$ the normalization constant so that $\frac{1}{Z}\int_{\mathcal{X}
}\exp \mathbb{E}[\ln m](x)\lambda (dx)=1$. Then, 
\begin{eqnarray*}
	R_{D_{RKL}}(\bar{m}|D)+\ln Z &=&\textstyle \int_{\mathcal{X}}\bar{m}(x)\ln \frac{\bar{m}
		(x)}{\exp \mathbb{E}[\ln m(x)]}\lambda (dx)+\int_{\mathcal{X}}\bar{m}(x)\ln
	Z\lambda (dx) \\
	&=&\textstyle \int_{\mathcal{X}}\bar{m}(x)\ln \frac{\bar{m}(x)}{\frac{1}{Z}\exp \mathbb{
			E}[\ln m(x)]}\lambda (dx) \\
	&=&\textstyle D_{RKL}\left(\frac{1}{Z}\exp \mathbb{E}[\ln m(x)]||\bar{m}\right).
\end{eqnarray*}
Therefore,  the extremum of $R_{D_{RKL}}(\bar{m}|D)$ is attained when the last expression vanishes, in other words when $\bar{m}$ is the Bayesian \emph{exponential}
model average given by
\begin{equation*}\textstyle 
	\hat{m}(x)=\frac{1}{Z}\exp \int_{\mathcal{M}}\ln m(x)\Pi_n
	(dm).
\end{equation*}\\

\section{Wasserstein barycenters}
\label{subsec wasserstein_barycenters}

\black{\begin{proof}[Proof of Theorem \ref{baryexists}]
	Assume  $\Gamma\in \mathcal W_{p}(\mathcal W_p(\mathcal X))$.  Then  $B_V:=\inf_{\nu \in \mathcal{M}} V_p(\nu)$ is finite. Now, let $\{\nu_n\}\subset \mathcal{M}$ be such that
	$$\textstyle \int_{\mathcal{W}_p(\mathcal X)}  W_p(\nu_n,m)^p\Gamma(dm)  \searrow B_V \mbox{ as }n\to \infty. $$
	For $n$ large enough we have
	$$\textstyle W_p\left(\nu_n\, ,\, \int_{\mathcal{W}_p(\mathcal X)}m\Gamma(dm) \right)^p\leq \int_{\mathcal{W}_p(\mathcal X)}W_p(\nu_n\, ,\,m)^p\Gamma(dm)\leq B_V +1=:K, $$
	by convexity of optimal transport costs. From this we derive that (for every $x$) 
	$$\textstyle \sup_n \int_{\mathcal X} d(x,y)^p\nu_n(dy)<\infty.$$
	By Markov inequality this shows, for each $\epsilon>0$, that there is $\ell$ large enough such that $\sup_n \nu_n(\{y\in \mathcal X:d(x,y)> \ell\})\leq \epsilon$. As explained in \cite{le2017existence}, the assumptions on $\mathcal X$ imply that $\{y\in \mathcal X:d(x,y)\leq \ell\}$ is compact (Hopf-Rinow theorem), and so we deduce the tightness of $\{\nu_n\}$. By Prokhorov theorem, up to selection of a subsequence, there exists $\nu\in \mathcal{M}$ which is its weak limit. By Fatou's lemma:
	$$\textstyle B_V=\lim \int_{}  W_p(\nu_n,m)^p\Gamma(dm) \geq \int_{} \liminf W_p(\nu_n,m)^p\Gamma(dm) \geq \int_{} W_p(\nu,m)^p\Gamma(dm),$$
	hence $\nu$ is a $p-$Wasserstein barycenter.	\black{For the converse implication, see Remark \ref{rem BWB implies WpWp and finmomBMA}.}  
\end{proof}}

\section{\black{A condition for existence of barycenters of Bayesian posteriors }  }
\label{subsec existence_barycenters}

We last provide a general condition on the prior $\Pi$ ensuring that, for given $p\geq 1$,  
$$\Pi_n\in \mathcal W_{p}(\mathcal W_p(\mathcal X))\,\,\, \text{for all possible data points } (x_1,\dots, x_n)\in {\cal X}^n \text{ and all }n.  $$

\begin{definition}
	We say that $\Pi\in\mathcal{P}(\mathcal P(\mathcal X))$ is $p-$\emph{integrable after updates} if it satisfies the conditions
	\begin{enumerate}
		\item For all $x\in \mathcal X,\ell>1$: $$\textstyle \int_{\mathcal M}m(x)^\ell \Pi(dm)<\infty.$$
		\item For some $y\in \mathcal X$ and $\varepsilon>0$:
		$$\textstyle \int_{\mathcal M}\left ( \int_{\mathcal X} d(y,z)^p m(dz) \right )^{1+\varepsilon} \Pi(dm)<\infty.$$
	\end{enumerate}
\end{definition}

\begin{remark}
Condition (2) above can be intuitively denoted   as $\Pi\in \mathcal W_{p+}(\mathcal W_p(\mathcal X))$.  Of course, one has $\mathcal W_{p+}(\mathcal W_p(\mathcal X))\subseteq \mathcal W_{p}(\mathcal W_p(\mathcal X))$.
\end{remark}

\begin{remark} If $\Pi \in \mathcal{P}(\mathcal{W}_{p,ac}(\mathcal{X}))$ has finite support, then Conditions (1) and (2) are satisfied. On the other hand, if $\Pi$ is supported on a scatter-location family (see Section \ref{sec scat loc}) containing one element with a bounded density and a finite $p$-moment, then Conditions (1) and (2) are fulfilled if for example $\supp(\Pi)$ is tight. Conditions (1) and (2)  are also satisfied in  Example \ref{ex:gaussBWBvsBMA}.
\end{remark}

\begin{lemma}\label{lem int upd}
	Suppose that $\Pi$ is $p-$integrable after updates. Then, for each $x\in \mathcal X$, the measure
	$$\textstyle \tilde \Pi(dm):= \frac{m( x)\Pi(dm)}{\int_{\mathcal M} \bar m(x)\Pi(d\bar m)},$$
	is also $p-$integrable after updates.
\end{lemma}

\begin{proof}
	We verify Property (1) first. Let $\ell> 1$ and $\bar x\in\mathcal X$ given. Then
	$$\textstyle \int_{\mathcal M}m(\bar x)^\ell m(x) \Pi(dm)\leq \left( \int_{\mathcal M}m(x)^s \Pi(dm)\right )^{1/s}\, \left( \int_{\mathcal M}m(\bar x)^{t\ell} \Pi(dm)\right )^{1/t},$$
	with $s,t$ conjugate H\"older exponents. This is finite since $\Pi$ fulfils Property (1).
	
	We now establish Property (2). Let $y\in \mathcal X,\varepsilon>0$. Then
	\begin{align*}
		&\,\,\textstyle \int_{\mathcal M}\left ( \int_{\mathcal X} d(y,z)^p m(dz) \right )^{1+\varepsilon} m(x)\Pi(dm)\\ \leq&\textstyle  \,\, \left( \int_{\mathcal M}m(x)^s \Pi(dm)\right )^{1/s}\, \left( \int_{\mathcal M}\left ( \int_{\mathcal X} d(y,z)^p m(dz) \right )^{(1+\varepsilon)t}\Pi(dm)\right )^{1/t}.
	\end{align*}
	The first term in the r.h.s.\ is finite by Property (1). The second term in the r.h.s.\ is finite by Property (2), if we take $\varepsilon$ small enough and $t$ close enough to $1$. We conclude.
\end{proof}

\begin{lemma}
	Suppose that $\Pi$ is $p-$integrable after updates. Then for all $n\in \mathbb N$ and $(x_1,\dots,x_n)\in\mathcal X^n$, the posterior $\Pi_n$ is also $p-$integrable after updates.
\end{lemma}

\begin{proof}
	By Lemma \ref{lem int upd}, we obtain that $\Pi_1$ is integrable after updates. By induction, suppose $\Pi_{n-1}$ has this property. Then as 
	$$\textstyle \Pi_n(dm)=\frac{m( x_n)\Pi_{n-1}(dm)}{\int_{\mathcal M} \bar m(x_n)\Pi_{n-1}(d\bar m)},$$
	we likewise conclude that $\Pi_n$ is $p-$integrable after updates.
\end{proof}

\end{document}